\documentclass{article}

\usepackage{lipsum} 

\PassOptionsToPackage{numbers, compress}{natbib}
\usepackage[preprint]{neurips_2019}

\usepackage{graphicx}
\graphicspath{{figures/}}

\usepackage{tabu}

\usepackage[colorlinks]{hyperref}

\usepackage{tepper}

\title{%
    Do place cells dream of conditional probabilities?\\
    Learning Neural Nystr\"om representations
}

\author{%
  Mariano Tepper \\
  Intel Labs \\
  Hillsboro, OR 97124 \\
  \texttt{mariano.tepper@intel.com} \\
}

\begin{document}
	
\maketitle

\begin{abstract}
    We posit that hippocampal place cells encode information about future locations under a transition distribution observed as an agent explores a given (physical or conceptual) space. The encoding of information about the current location, usually associated with place cells, then emerges as a necessary step to achieve this broader goal.
    We formally derive a biologically-inspired neural network from Nystr\"om kernel approximations and empirically demonstrate that the network successfully approximates transition distributions. The proposed network yields representations that, just like place cells, soft-tile the input space with highly sparse and localized receptive fields.
    Additionally, we show that the proposed computational motif can be extended to handle supervised problems, creating class-specific place cells while exhibiting low sample complexity.
\end{abstract}

\section{Introduction}
\label{sec:intro}

Neuroscientists have observed that the receptive fields of many neurons are localized in and effectively tile the parameter space they represent. For example, a V1 neuron responds to input localized in visual space and orientation \cite{hubelReceptiveFieldsBinocular1962}, an auditory neuron responds to input localized in frequency space \cite{knudsenCentersurroundOrganizationAuditory1978}, and a hippocampal place cell is active in a particular spatial location \cite{okeefeHippocampusCognitiveMap1978}.

Following~\cite{stachenfeldDesignPrinciplesHippocampal2014}, we posit that hippocampal place fields' purpose is to encode predictions about future locations under the transition distribution observed as an agent explores a given (physical or conceptual) space; encoding information about the current location is needed to achieve this broader goal. For example, the successor representation (SR)~\cite{dayanImprovingGeneralizationTemporal1993}, commonly used in reinforcement learning to model the agent's transition distribution, has an eigendecomposition akin to grid cells~\cite{stachenfeldDesignPrinciplesHippocampal2014}. Low-rank approximations are appealing as they embody the efficiency/accuracy dilemma when learning representations.
This brings forward the main question of this work: Can place cells also arise from low-rank approximations of a transition distribution? We show that this is indeed possible by introducing a biologically-inspired neural network that leverages the low-rank Nystr\"om method~\cite{williamsUsingNystromMethod2001}. 

In the machine learning literature, approximating an input conditional probability with an output kernel is a recurring theme. Among other key exemplars~\cite[e.g.,][]{bojanowskiEnrichingWordVectors2017,globersonEuclideanEmbeddingCooccurrence2007}, we have SNE~\cite{hintonStochasticNeighborEmbedding2003} and t-SNE~\cite{vandermaatenVisualizingHighdimensionalData2008} for data visualization, and word2vec~\cite{mikolovDistributedRepresentationsWords2013} and GLOVE~\cite{penningtonGloveGlobalVectors2014} to produce vector embeddings for words.
In all cases, the input conditional probability models the transition from a given datum to any other datum in the dataset.
While in different settings these input distributions take different forms, our formulation can work with any of them. Next, we summarize some of these alternatives. 

\textbf{Row-normalized kernels.}
Often, in many forms of machine learning, the input conditional probability distribution takes the form of a row-normalized kernel, i.e.,
\begin{equation}
    p_{\text{in}} (\vect{y} | \vect{x}) \coloneqq
    \tfrac{
        K_{\text{in}} (\vect{x}, \vect{y})
    }{
        \int K_{\text{in}} (\vect{x}, \vect{z})
        \, d\vect{z}
    }
    .
    \label{eq:kernel_input_proba}
\end{equation}
For computational simplicity, the denominator is usually computed only over the observed data $\vect{z} \in \{ \vect{x}_i \}_{i=1}^{n}$.
This form has shown great success for nonlinear dimensionality reduction~\cite{belkinLaplacianEigenmapsDimensionality2003,coifmanGeometricDiffusionsTool2005,hintonStochasticNeighborEmbedding2003,vandermaatenVisualizingHighdimensionalData2008}.

\textbf{Semidefinite programs.}
Peng et al.~\cite{pengApproximatingKmeanstypeClustering2007} introduced a semidefinite relaxation of k-means clustering which yields the same solution as the non-convex original formulation for datasets with well segregated clusters.
For dataset $\mat{X} \in \Real^{d \times n}$ with $n$ datapoints in $d$ dimensions, the formulation is
\begin{equation}
    \mat{Q}^{*} = \argmax_{\mat{Q} \in \Real^{n \times n}}  \traceone{\transpose{\mat{X}} \mat{X} \mat{Q}}
	\quad \text{s.t.} \quad
	\begin{gathered}
	    \mat{Q} \psd \mat{0} ,\quad
	    \mat{Q} \vect{1} = \vect{1} ,\quad
    	\traceone{\mat{Q}} = k  ,\quad
    	\mat{Q} \geq \mat{0} .
	\end{gathered}
    \label[problem]{eq:sdp-kmeans}
\end{equation}
Alternatively, the trace constraint can be replaced by $\diag{\mat{Q}} = \tfrac{k}{n} \vect{1}$.
Unexpectedly, when data contains (greater than zero-dimensional) manifolds, the semidefinite program's solution, $\mat{Q}^*$, captures such geometrical structures~\cite{tepperClusteringSemidefinitelyNot2018}. 
In this setting, $\mat{Q}^*$ is effectively a kernel that is automatically learned from data~\cite{tepperClusteringSemidefinitelyNot2018}. Under the constraints $\mat{Q} \vect{1} = \vect{1}$ and $\mat{Q} \geq \mat{0}$, we can consider $\mat{Q}^{*}$ as a conditional probability of reaching, from any starting data point, any other data point.

\textbf{Successor representations.}
Here, we consider an agent interacting with its environment in a
sequential manner. Starting from a state $s_0 \in \set{S}$, at each
step $t$ the agent takes an action $a_t \in \set{A}$ following some policy $\pi$, to which the environment responds with a state $s_{t+1} \in \set{S}$ according to a transition probability function $p(s_{t+1} | s_{t}, a_{t})$.
Then, the SR~\cite{dayanImprovingGeneralizationTemporal1993} is $\mat{\Psi} (s, s') = \mathbb{E}_{\pi, p} \left[ \sum_{t=0}^{\infty} \alpha^t \indicator{s_t = s'} \big| s_0 = s \right]$, where $\indicator{\cdot}$ is the indicator function and $\alpha$ encodes future discounts.
% \begin{equation}
%     \mat{\Psi} (s, s') = \mathbb{E}_{\pi, p} \left[ \sum_{t=0}^{\infty} \gamma^t \indicator{s_t = s'} \bigg| s_0 = s \right]
% \end{equation}
By linearity of expectation, $\sum_{s'} \mat{\Psi} (s, s') = (1 - \alpha)^{-1}$.
We define the conditional probability $p_{\text{in}}$ of landing in state $s'$ given the current state $s$ as
\begin{equation}
    \textstyle
    p_{\text{in}} (s' | s) \coloneqq \left( 1 - \alpha \right)
    \mathbb{E}_{\pi, p} \left[ \sum_{t=0}^{\infty} \alpha^t \indicator{s_t = s} \big| s_0 = s \right] .
    \label{eq:successor_representation_input_proba}
\end{equation}

\subsection{Related work}

Our work takes inspiration from convolutional kernel networks (CKN)~\cite{mairalConvolutionalKernelNetworks2014,mairalEndtoendKernelLearning2016}. CKNs replace the matrix-vector multiplications used in convolutional networks by kernel feature maps and are used for supervised classification. Their feature maps use Nystr\"om-like approximations. CKNs are also related to radial basis function networks~\cite{broomheadRadialBasisFunctions1988} and self-organizing maps~\cite{kohonenSelforganizedFormationTopologically1982}.

Our model also shares some similarities with the bag-of-features (BoF) approach, as place cells can be interpreted as soft-quantizers of their input space.
Passalis et al.~\cite{passalisLearningBagoffeaturesPooling2017} proposed a convolutional neural network that incorporates a BoF layer composed of $\ell_1$-normalized neurons with RBF receptive fields. They use this model for supervised classification. 

Sengupta et al.~\cite{senguptaManifoldtilingLocalizedReceptive2018} show that localized receptive fields emerge in similarity-preserving
networks of rectifying neurons. These networks learn to represent low-dimensional manifolds populated by sensory inputs and yield localized receptive fields tiling these manifolds.

\textbf{Contributions.}
In \cref{sec:neustrom}, we show that place cells can arise from low-rank approximations of conditional probabilities. We do so by leveraging the Nystr\"om method within a siamese network architecture. We streamline Nystr\"om approximations by making use of key kernel properties, resulting in a biologically-inspired model with reduced computational cost compared to classical Nystr\"om. Additionally, in \cref{sec:optimization} we handle the optimization problem in a biologically-inspired and yet computationally efficient way.
% Code is available at ANONYMIZED URL.

In \cref{sec:supervised}, we show that once place cells are learned in an unsupervised fashion, supervised learning becomes easier, similarly to BoF that classically facilitates supervised learning. In contrast to these methods, we show that the very same siamese architecture can be re-used, causing the emergence of goal/class specific place cells while using the same computational motif.

In \cref{sec:episodic}, we show that the proposed formulation does not need to have explicit access to $p_{\text{in}}$. Using successor representations as an example, we derive an algorithm for episodic computing that drawing samples from $p_{\text{in}}$ (which can be done without actually knowing $p_{\text{in}}$).

In \cref{sec:results}, we present experimental results.
We provide some concluding remarks in \cref{sec:conclusions}.

\section{Neural Nystr\"om representations}
\label{sec:neustrom}

The Kullback-Leibler (KL) divergence is popular in the machine learning literature to compare a distribution and some learned approximation~\cite{bojanowskiEnrichingWordVectors2017,globersonEuclideanEmbeddingCooccurrence2007,hintonStochasticNeighborEmbedding2003,leeAlgorithmsNonnegativeMatrix2000,mikolovDistributedRepresentationsWords2013,vandermaatenVisualizingHighdimensionalData2008,yangKullbackLeiblerDivergenceNonnegative2011}.
Given an input conditional probability $p_{\text{in}}$, we seek the conditional probability $p_{\text{out}}$ that minimizes the KL divergence. Let $p_{\text{out}}$ be defined on a space $\set{V}$ such that $\vect{v}_{\vect{x}} \in \set{V}$ is computed from $\vect{x}$ (in our work, using a neural network). Often, the objective is to find the mapping $\vect{v}$ itself. Our loss function is then
\begin{equation}
	\mathcal{L}
	= 
	\iint p_{\text{in}} (\vect{y} | \vect{x})
	\log
	\tfrac{
	    p_{\text{in}} (\vect{y} | \vect{x})
	}{
	    p_{\text{out}} (\vect{v}_{\vect{y}} | \vect{v}_{\vect{x}})
	}
	\, d\vect{x} d\vect{y}
	= 
	- \iint p_{\text{in}} (\vect{y} | \vect{x})
	\log p_{\text{out}} (\vect{v}_{\vect{y}} | \vect{v}_{\vect{x}})
	\, d\vect{x} d\vect{y}
	.
	\label{eq:kernel_kl-divergence}
\end{equation}
Our main assumption is that $p_{\text{in}}$ is dominated by local interactions and far-away interactions are negligible (a common assumption in the literature, e.g., in t-SNE~\cite{vandermaatenVisualizingHighdimensionalData2008} and word2vec~\cite{mikolovDistributedRepresentationsWords2013}). We can then represent the joint distribution $p_{\text{out}}$ with a kernel $K$, i.e.,
$p_{\text{out}} (\vect{v}_{\vect{x}} , \vect{v}_{\vect{y}}) \propto K(\vect{v}_{\vect{x}}, \vect{v}_{\vect{y}})$.
% \begin{equation}
%     p_{\text{out}} (\vect{v}_{\vect{x}} , \vect{v}_{\vect{y}}) \propto K(\vect{v}_{\vect{x}}, \vect{v}_{\vect{y}})
%     .
% \end{equation}
For clarity of exposition, we simply consider the RBF kernel
$
	K(\vect{v}_{\vect{x}}, \vect{v}_{\vect{y}}) = \exp (-\norm{\vect{v}_{\vect{x}} - \vect{v}_{\vect{y}}}{2}^2 )
$ ,
although there are many suitable choices and our formulations can handle them seamlessly.
From the definition of a conditional probability, i.e., $p_{\text{out}} (\vect{v}_{\vect{x}} , \vect{v}_{\vect{y}}) = p_{\text{out}} (\vect{v}_{\vect{y}} | \vect{v}_{\vect{x}}) \cdot \int p_{\text{out}} (\vect{v}_{\vect{x}} , \vect{v}_{\vect{z}}) \, d\vect{v}_{\vect{z}}$, we can write
% \begin{equation}
%     p_{\text{out}} (\vect{v}_{\vect{y}} | \vect{v}_{\vect{x}}) =
%     \frac{
%         K(\vect{v}_{\vect{x}}, \vect{v}_{\vect{y}})
%     }{
%         \int K(\vect{v}_{\vect{x}}, \vect{\vect{z}})
%         \, d\vect{v}_{\vect{z}}
%     }
% \end{equation}
\begin{equation}
	\mathcal{L}
	= 
	- \iint p_{\text{in}} (\vect{y} | \vect{x})
	\log
	\tfrac{
	    K(\vect{v}_\vect{x}, \vect{v}_\vect{y})
	}{
	    \int K(\vect{v}_{\vect{x}}, \vect{v}_{\vect{z}})
        \, d\vect{v}_{\vect{z}}
	}
	\, d\vect{x} d\vect{y}
	.
	\label{eq:KL_loss}
\end{equation}
So far, the only parameters of the model are the ones used to compute $\vect{v}_{\vect{x}}$ from $\vect{x}$.

\subsection{Nystr\"om kernel approximation}

While we seek to approximate $p_{\text{in}}$ with $K$, we simultaneously seek a low-rank representation of $K$, which will provide computational efficiency.
The Nystr\"om method~\cite{williamsUsingNystromMethod2001} has proven successful for  this task and we use it as a starting point for our approach.
The Nystr\"om approximation is (see \cref{sec:nystrom} for a detailed justification)
\begin{equation}
    K(\vect{v}_{\vect{x}}, \vect{v}_{\vect{y}})
    \approx
    \transpose{\vect{k}_{\mat{W}, \vect{v}_{\vect{x}}}} \mat{K}_{\mat{W}, \mat{W}}^{-1} \vect{k}_{\mat{W}, \vect{v}_{\vect{y}}}
    ,
\end{equation}
where $\mat{W} = [\vect{w}_1, \dots, \vect{w}_r]$ is a collection of landmark points and
\begin{align}
	\vect{k}_{\mat{W}, \vect{v}}
	&= 
	\begin{bmatrix}
    	K(\vect{w}_1, \vect{v}) \\
        \vdots \\
        K(\vect{w}_r, \vect{v})
    \end{bmatrix}
    \in \Real_+^{r} ,
    &
    \mat{K}_{\mat{W}, \mat{W}}
    &=
	\begin{bmatrix}
    	K(\vect{w}_1, \vect{w}_1) & \cdots & K(\vect{w}_1, \vect{w}_r) \\
        \vdots & \ddots & \vdots \\
        K(\vect{w}_r, \vect{w}_1) & \cdots & K(\vect{w}_r, \vect{w}_r)
    \end{bmatrix}
    \in \Real_+^{r \times r}
    .
    \label{eq:nystrom_matrices}
\end{align}
% \begin{subequations}
% \begin{align}
% 	\vect{k}_{\mat{W}, \vect{v}}
% 	&= 
% 	\begin{bmatrix}
%     	K(\vect{w}_1, \vect{v}) \\
%         \vdots \\
%         K(\vect{w}_r, \vect{v})
%     \end{bmatrix}
%     \in \Real_+^{r} ,
%     \\
%     \mat{K}_{\mat{W}, \mat{W}}
%     &=
% 	\begin{bmatrix}
%     	K(\vect{w}_1, \vect{w}_1) & \cdots & K(\vect{w}_1, \vect{w}_r) \\
%         \vdots & \ddots & \vdots \\
%         K(\vect{w}_r, \vect{w}_1) & \cdots & K(\vect{w}_r, \vect{w}_r)
%     \end{bmatrix}
%     \in \Real_+^{r \times r} .
% \end{align}
% \end{subequations}
We assume that $\mat{K}_{\mat{W}, \mat{W}}$ is invertible (otherwise, we take a pseudo-inverse).
Traditionally, when the set of vectors $\vect{v}_{\vect{x}}$ is fixed, placing the landmark points $\mat{W}$ wisely is key to the method's success.
Using $\mat{K}_{\mat{W}, \mat{W}}^{-1} = \mat{K}_{\mat{W}, \mat{W}}^{-1/2} \mat{K}_{\mat{W}, \mat{W}}^{-1/2}$,
we have
\begin{equation}
    K(\vect{v}_{\vect{x}}, \vect{v}_{\vect{y}})
    \approx
    % \transpose{\vect{k}_{\mat{W}, \vect{v}_{\vect{x}}}} \mat{K}_{\mat{W}, \mat{W}}^{-1/2} \mat{K}_{\mat{W}, \mat{W}}^{-1/2} \vect{k}_{\mat{W}, \vect{v}_{\vect{y}}}
    % =
    \transpose{\vect{f}_{\vect{x}}} \vect{f}_{\vect{y}}
    ,
    \quad\quad
    \text{where}
    \quad\quad
    \vect{f}_{\vect{x}}
    \coloneqq
    % f(\mat{W}, \vect{v}_{\vect{x}})
    % =
    \mat{K}_{\mat{W}, \mat{W}}^{-1/2} \vect{k}_{\mat{W}, \vect{v}_{\vect{x}}} .
    \label{eq:feature_nystrom}
\end{equation}
The computation of feature $\vect{f}_{\vect{x}}$ acts as a network with 3 blocks:
(1) \textbf{Embedding:} Given input $\vect{x}$, compute an embedding vector $\vect{v}_{\vect{x}}$ (this sub-network may contain multiple layers).
(2) \textbf{Kernel layer:} given input $\vect{v}_{\vect{x}}$, produce output $\vect{a}_{\vect{x}} = \vect{k}_{\mat{W}, \vect{v}_{\vect{x}}}$.
(3) \textbf{Fully-connected layer:} given input $\vect{a}_{\vect{x}}$, produce output $\vect{f}_{\vect{x}} = \mat{K}_{\mat{W}, \mat{W}}^{-1/2} \vect{a}_{\vect{x}}$.
The kernel and fully-connected layers share weights $\mat{W}$.
Now, plugging \cref{eq:feature_nystrom} into \cref{eq:KL_loss} we get
\begin{equation}
	\mathcal{L}
	=
	- \iint p_{\text{in}} (\vect{y} | \vect{x})
	\log
	\tfrac{
	    \transpose{\vect{f}_{\vect{x}}} \vect{f}_{\vect{y}}
	}{
	    \int \transpose{\vect{f}_{\vect{x}}} \vect{f}_{\vect{z}}
        \, d\vect{z}
	}
	\, d\vect{x} d\vect{y} .
\end{equation}
All model parameters ($\mat{W}$ and the embedding) can be updated using backpropagation~\cite{mairalEndtoendKernelLearning2016}.
In \cref{sec:optimization}, we discuss how to handle the partition function, i.e., the integral $\int \transpose{\vect{f}_{\vect{x}}} \vect{f}_{\vect{z}} \, d\vect{z}$. 

The Nystr\"om-based network architecture just presented has some limitations both from the computational and neuroscience viewpoints. Computing the inverse square root of an $r \times r$ matrix is expensive at $O(r^3)$ and needs to be handled with care to avoid numerical instability during backpropagation. From a neuroscience perspective, it is not biologically plausible to have different neurons sharing weights, as the update operations for the synaptic weights become necessarily non-local.

\subsection{Introducing Neural Nystr\"om representations}

We now present a method that overcomes the limitations of the Nystr\"om formulation.
First, we replace $\mat{K}_{\mat{W}, \mat{W}}^{-1/2}$ by a regular fully connected layer (not sharing weights with the kernel layer).
Let us denote these new weights by $\mat{M}$.
Notice that if $\mat{M}$ is a diagonal nonnegative matrix, the approximation $\mat{M} \vect{k}_{\mat{W}, \vect{v}_{\vect{x}}}$ becomes a Gaussian quadrature.

Next, we introduce nonlinearities that stem from key kernel characteristics.
For commonly used kernels, we have $K(\vect{v}_{\vect{x}}, \vect{v}_{\vect{y}}) \geq 0$ (an additional undesirable feature of the Nystr\"om representation is that there are no guarantees that $\transpose{\vect{f}_{\vect{x}}} \vect{f}_{\vect{y}} \geq \vect{0}$). Furthermore, we can exert control on the diagonal of the kernel approximation from the value of $K(\vect{v}_{\vect{x}}, \vect{v}_{\vect{x}})$. With many of the commonly used kernels, we have $K(\vect{v}_{\vect{x}}, \vect{v}_{\vect{x}}) = 1$ but this is not necessarily always the case. For example, the kernel $K(\vect{v}_{\vect{x}}, \vect{v}_{\vect{y}}) = \exp( \transpose{\vect{v}_{\vect{x}}} \vect{v}_{\vect{y}})$ is popular in word embedding models. Thus, we add the constraints
\begin{align}
    \mat{M} \vect{k}_{\mat{W}, \vect{v}_{\vect{x}}} &\geq \vect{0}
    ,
    &
    \norm{\mat{M} \vect{k}_{\mat{W}, \vect{v}_{\vect{x}}}}{2}^2 &= K(\vect{v}_{\vect{x}}, \vect{v}_{\vect{x}})
    .
\end{align}
These constraints (nonnegativity, fixed diagonal, and control of the partition function) are used  in \cref{eq:sdp-kmeans}.
They also recently appeared in a biologically-plausible neural network~\cite{senguptaManifoldtilingLocalizedReceptive2018} that learn manifold-tiling localized receptive fields from upstream network activity. Nonnegativity has a long use in neuroscience, supported by the nonnegativity of firing rates. Normalization also appears frequently in neuroscience~\cite[e.g.,][]{carandiniNormalizationCanonicalNeural2012} and is commonly considered as an inhibitory process~\cite{pehlevanClusteringNeuralNetwork2017,senguptaManifoldtilingLocalizedReceptive2018}.

Let $\left[ \cdot \right]_+$ denote the ReLU function. We then obtain the Neural Nystr\"om network architecture
\begin{align}
    \vect{g}_{\vect{x}}
    =
    % g(\mat{M}, \mat{W}, \vect{v}_{\vect{x}})
    % =
    \sigma ( \mat{M} \vect{k}_{\mat{W}, \vect{v}_{\vect{x}}}, \sqrt{K(\vect{v}_{\vect{x}}, \vect{v}_{\vect{x}})} )
    ,
    \label{eq:neustrom_net}
    \\
% \end{equation}
% \begin{equation}
    \text{where,}\quad
    \sigma \left( \vect{a}, \lambda \right)
    =
    \tfrac{
        \lambda
    }{
        \norm{\left[ \vect{a} \right]_+}{2}
    } \left[ \vect{a} \right]_+
    .
    \label{eq:neustrom_nonlinearity}
\end{align}
In summary, the Neural Nystr\"om network architecture is similar to the one presented in the previous subsection, but with a few key differences. The \textbf{fully connected layer} (layer 3) does not share weights with the kernel layer. We add an additional layer with the nonlinearity defined in \cref{eq:neustrom_nonlinearity}.
The resulting Neural Nystr\"om architecture is depicted in \cref{fig:neustrom_architecture}, with the new loss function being
\begin{equation}
   	- \iint 
   	p_{\text{in}} (\vect{y} | \vect{x})
   	\log
   	\tfrac{
   	    \transpose{\vect{g}_{\vect{x}}} \vect{g}_{\vect{y}}
   	}{
   	    \int \transpose{\vect{g}_{\vect{x}}} \vect{g}_{\vect{z}}
        \, d\vect{z}
   	}
   		\, d\vect{x} d\vect{y} .
   	\label{eq:neustrom_loss}
\end{equation}

\begin{figure}
    \centering
    
    \begin{subfigure}{.57\textwidth}
        \includegraphics[width=\linewidth]{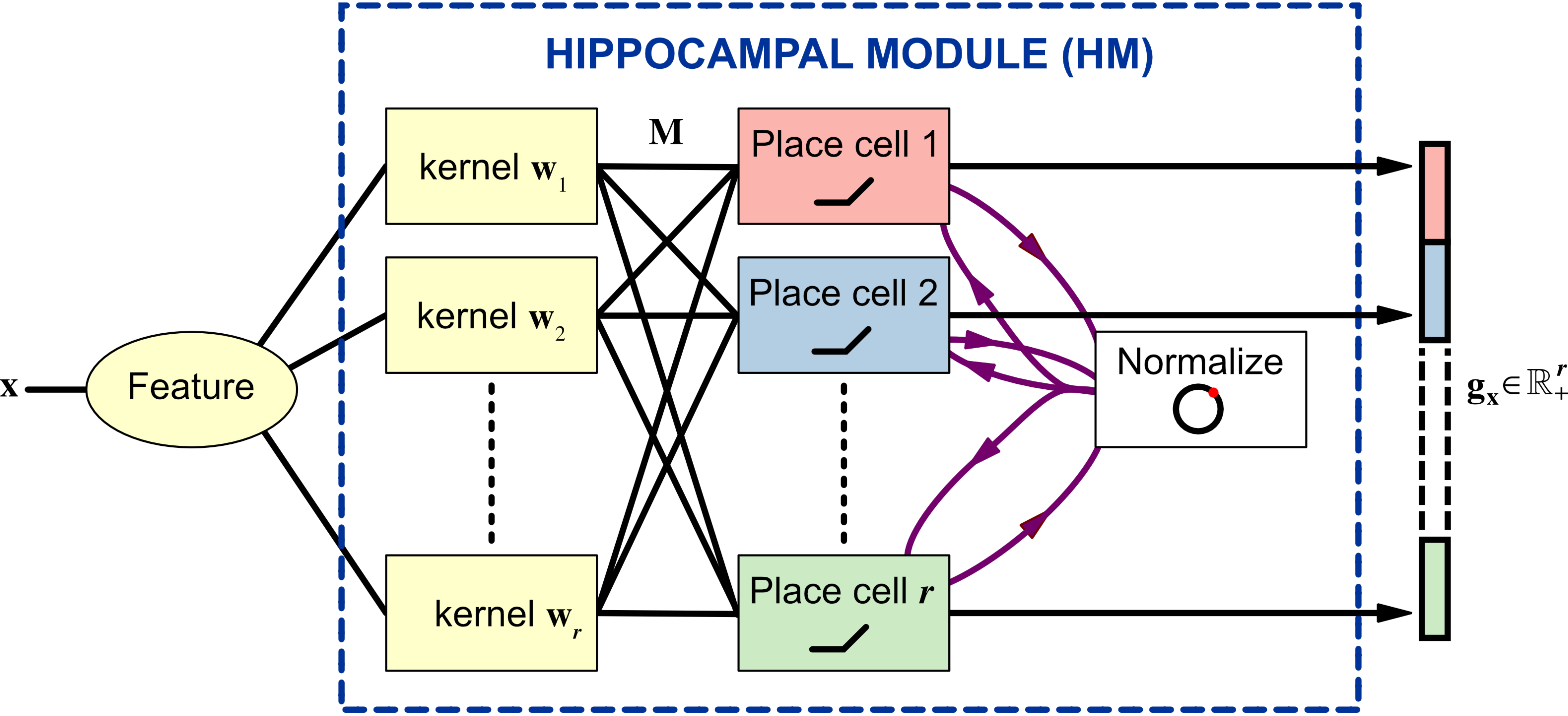}

        \caption{Neural Nystr\"om architecture.}
        \label{fig:neustrom_arch_hip_module}
    \end{subfigure}
    \hfill
    \begin{subfigure}{.37\textwidth}
        \includegraphics[width=\linewidth]{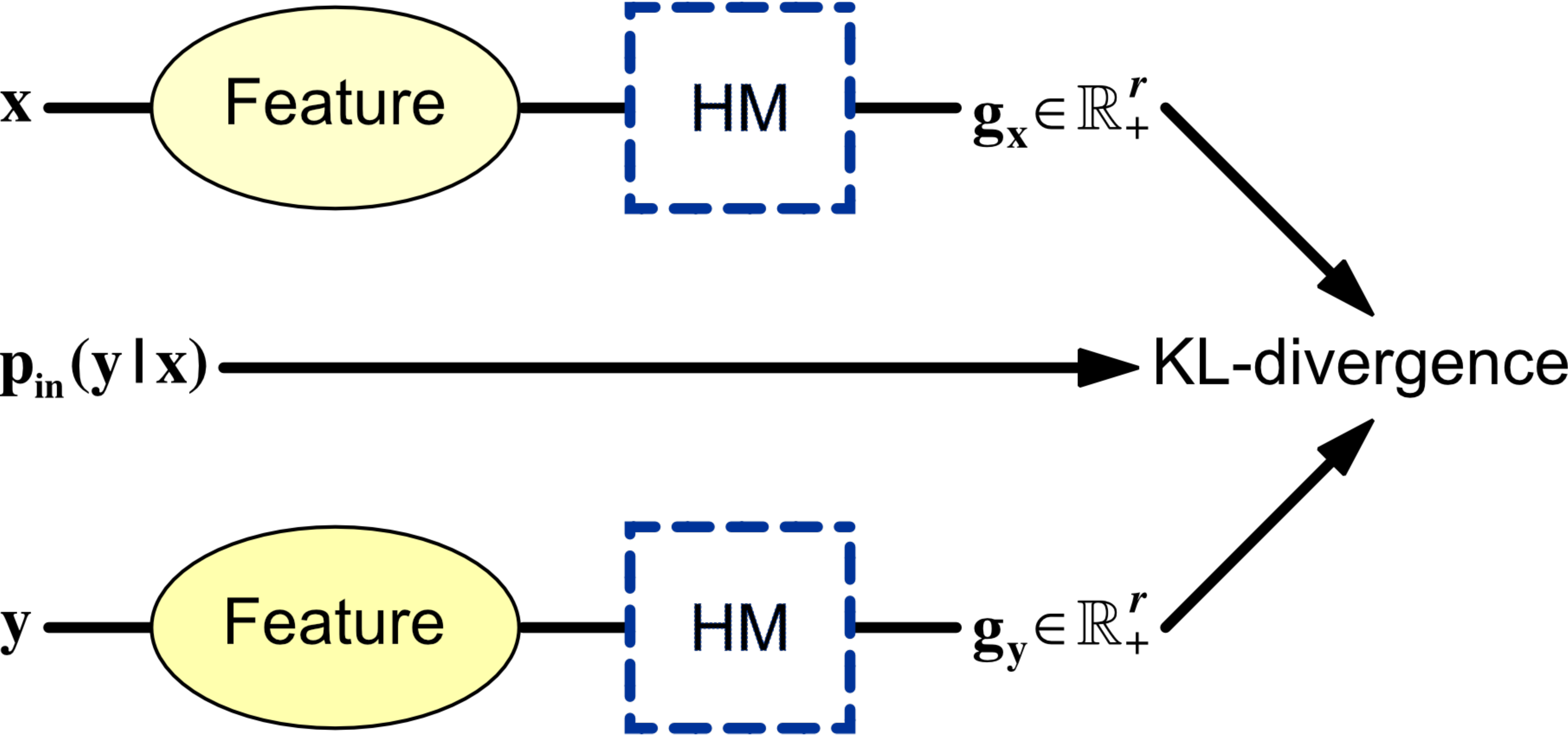}
        \caption{Siamese network architecture.}
        \label{fig:neustrom_arch_siamese}
    \end{subfigure}

    % \begin{tabu} to \textwidth {X[c,m]  X[c,m]}
    %     \includegraphics[width=\linewidth]{neustrom_arch_hip_module.pdf}%
    %     &
    %     \includegraphics[width=\linewidth]{neustrom_arch_siamese.pdf}%
    % \end{tabu}
    
    \caption{
        Neural Nystr\"om architecture, see \cref{eq:neustrom_net,eq:neustrom_loss}.
        The purple arrows to and from the Normalize block in (\subref{fig:neustrom_arch_hip_module}) express that normalization takes inputs from and produces an output distributed to every place cell.
        For simplicity, we omit the partition function from (\subref{fig:neustrom_arch_siamese}), see \cref{sec:optimization}.}
    \label{fig:neustrom_architecture}
\end{figure}

% \begin{enumerate}
%     \item Embedding layer: given input $\vect{x}$, produce output $\vect{v}_{\vect{x}}$.
%     \item Kernel layer: given input $\vect{v}_\vect{x}$, produce output $\vect{a}_\vect{x} \gets \mat{K}_{\mat{W}, \vect{v}_\vect{x}}$.
%     \item Fully-connected layer: given input $\vect{a}_\vect{x}$, produce output $\vect{g}_\vect{x} \gets \mat{M} \vect{a}_\vect{x}$.
%     \item ReLU non-linearity: $\vect{g}_\vect{x} \gets \left[ \vect{g}_\vect{x} \right]_+$
%     \item Normalization: $\vect{g}_\vect{x} = \frac{1}{\norm{\vect{g}_\vect{x}}{2}} \vect{g}_\vect{x}$
% \end{enumerate}

In \cref{sec:optimization}, we propose an efficient method to compute the computationally demanding partition function $\int \transpose{\vect{g}_{\vect{x}}} \vect{g}_{\vect{z}} \, d\vect{z}$.
We model the embedding $\vect{v}_{\vect{x}}$ with a neural network, as detailed in \cref{sec:embedding}. We point out that our goal is not data visualization. Thus, the embedding dimensionality is much larger than 2 and we regard it as a free parameter.
We update all model parameters using backpropagation.

\textbf{Emergence of place cells.}
Each component $K(\vect{w}_k, \vect{v}_{\vect{x}})$ of $\vect{k}_{\mat{W}, \vect{v}_{\vect{x}}}$ in \cref{eq:nystrom_matrices} looks like the receptive field of a place cell. However, there are key differences.
First, notice its receptive fields (e.g., Gaussian) are non-sparse. Additionally, depending on the scale of the cloud of embedding vectors $\vect{v}_{\vect{x}}$, it may either not be local at all or too local.
Contrarily, as shown in \cref{sec:results}, the proposed Neural Nystr\"om representations in \cref{eq:neustrom_net} are highly sparse and localized in their input space, just like place cells.
It is the fully connected layer and our nonlinearities that ``fix'' these ``mistakes'': they create place cells, one for each component of $\vect{g}_{\vect{x}}$. We point out that this is not trivial at all as, for example, a linear combination of Gaussian-like receptive fields will not be Gaussian-like in general. The degree of locality of our place cells is directly determined by the locality of $p_{\text{in}}$ and our quest to approximate it. Finally, our place cells soft-tile the embedding (i.e., cover it with overlapping soft-clusters), as experimentally shown in \cref{sec:results}.

\section{Stochastic optimization for finite-sized datasets}
\label{sec:optimization}

We now present an optimization method that reveals further connections with the neuroscience literature, while being computationally efficient.
Of course, negative sampling~\cite{mikolovDistributedRepresentationsWords2013} or noise contrastive estimation~\cite{gutmannNoisecontrastiveEstimationNew2010} are valid alternatives.

We assume that we are provided with a dataset $\{ \vect{x}_i \}_{i=1}^{n}$. We will make the common simplification of normalizing our kernel only over the observed data.
Let us, for brevity, notate $\vect{g}_i = \vect{g}_{\vect{x}_i}$.
Plugging the Neural Nystr\"om approximation into our (now discretized) objective function, we get
\begin{equation}
   	\mathcal{L}
   	= 
   	- \sum_{ij}
   	p_{\text{in}} (\vect{x}_j | \vect{x}_i)
   	\log
   	\tfrac{
   	    \transpose{\vect{g}_i} \vect{g}_j
   	}{
   	    \transpose{\vect{g}_i} \vect{c}
   	}
   	,
   	\label{eq:neustrom_loss_finite}
\end{equation}
where $\vect{c} = \sum_{i=0}^{n} \vect{g}_i$ summarizes the contribution of all points to the partition function. This situation is encountered in neuroscience, regarding $\vect{c}$ as the activity of an accumulator neuron~\cite{pehlevanClusteringNeuralNetwork2017}.

This situation occurs in online algorithms which have dependencies between the computations for different data points. One such example is online dictionary learning~\cite{mairalOnlineLearningMatrix2010}, where the dictionary update depends on all previously observed points. 
We regard $\vect{c}$ as a summary and consider that its summands cannot be individually recovered. We thus perform no backpropagation through $\vect{c}$.

A simple strategy for small datasets is, for each $i$, to keep track of old values of $\vect{g}_i$ and, update $\vect{c}$ by subtracting the old value and adding the new one.
\cref{algo:neustrom_algorithm_fixed_large} implements a similar idea, more suitable for large but finite-sized datasets: we remove the information from $\vect{c}$ that is older than two epochs~\cite{mairalOnlineLearningMatrix2010}. Since within each epoch we may encounter each datapoint $\vect{x}_i$ multiple times, we make sure to only include it in the summation once per epoch.

\section{Supervised multi-task learning}
\label{sec:supervised}

Here, we assume that we have successfully learned Neural Nystr\"om representations, which encode information about conditional probabilities. We now observe a relatively small number of samples that have (classification) labels associated to them. Each new such observation will be considered a new task, that we number from $1$ to $k$.
For each task, it is easy to translate these labels into a new conditional probability.
Essentially, for the $k$-th task,
\begin{equation}
    p^{(k)} (\vect{y} | \vect{x}) = \tfrac{1}{\#\operatorname{class}(\vect{x})} \indicator{\operatorname{class}(\vect{y}) = \operatorname{class}(\vect{x})}
    .
    \label{eq:supervised_input_proba}
\end{equation}

Having tiled the manifold in which the embedding vectors $\vect{v}_{\vect{x}}$ lie with place cells, subdividing the manifold into classes is a relatively simple task. As long as the frequency of class transitions is not significantly higher than the place cell width, this can almost be posed as assigning place cells to classes. Of course, class transitions will not necessarily align with place cell transitions. To cope with this misalignment, we will add another place cell layer to our network and obtain an augmented network defined by $\vect{h}^{(k)}_{\vect{x}} = \sigma ( \mat{M}^{(k)} \vect{g}_{\vect{x}}, 1 )$, where the nonlinearity $\sigma$ is defined in \cref{eq:neustrom_nonlinearity}, and all the parameters used to compute $\vect{g}_{\vect{x}}$ are now fixed.
The only parameter to be learned is $\mat{M}^{(k)}$. The neuron corresponding to each component of $\vect{h}^{(k)}_{\vect{x}}$ forms class-specific place cells, as shown in \cref{sec:results}. We re-use the loss in \cref{eq:neustrom_loss}, replacing $p_{\text{in}} (\vect{y} | \vect{x})$ and $\vect{g}_{\vect{x}}$ by $p^{(k)} (\vect{y} | \vect{x})$ and $\vect{h}^{(k)}_{\vect{x}}$, respectively.

This problem setup differs from traditional supervised learning. Here, we have access to a large amount of unlabeled data, for which we build our  Neural Nystr\"om representations. Then, a small amount of labeled data is presented and we leverage the learned representations to solve a (possibly multi-class) classification problem.

We can interpret this augmented network structure as transfer learning. We first build the Neural Nystr\"om representations in an unsupervised fashion (from structure of the data itself) and then re-utilize these layers as a foundation for a lean supervised layer. The re-usability of the exact same computational motif (i.e., learning mechanism) for supervised and unsupervised tasks, coupled with transfer learning, is appealing as a biologically inspired model.

\begin{figure}
    \centering
    
    \includegraphics[width=.97\linewidth]{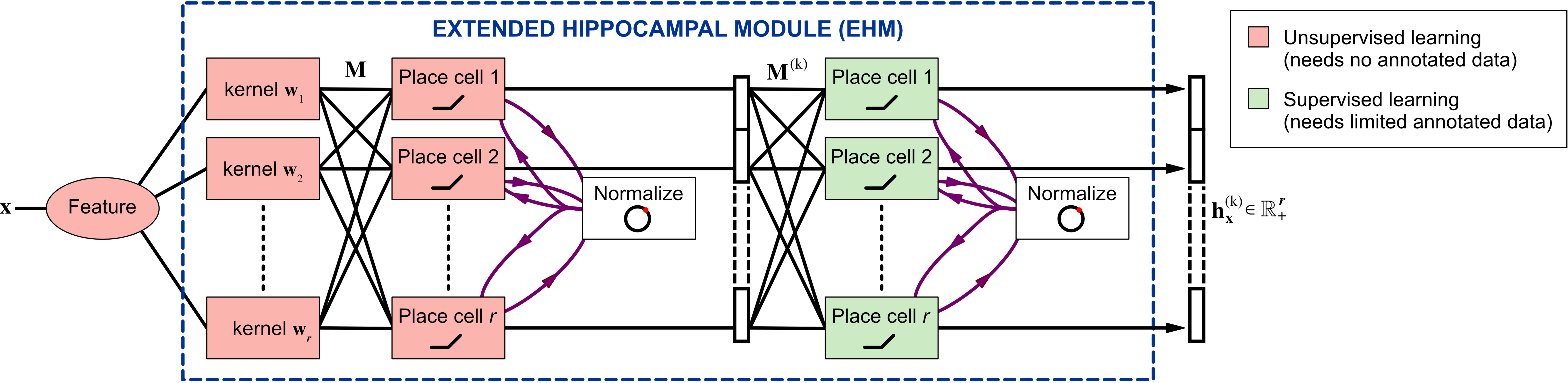}
        
    \caption{Augmented Neural Nystr\"om architecture for supervised learning. The red neurons are trained in an unsupervised fashion to represent conditional probability distributions. They are then used as the building blocks for a lean classification layer, in green.}
    \label{fig:neustrom_architecture_extended}
\end{figure}

The proposed method, as hinted above, only requires a few samples per place cell to correctly approximate the conditional probability in \cref{eq:supervised_input_proba}. This is easy to see in the assign-place-cells-to-classes scenario, where only one sample is sufficient to assign each cell. For more complicated boundaries, the number of samples needed increases gracefully, as shown in \cref{sec:results}.

\section{Episodic online computing}
\label{sec:episodic}

\begin{figure}
\begin{minipage}[t]{.50\textwidth}
\begin{algorithm2e}[H]
    \caption{Stochastic optimization for finite-sized dataset $\{ \vect{x}_i \}_{i=1}^{n}$ (see \cref{sec:supervised}).}
    \label{algo:neustrom_algorithm_fixed_large}

    \begin{small}
        $\vect{c} \gets \vect{0}$;\quad
        $\vect{c}' \gets \vect{0}$%
        \tcp*{no backprop}
        
        \ForEach{epoch}{
        
            $(\forall i) \, \operatorname{updated}(i) \gets 0$\;
            
            \ForEach{$i,j$ such that $p_{\text{in}} (\vect{x}_j | \vect{x}_i) > 0$}{
                Get $\vect{g}_i, \vect{g}_j$ from $\vect{x}_i, \vect{x}_j$%
                \tcp*{\cref{eq:neustrom_net}}

                \If{$\operatorname{updated}(i) = 0$}{
                    $\operatorname{updated}(i) \gets 1$\;
                    
                    $\vect{c} \gets \vect{c} + \vect{g}_i$%
                    \tcp*{no backprop}
                    $\vect{c}' \gets \vect{c}' + \vect{g}_i$%
                    \tcp*{no backprop}
                }
            
                Update parameters using the loss in \cref{eq:neustrom_loss_finite}\;
                % \begin{equation*}
                %   	\mathcal{L}_{ij}
                %   	= 
                %   	- p_{\text{in}} (\vect{x}_j | \vect{x}_i)
                %   	\log
                %   	% \frac{
                %   	%     \transpose{\vect{g}_i} \vect{g}_j
                %   	% }{
                %   	%     \transpose{\vect{g}_i} \vect{c}
                %   	% }
                %   	\left(
                %   	    \transpose{\vect{g}_i} \vect{g}_j
                %   	/
                %   	    \transpose{\vect{g}_i} \vect{c}
                %   	\right)
                %   	;
                % \end{equation*}
            }
            $\vect{c} \gets \vect{c}'$;\quad
            $\vect{c}' \gets \vect{0}$%
            \tcp*{no backprop}
        }
    \end{small}
\end{algorithm2e}
\end{minipage}%
\hfill%
\begin{minipage}[t]{.49\textwidth}
\begin{algorithm2e}[H]
    \caption{Stochastic optimization for episodic computing (see \cref{sec:episodic}).}
    \label{algo:neustrom_algorithm_streaming}

    \begin{small}
        $\vect{c} \gets \vect{0}$%
        \tcp*{no backprop}
    
        \ForEach{episode}{
            Select $\vect{x}_{0}$ and
            sample trajectory $\tau_{\pi, p} = [ \vect{x}_{0}, \vect{x}_{1}, \dots, \vect{x}_{T} ]$\;
    
            $\vect{c} \gets \beta^{(t)} \vect{c} + \vect{g}_{\vect{x}_{0}}$%
            \tcp*{no backprop}
            \label{algo:line:c_update}
    
            \ForEach{$t = 1, \dots, T$}{
                Get $\vect{g}_{\vect{x}_{0}}, \vect{g}_{\vect{x}_{t}}$ from $\vect{x}_0, \vect{x}_t$%
                \tcp*{\cref{eq:neustrom_net}}
                
                Update parameters using the loss
                \begin{equation*}
                    \mathcal{L}_{0t}
                    =
                   	- \gamma^t
                   	\log
                   	% \frac{
                   	%     \transpose{\vect{g}_{\vect{x}_{0}}} \vect{g}_{\vect{x}_{t}}
                   	% }{
                   	%     \transpose{\vect{g}_{\vect{x}_{0}}} \vect{c}
                   	% }
                   	\left(
                   	    \transpose{\vect{g}_{\vect{x}_{0}}} \vect{g}_{\vect{x}_{t}}
                   	/
                   	    \transpose{\vect{g}_{\vect{x}_{0}}} \vect{c}
                   	\right)
                   	;
                \end{equation*}
            }
        }
    \end{small}
\end{algorithm2e}
\end{minipage}
\end{figure}

A fair critique to the proposed method is: how does the hippocampus obtain $p_{\text{in}} (\vect{y} | \vect{x})$? We show, using the successor representation (SR) as an example, that such information is not necessary, as long as one has access to samples from such distribution. 

SR was briefly introduced in \cref{sec:intro} with \cref{eq:successor_representation_input_proba}.
Plugging \cref{eq:successor_representation_input_proba} into \cref{eq:neustrom_loss}, we get
\begin{equation}
	\mathcal{L}
	\propto 
	- \int
	\sum_{\substack{
	    \tau_{\pi, p} = [ \vect{x}_{0}, \vect{x}_{1}, \dots, \vect{x}_{T} ]
	    \\
	    \vect{x}_{0} = \vect{x}
	}}
	\sum_{t=0}^{T} \alpha^t
	\log
	\tfrac{
	    \transpose{\vect{g}_\vect{x}} \vect{g}_{\vect{x}_{t}}
	}{
	    \int \transpose{\vect{g}_\vect{x}} \vect{g}_{\vect{z}}
        \, d\vect{v}_{\vect{z}}
	}
	\, d\vect{x}
	.
    \label{eq:SR_neustrom_online}%
\end{equation}
See \cref{sec:online_step_by_step} for a step-by-step derivation of this result.
A similar approach is used in word embedding methods~\cite{mikolovDistributedRepresentationsWords2013}, where the trajectory is a word context (a portion of text) and $\alpha=1$.

\cref{algo:neustrom_algorithm_streaming} depicts an algorithm implementing online learning using \cref{eq:SR_neustrom_online}. An agent following a certain policy generates, along the way, trajectories of samples $\vect{x}_t$; we think of $\vect{x}_t$ as observations. From these observations, and their distance in time, we learn Neural Nystr\"om representations that reflect the transition distribution between them. As a side effect, place cells are created, soft-tiling the cloud of embedding vectors $\vect{v}_{\vect{x}}$.

Following \cref{sec:optimization}, we replace $\int \vect{g}_{\vect{z}} \, d\vect{v}_{\vect{z}}$ by an accumulator neuron $\vect{c}$. At each iteration, new information is added to $\vect{c}$. For finite-sized datasets, we forcefully remove older information. This is not possible in the online case (as the dataset is not finite-sized anymore) but we can still ``forget'' old information by gradually downscaling it, see \cref{algo:line:c_update} in \cref{algo:neustrom_algorithm_streaming}.
For example, we can use
$\beta^{(t)} = \left( 1 - 1 / t \right)^\rho$ where $\rho$ is a hyperparameter~\cite{mairalOnlineLearningMatrix2010}.

% \begin{algorithm2e}[t]
%     \caption{Stochastic optimization algorithm for episodic computing}
%     \label{algo:neustrom_algorithm_streaming}

%     \begin{small}
%         $\vect{c} \gets \vect{0}$%
%         \tcp*{no backprop through this operation}
    
%         \ForEach{episode}{
%             Select $\vect{x}_{0}$ and
%             sample trajectory $\tau_{\pi, p} = [ \vect{x}_{0}, \vect{x}_{1}, \dots, \vect{x}_{T} ]$\;
    
%             $\vect{c} \gets \beta^{(t)} \vect{c} + \vect{g}_{\vect{x}^{(0)}}$%
%             \tcp*{no backprop through this operation}
%             \label{algo:line:c_update}
    
%             \ForEach{$t = 1, \dots, T$}{
%                 Update parameters, using stochastic gradient descent and backpropagation,
%                 \begin{equation}
%                     \mathcal{L}_{0t}
%                     =
%                   	- \gamma^t
%                   	\log \frac{
%                   	    \transpose{\vect{g}_{\vect{x}_{0}}} \vect{g}_{\vect{x}_{t}}
%                   	}{
%                   	    \transpose{\vect{g}_{\vect{x}_{0}}} \vect{c}
%                   	}
%                   	;
%                 \end{equation}
%             }
%         }
%     \end{small}
% \end{algorithm2e}

Of course, the underlying assumption here is that the perceptual proximity between two samples $\vect{x}_{t}, \vect{x}_{t'}$, e.g., $\norm{\vect{x}_{t} - \vect{x}_{t'}}{2}$, is correlated with  their temporal proximity $|t - t'|$. However, in most natural scenarios, this assumption does not present a conceptual problem.

\cref{algo:neustrom_algorithm_streaming} clearly shows that we do not need to explicitly know $p_{\text{in}} (\vect{y} | \vect{x})$ as long as we can sample from it. We leave for future work the application of \cref{algo:neustrom_algorithm_streaming} within a reinforcement learning context.

\section{Experimental results}
\label{sec:results}

In general, for simplicity, we compute the input conditional probabilities as a row-normalized RBF kernel, see \cref{eq:kernel_input_proba}. Implementation specifications and details are provided in \cref{sec:implementation}.

We show with two synthetic examples that Neural Nystr\"om representations successfully approximate the input conditional probabilities, see \cref{fig:synthetic_unsupervised}. In both cases, the receptive field of each place cell is a localized sparse vector.
An additional result, showing that Neural Nystr\"om representations work in high-dimensional spaces is provided in \cref{fig:teapot}, \cref{sec:additional_results} (because of space constraints, some figures are relegated to the appendices). Furthermore, \cref{fig:teapot_nomad} in \cref{sec:additional_results} shows the same example but using \cref{eq:sdp-kmeans} to generate the input conditional probabilities.

\begin{figure}[p]
    \centering

    \begin{minipage}[t]{.03\textwidth}
        \vspace{-38pt}
        \textbf{A}
    \end{minipage}%
    \hspace{4pt}%
    \begin{minipage}[t]{.9\textwidth}
    \begin{minipage}{.75\textwidth}
        \includegraphics[width=\linewidth]{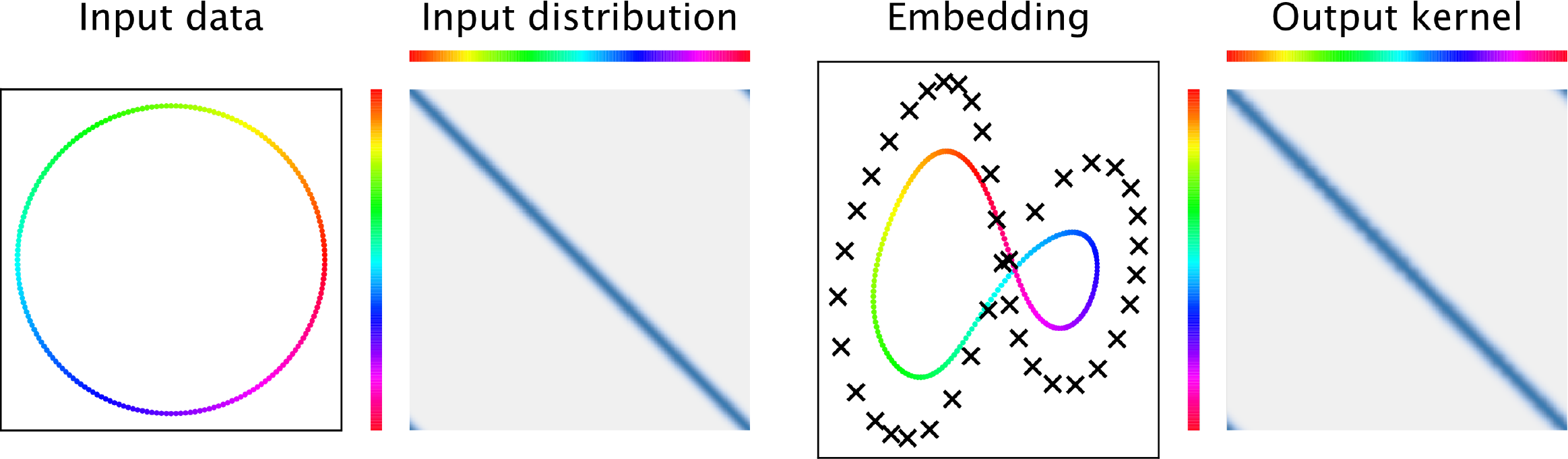}
    \end{minipage}%
    \hfill%
    \begin{minipage}{.24\textwidth}
        \includegraphics[width=\linewidth]{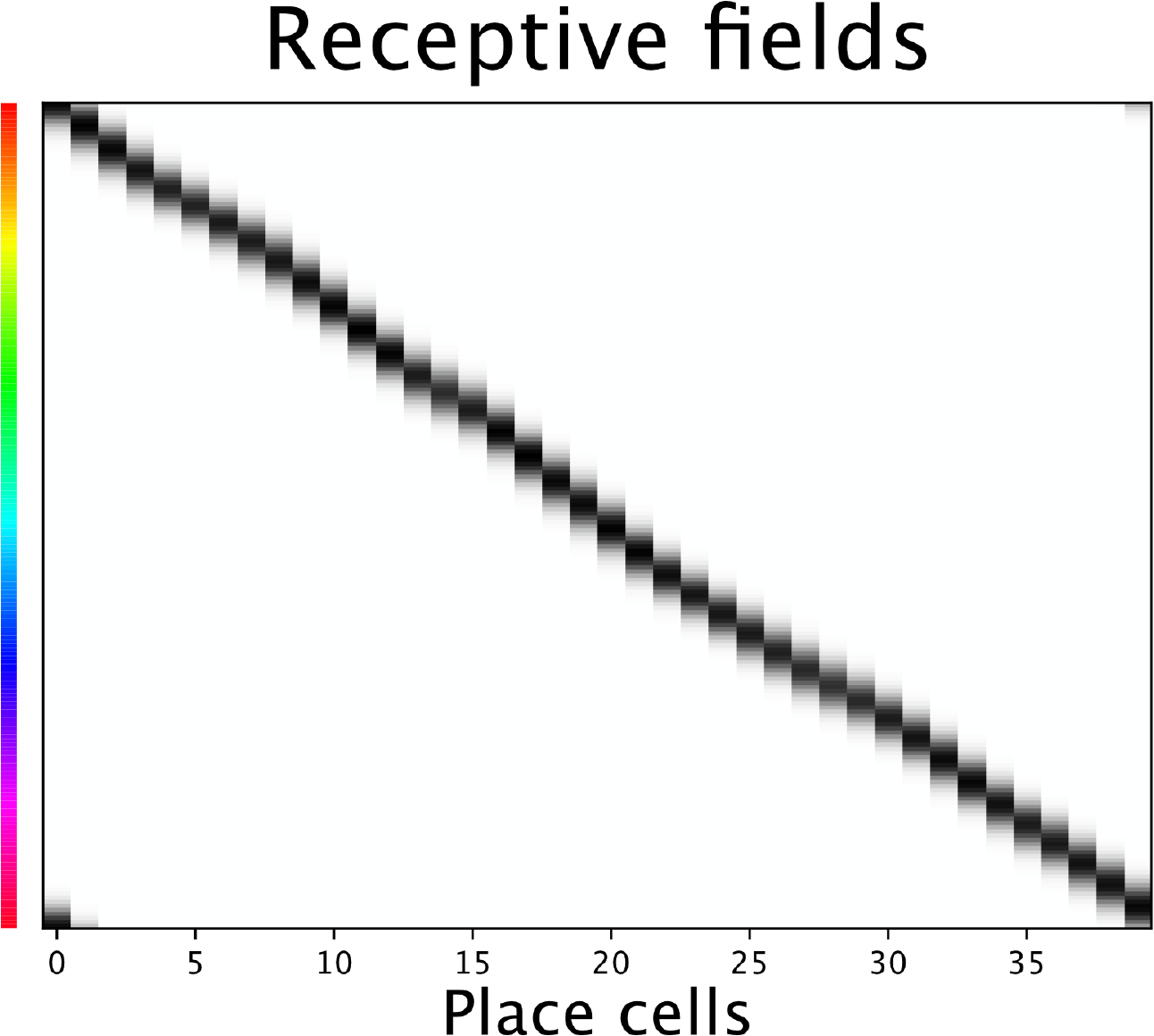}
    \end{minipage}%

    \vspace{4pt}
    
    \begin{tabu} to \textwidth {@{\hspace{0pt}} *{9}{X[c,m] @{\hspace{6pt}}} X[c,m] @{\hspace{0pt}}}
        \includegraphics[width=\linewidth]{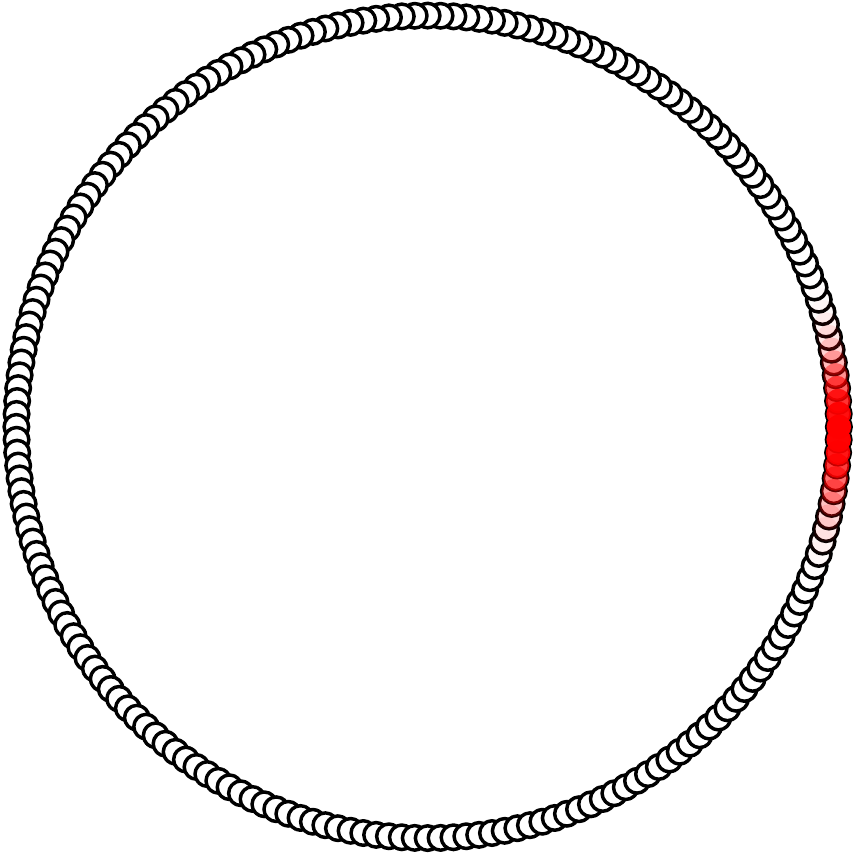}
        &
        \includegraphics[width=\linewidth]{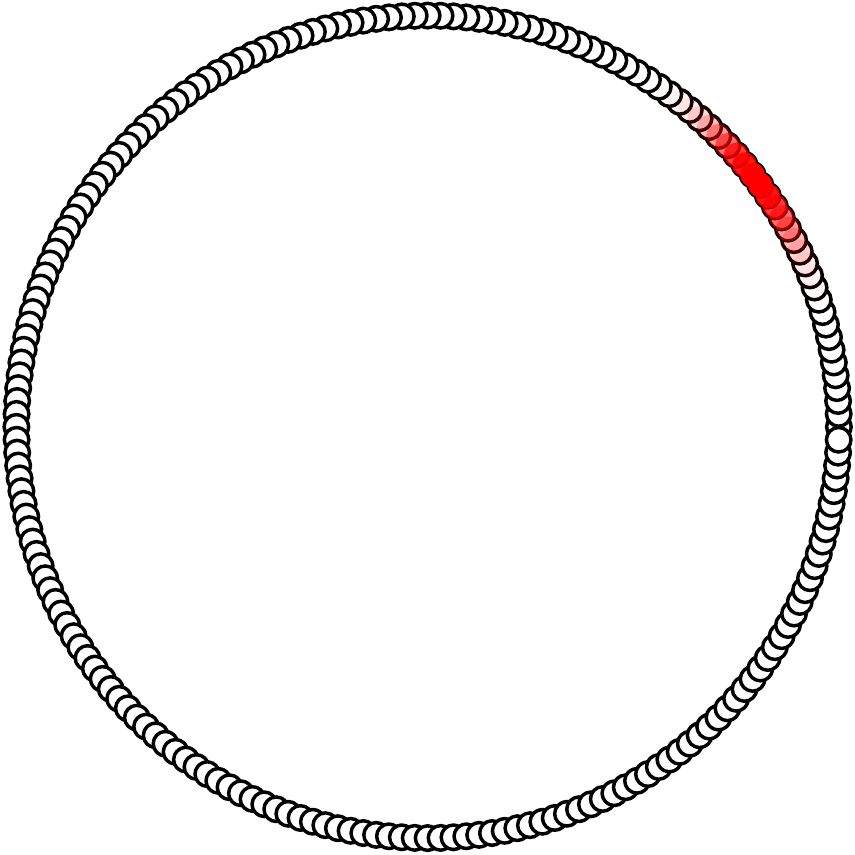}
        &
        \includegraphics[width=\linewidth]{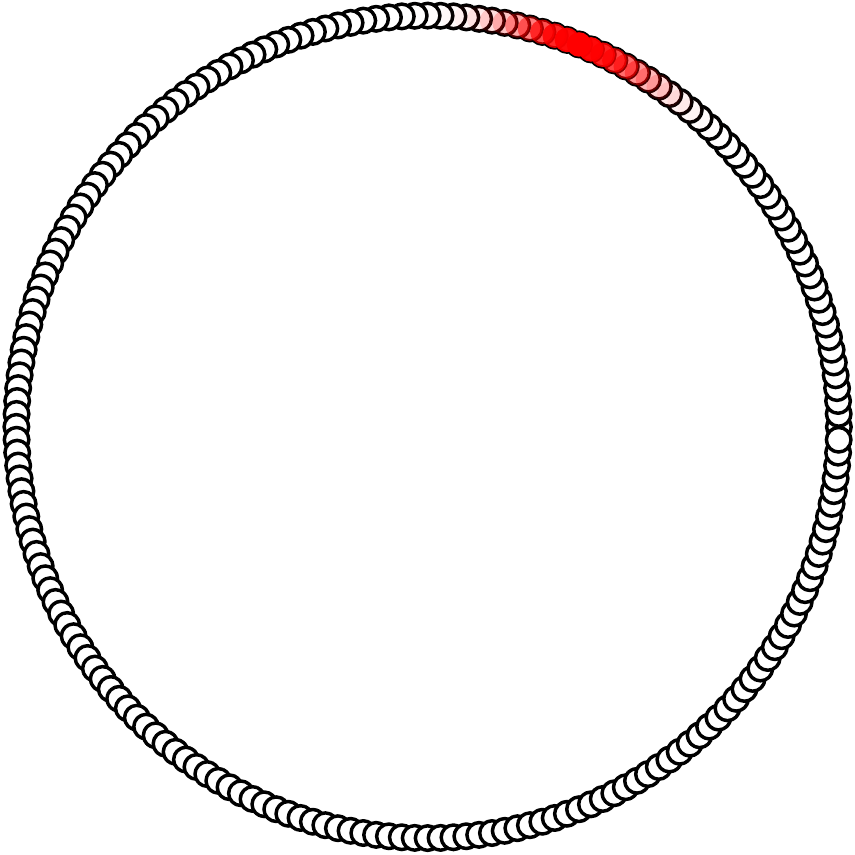}
        &
        \includegraphics[width=\linewidth]{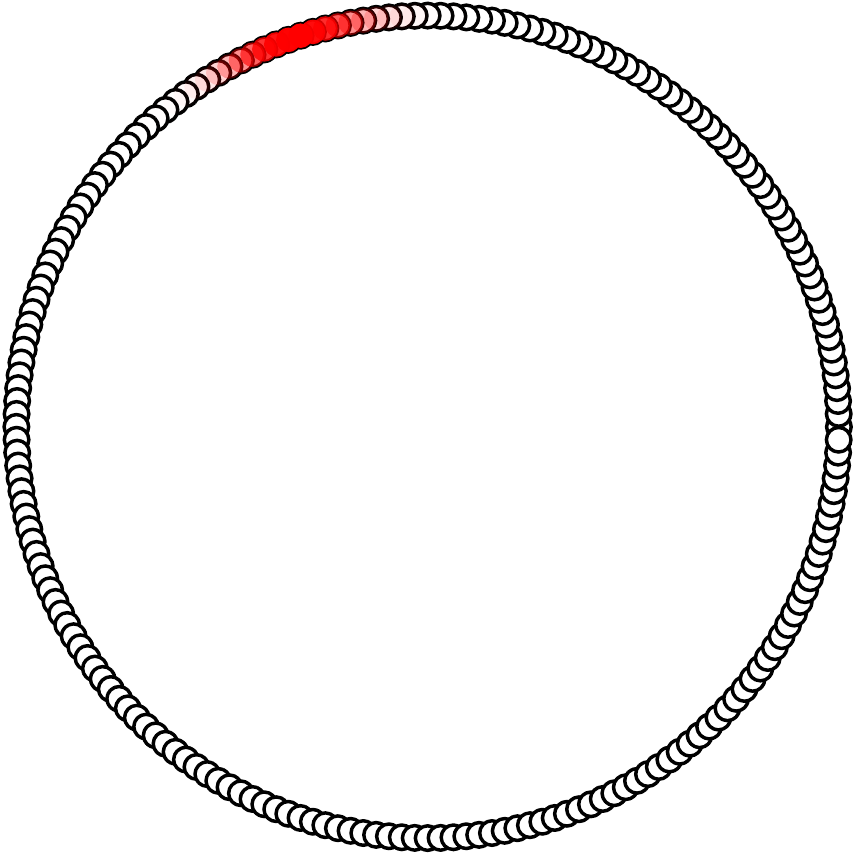}
        &
        \includegraphics[width=\linewidth]{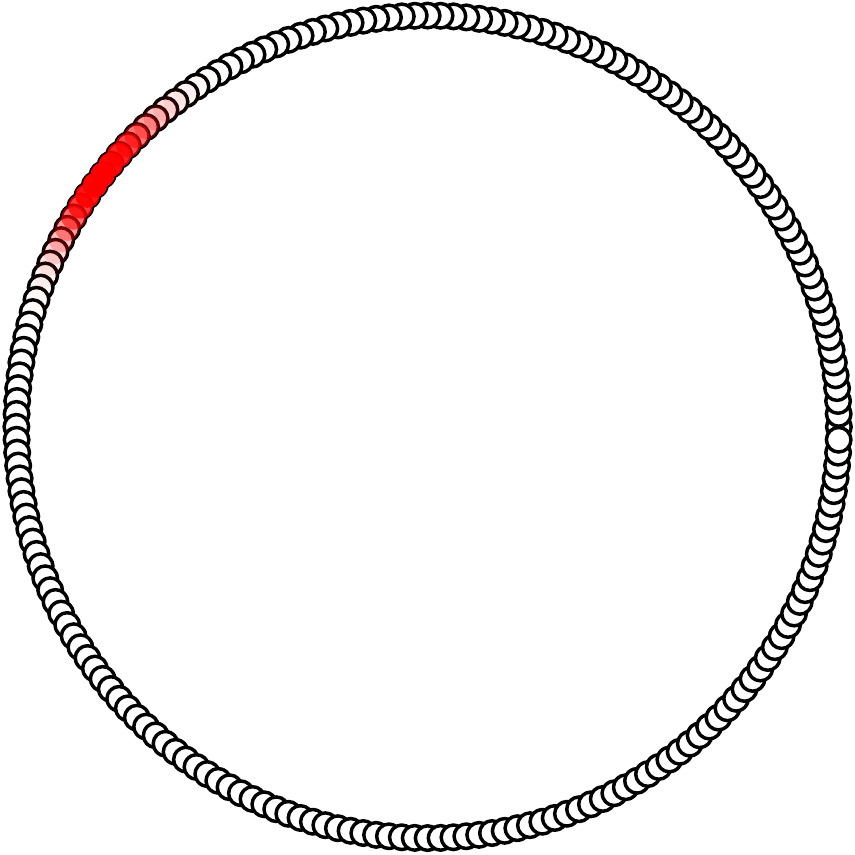}
        &
        \includegraphics[width=\linewidth]{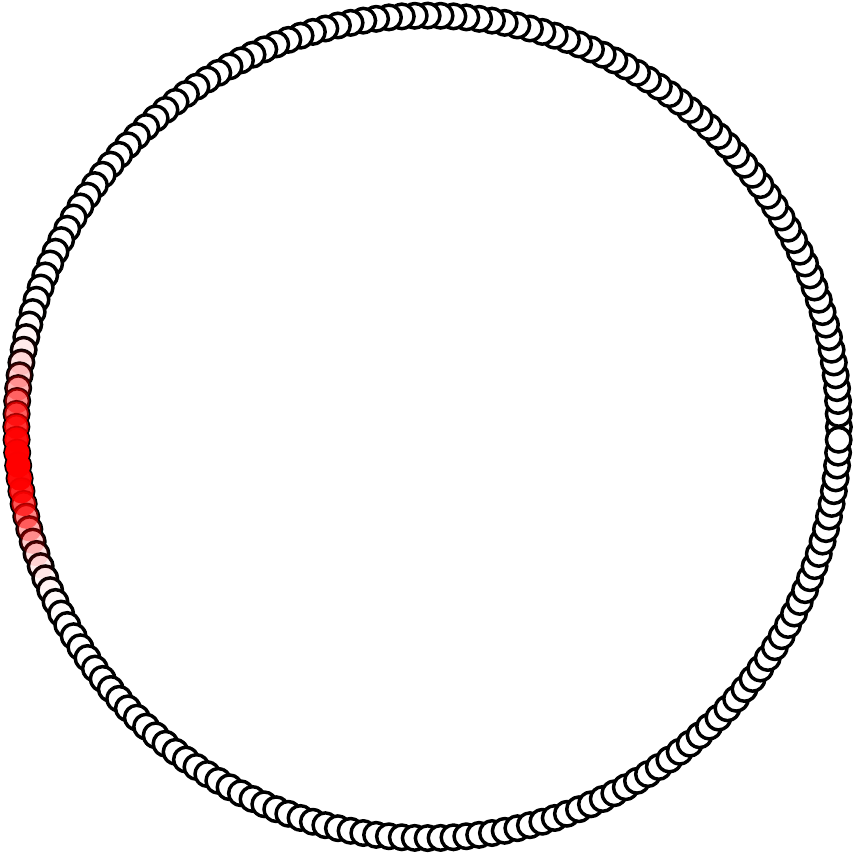}
        &
        \includegraphics[width=\linewidth]{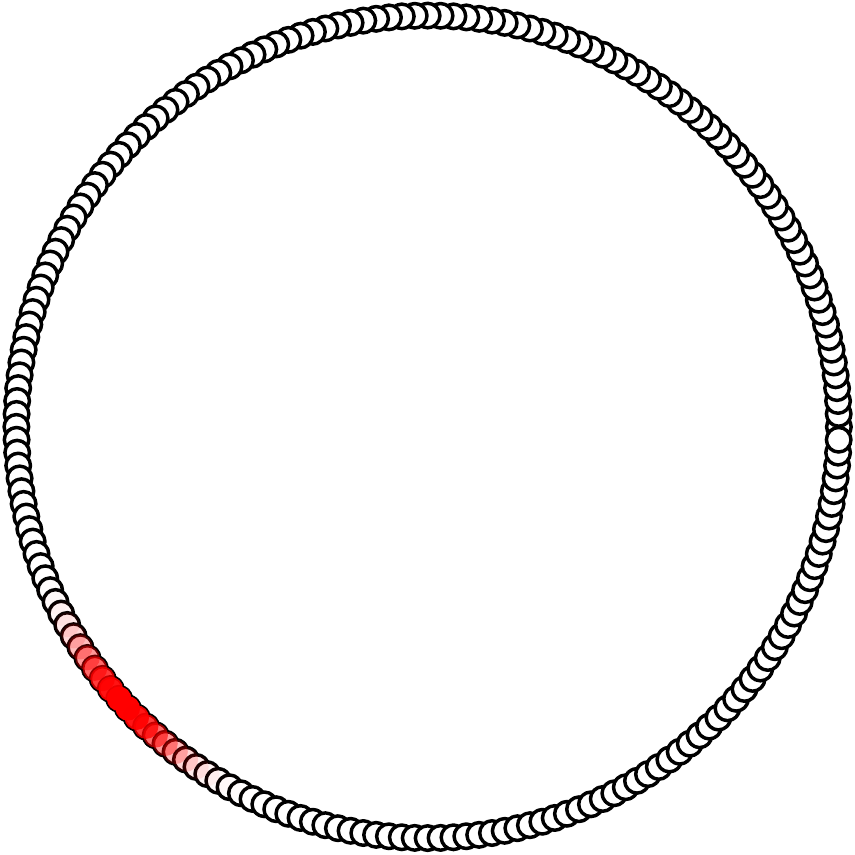}
        &
        \includegraphics[width=\linewidth]{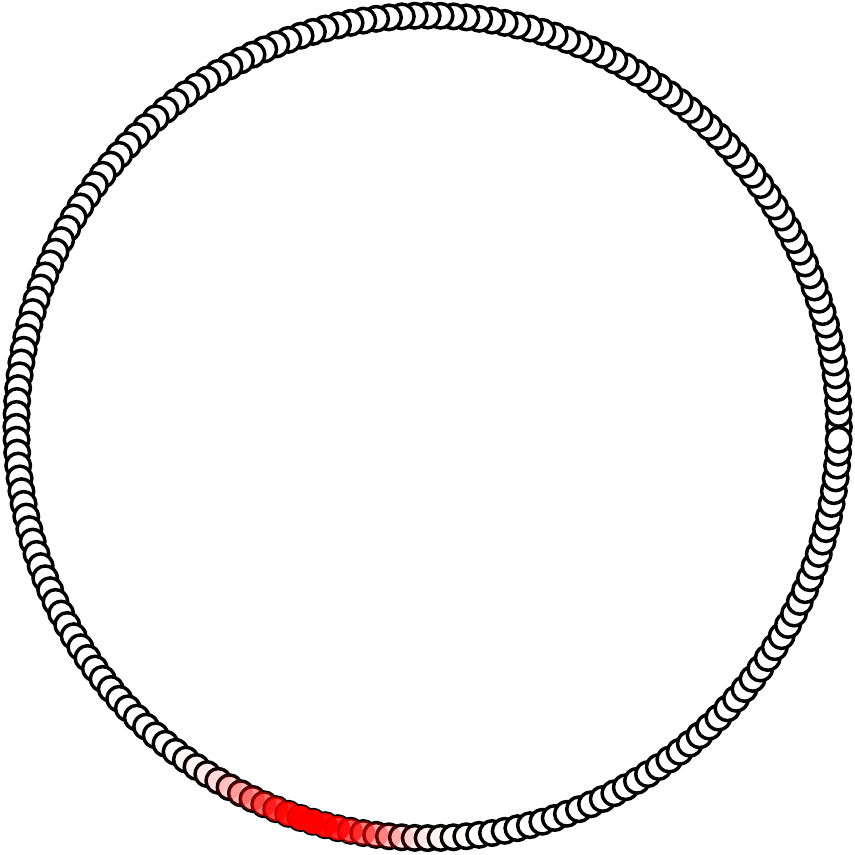}
        &
        \includegraphics[width=\linewidth]{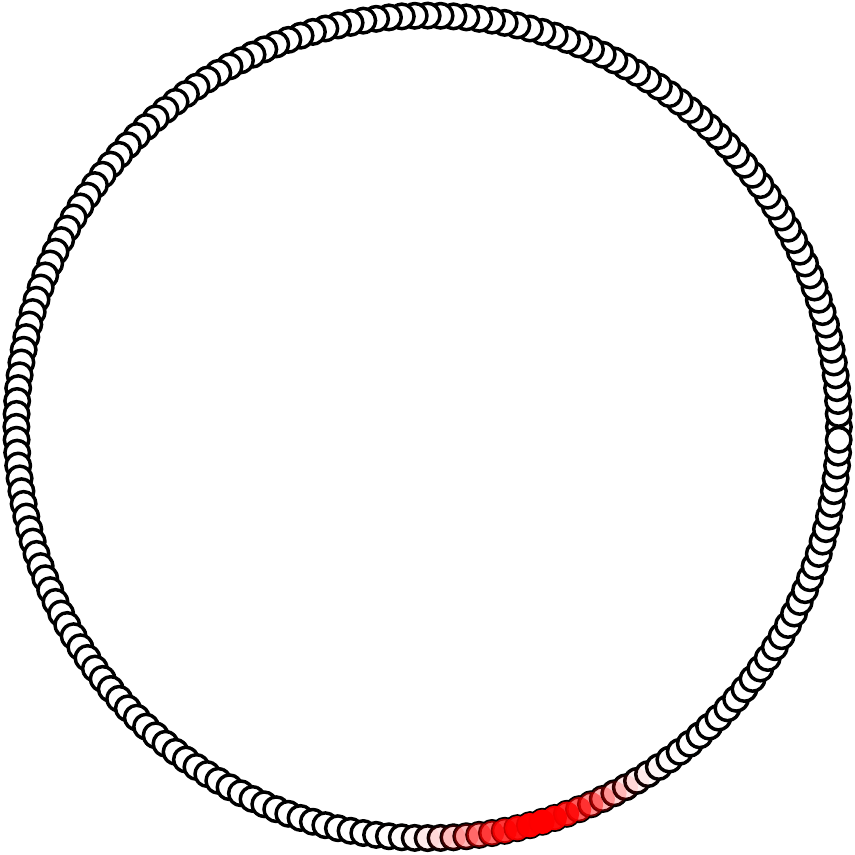}
        &
        \includegraphics[width=\linewidth]{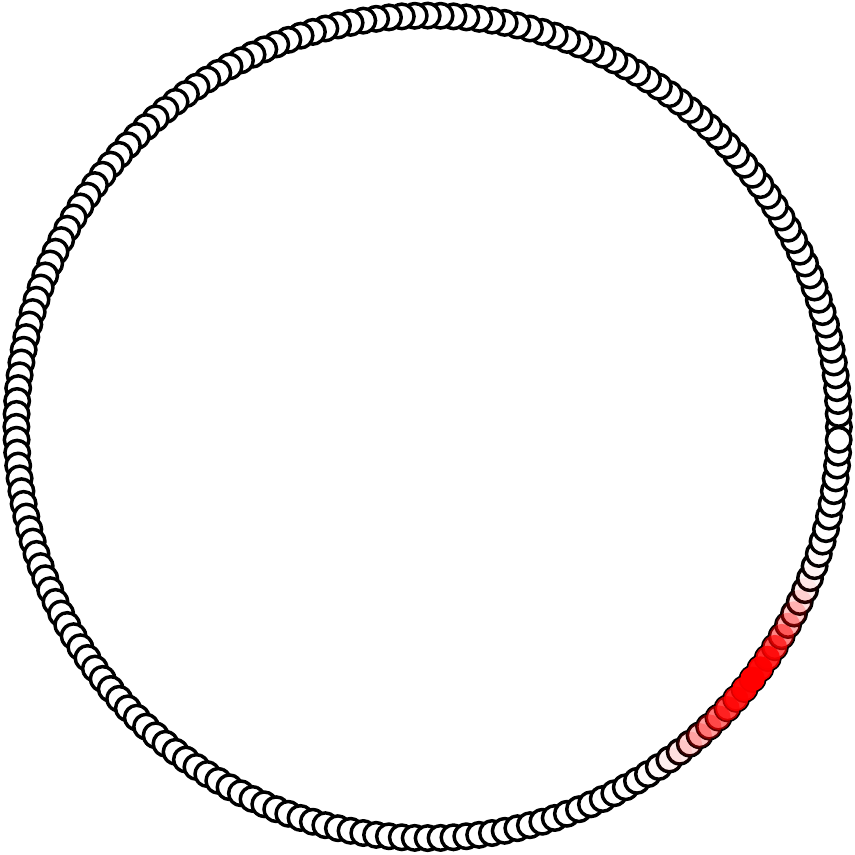}
        \\
    \end{tabu}
    \end{minipage}%
    
    \vspace{4pt}
    
    \begin{minipage}[t]{.03\textwidth}
        \vspace{-53pt}
        \textbf{B}
    \end{minipage}%
    \hspace{4pt}%
    \begin{minipage}{.9\textwidth}
    \begin{minipage}{.75\textwidth}
        \includegraphics[width=\linewidth]{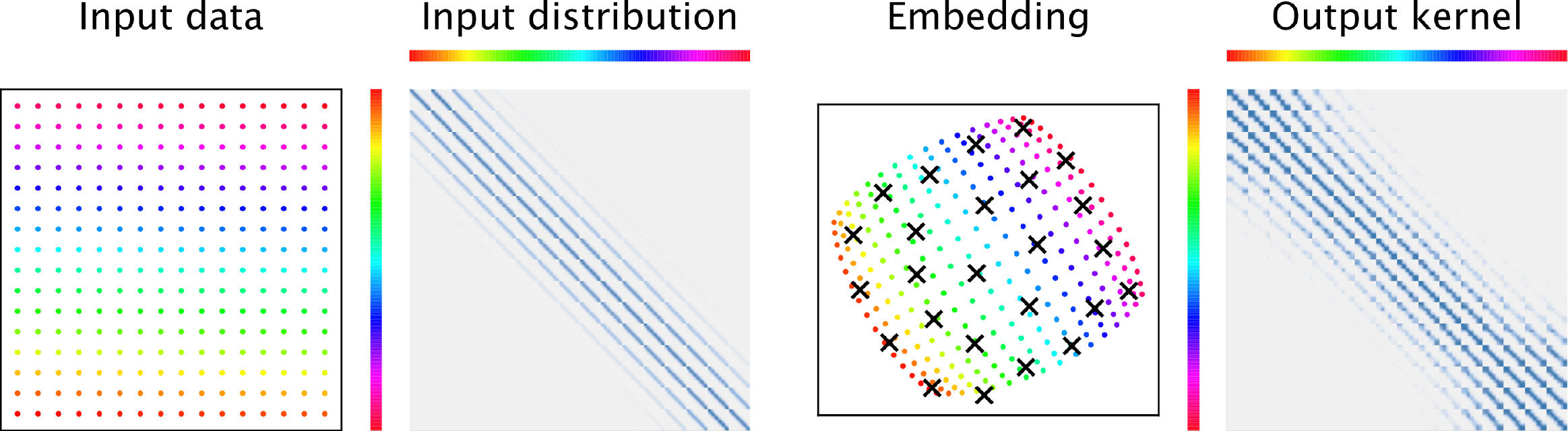}
    \end{minipage}%
    \hfill%
    \begin{minipage}{.23\textwidth}
        \includegraphics[width=\linewidth]{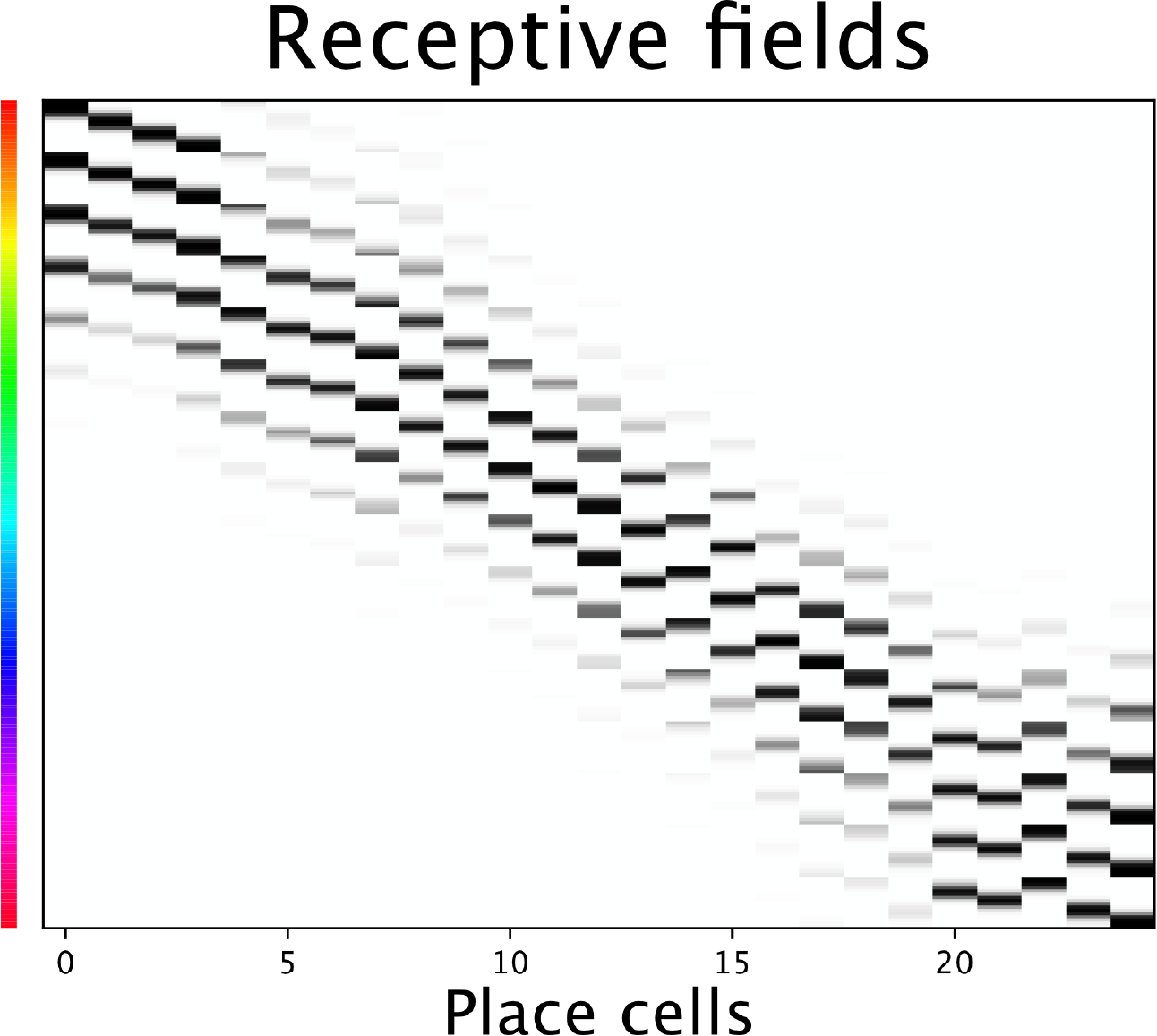}
    \end{minipage}%
    
    \vspace{4pt}
    
    \begin{tabu} to \textwidth {@{\hspace{0pt}} *{9}{X[c,m] @{\hspace{6pt}}} X[c,m] @{\hspace{0pt}}}
        \includegraphics[width=\linewidth]{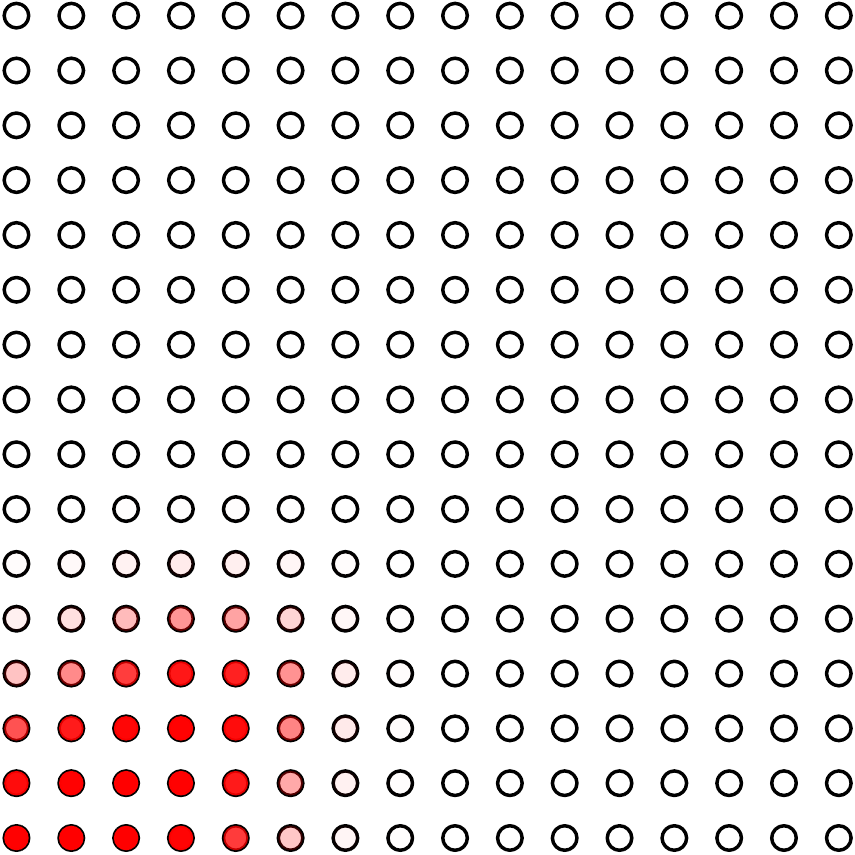}
        &
        \includegraphics[width=\linewidth]{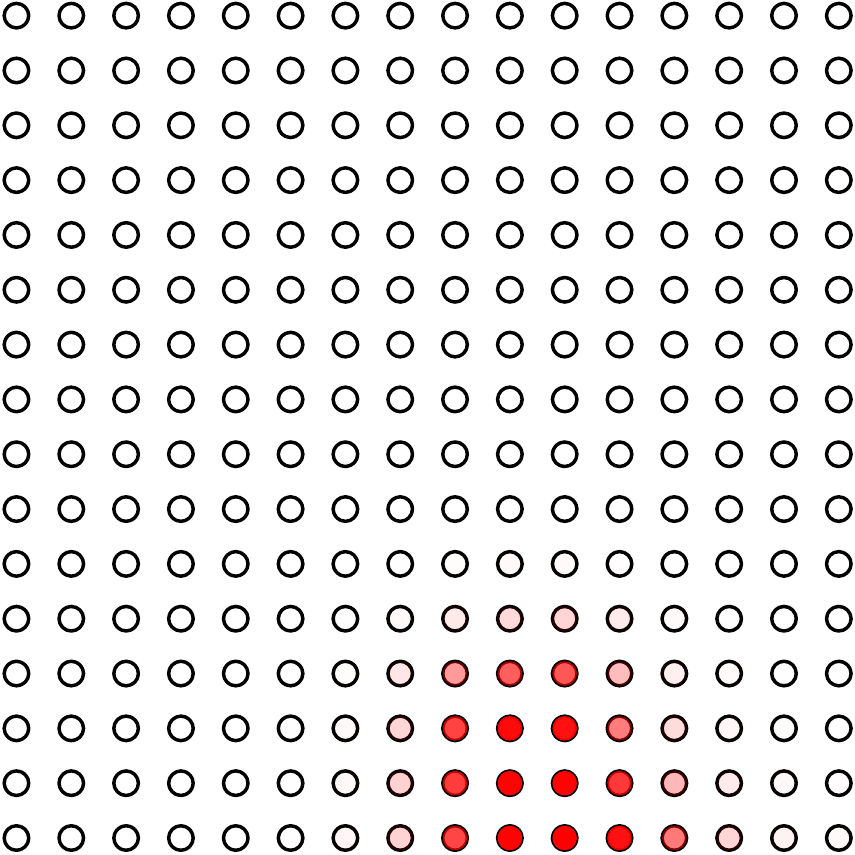}
        &
        \includegraphics[width=\linewidth]{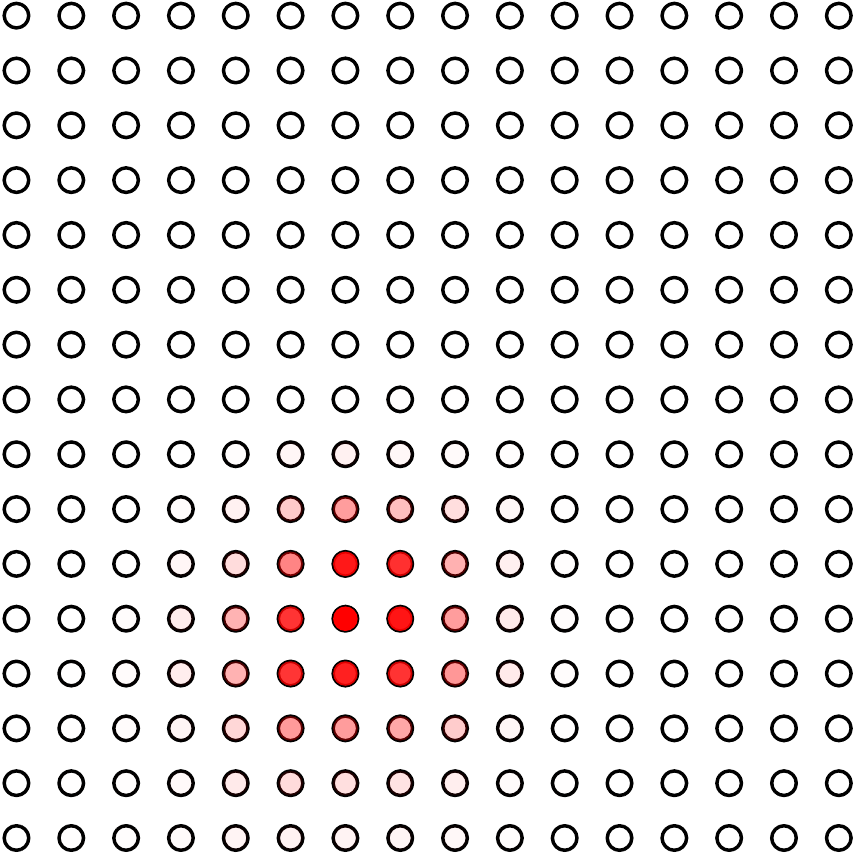}
        &
        \includegraphics[width=\linewidth]{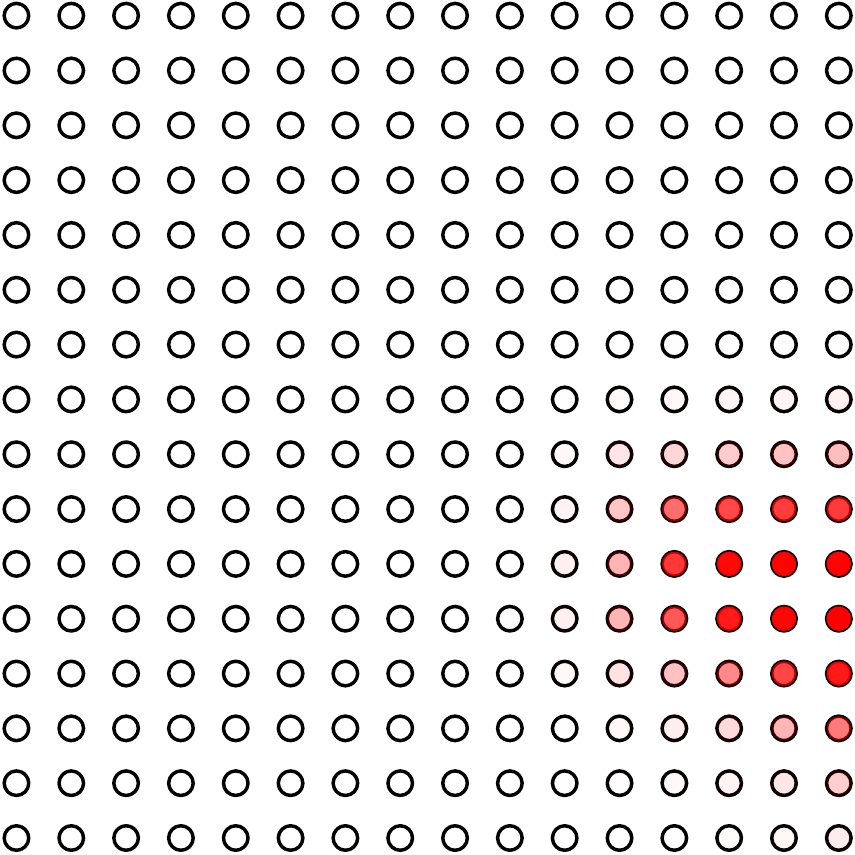}
        &
        \includegraphics[width=\linewidth]{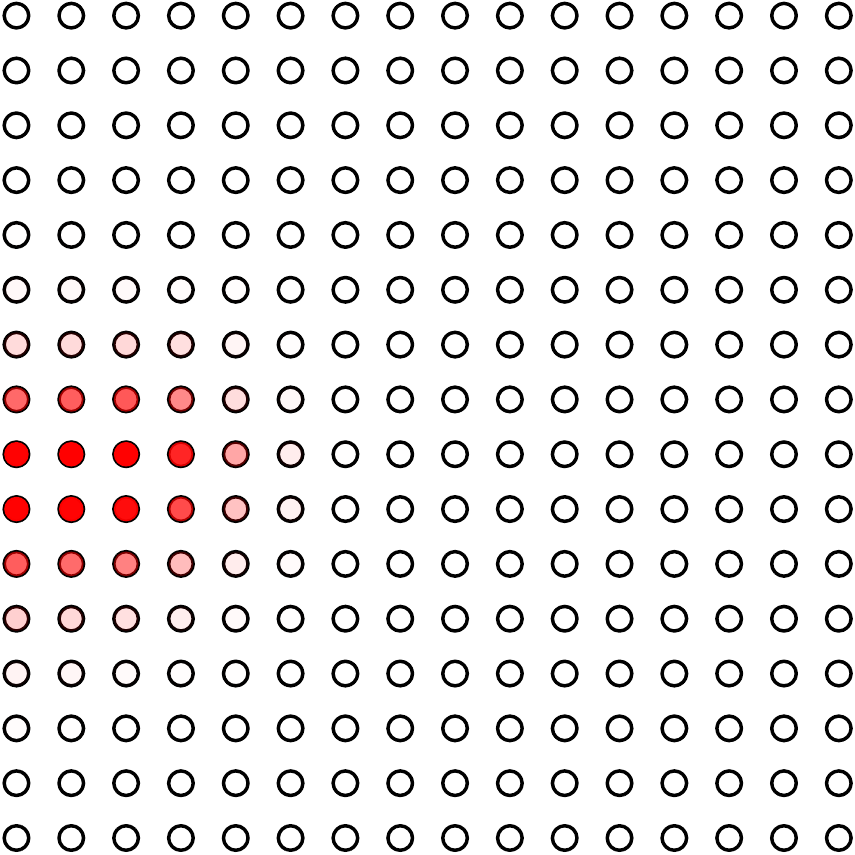}
        &
        \includegraphics[width=\linewidth]{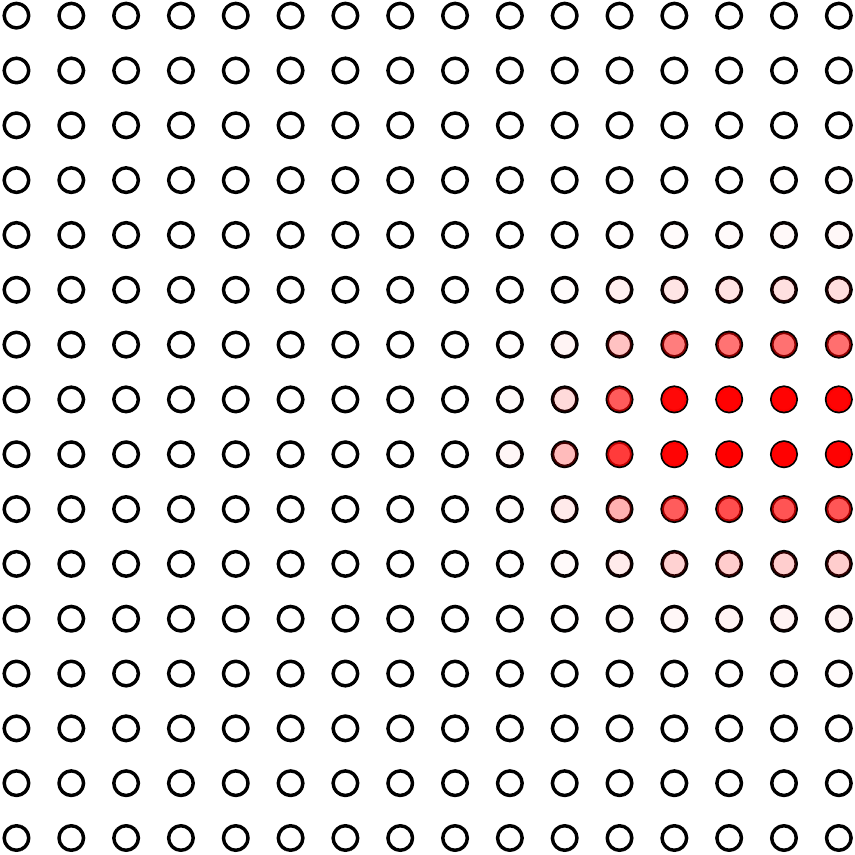}
        &
        \includegraphics[width=\linewidth]{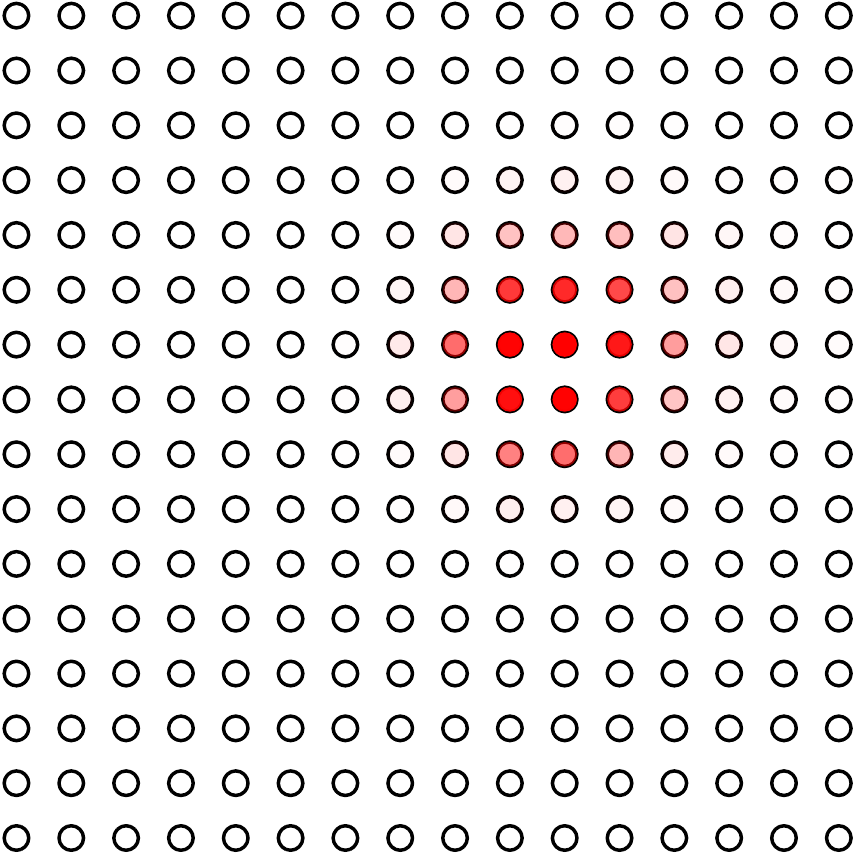}
        &
        \includegraphics[width=\linewidth]{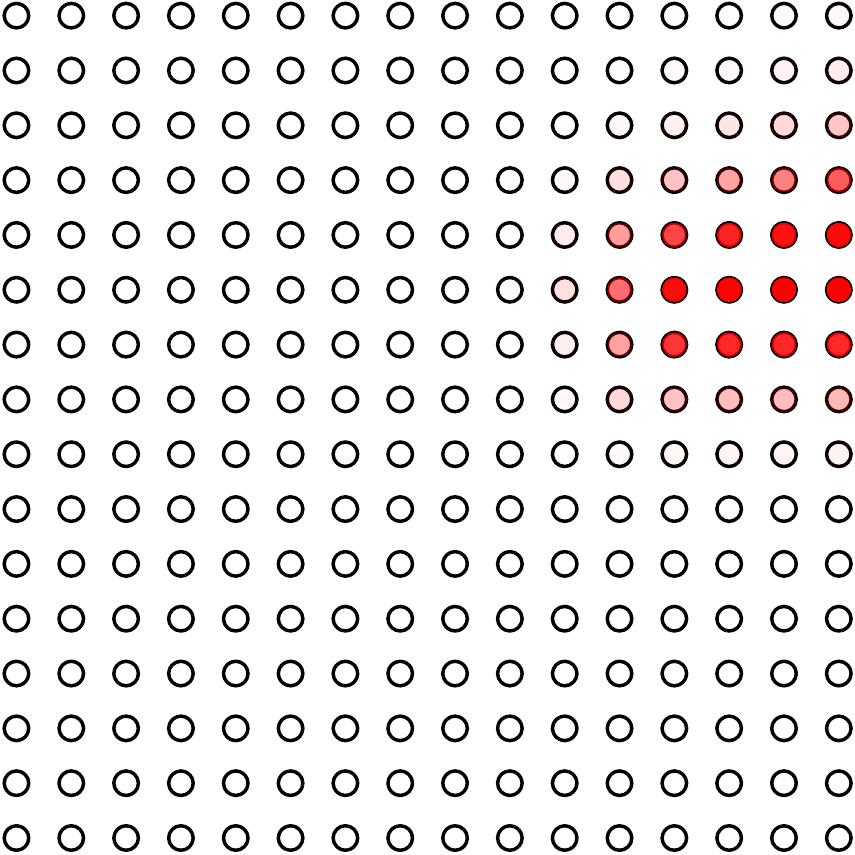}
        &
        \includegraphics[width=\linewidth]{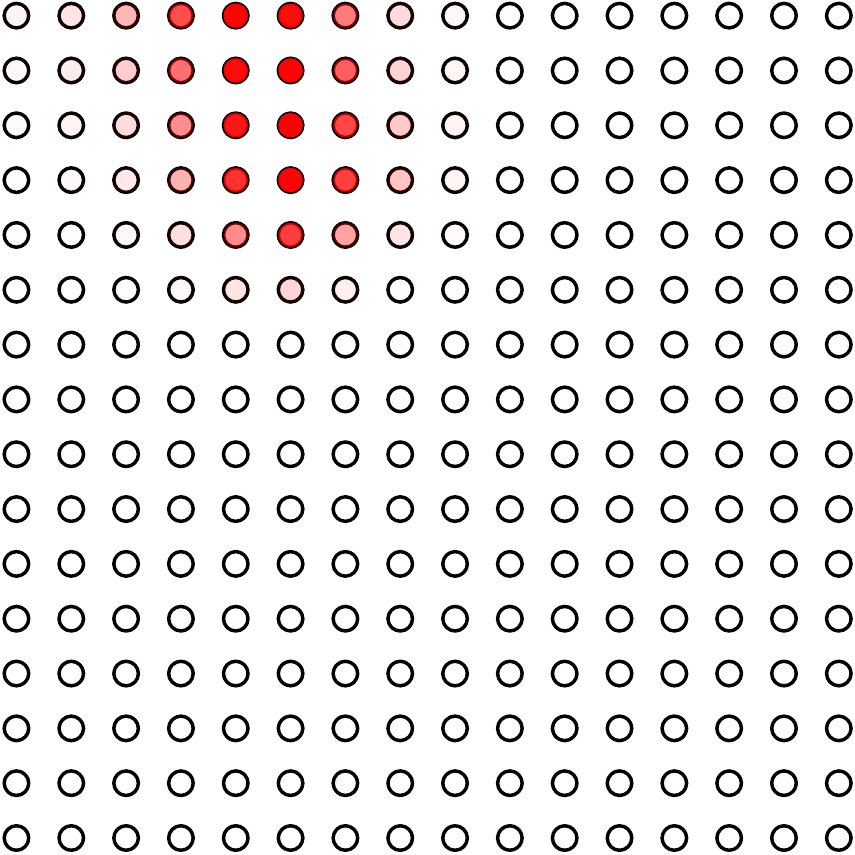}
        &
        \includegraphics[width=\linewidth]{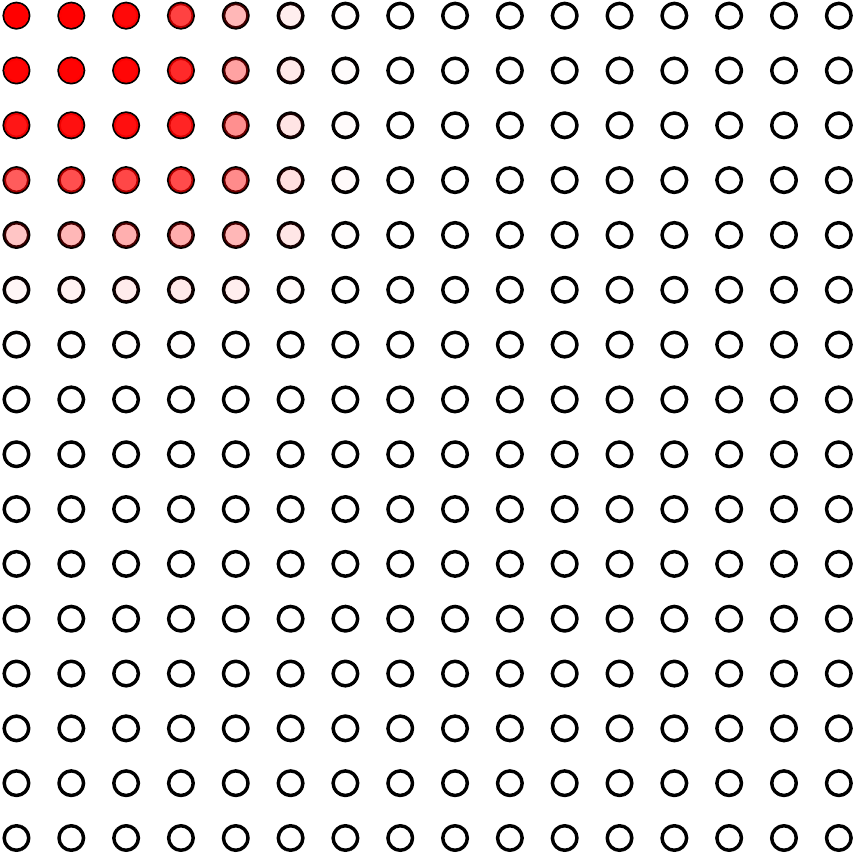}
        \\
    \end{tabu}
    \end{minipage}
    
    \caption{
        Synthetic examples show the emergence of place cells from Neural Nystr\"om representations.
        We show the two main principal components of the embedding points and of the landmarks (black crosses); this projection may twist the manifold.
        The output kernel and the receptive fields (RF) are $\mat{G} \transpose{\mat{G}}$ and $\transpose{\mat{G}}$, respectively, with $\mat{G} \coloneqq [\vect{g}_{\vect{x}_1}, \dots, \vect{g}_{\vect{x}_n}] \in \Real^{r \times n}$, see \cref{eq:neustrom_net} and \cref{fig:neustrom_arch_hip_module}. 
        The bottom row in A and B shows a few RF in red over the input points.
        Better viewed with zoom.
    }
    \label{fig:synthetic_unsupervised}
\end{figure}

When the input data lives in an union of disjoint manifolds, as in \cref{fig:circles}A, a similar pattern occurs \textit{separately} on each manifold. Moreover, Neural Nystr\"om's output kernel does not link points across manifolds, preserving manifold disentangling. 

As described in \cref{sec:supervised}, we leverage these unsupervised representations to perform supervised classification using only a relatively small percentage of annotated data.
Let $\mat{H}^{(k)} = [\vect{h}^{(k)}_{\vect{x}_1}, \dots, \vect{h}^{(k)}_{\vect{x}_n}] \in \Real^{r \times n}$ (recall $r$ is the number of landmarks). In \cref{fig:circles}B, the spectrum of $\mat{H}^{(k)}$ is dominated by only 2 components (in \cref{fig:circles_additional}, \cref{sec:additional_results}, we see the emergence of class-specific place cells). The supervised output kernel, $\mat{H}^{(k)} \transpose{\mat{H}^{(k)}}$, is successful in retrieving the outer product of the labels, even when trained using annotations on just 20\% of the available data.

\begin{figure}[p]
    \centering
    
    \begin{tabu} to .9\textwidth {@{\hspace{0pt}} X[2,c] @{\hspace{4pt}} X[30,c,m] @{\hspace{20pt}} X[2,c] @{\hspace{4pt}} X[20,c,m] @{\hspace{0pt}}}
        \vspace{-33pt}
        \textbf{A}
        &
        \includegraphics[width=\linewidth]{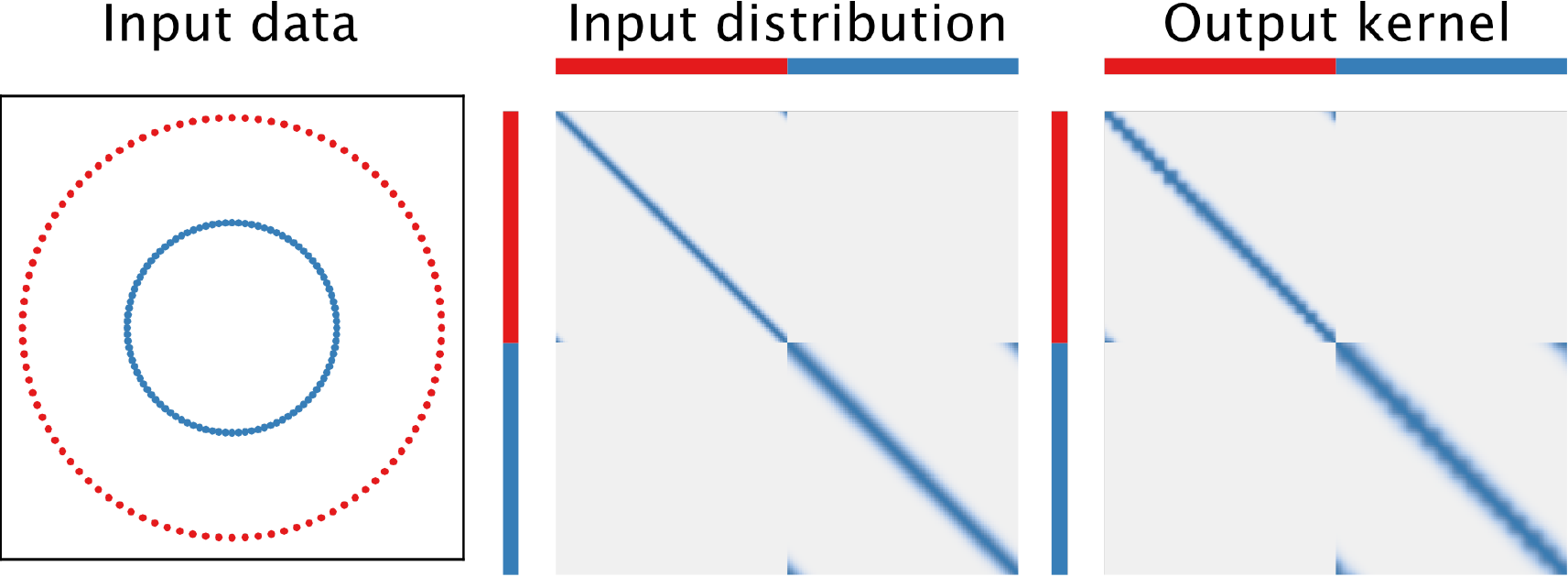}
        &
        \vspace{-33pt}
        \textbf{B}
        &
        \includegraphics[width=\linewidth]{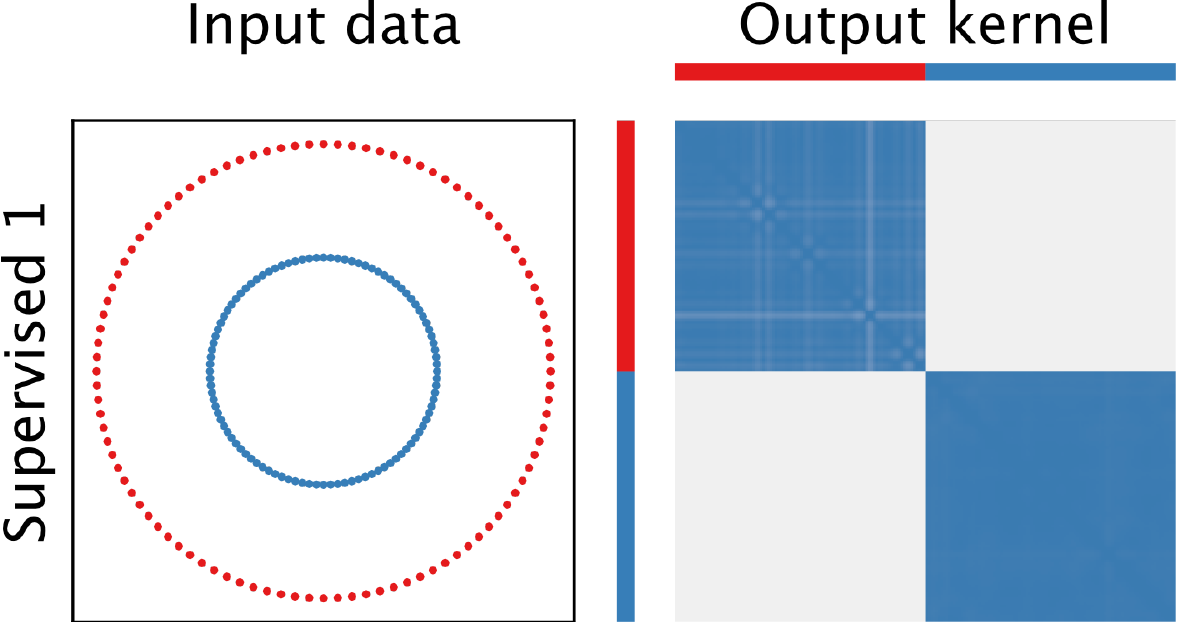}
        \\
        \\[-6pt]
        
        &
        \begin{minipage}{.44\linewidth}
        \includegraphics[width=\linewidth]{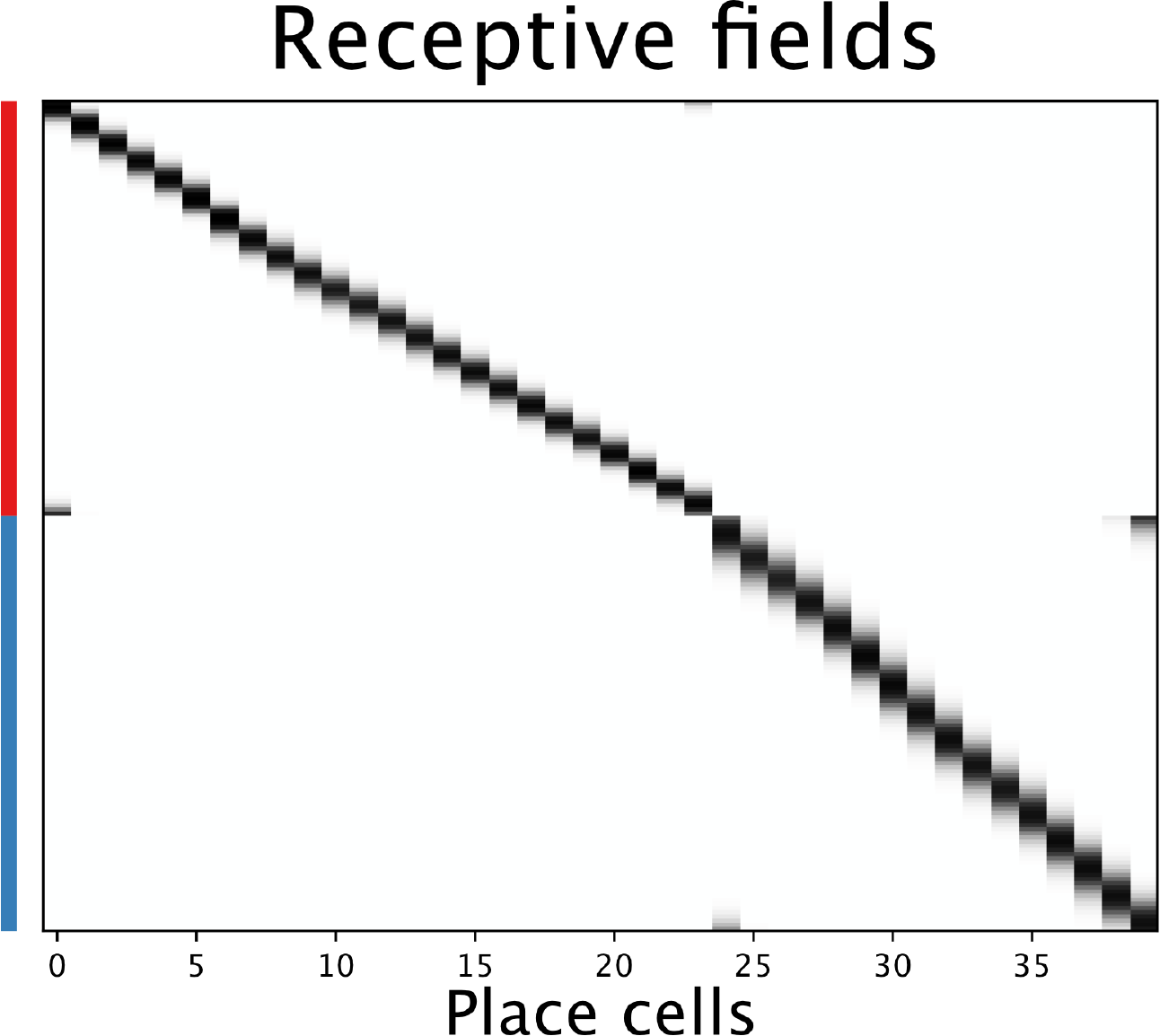}
        \end{minipage}%
        \hfill%
        \begin{minipage}{.52\linewidth}
        \begin{tabu} to \textwidth {@{\hspace{0pt}} *{2}{X[c,m] @{\hspace{2pt}}} X[c,m] @{\hspace{0pt}}}
            \includegraphics[width=\linewidth]{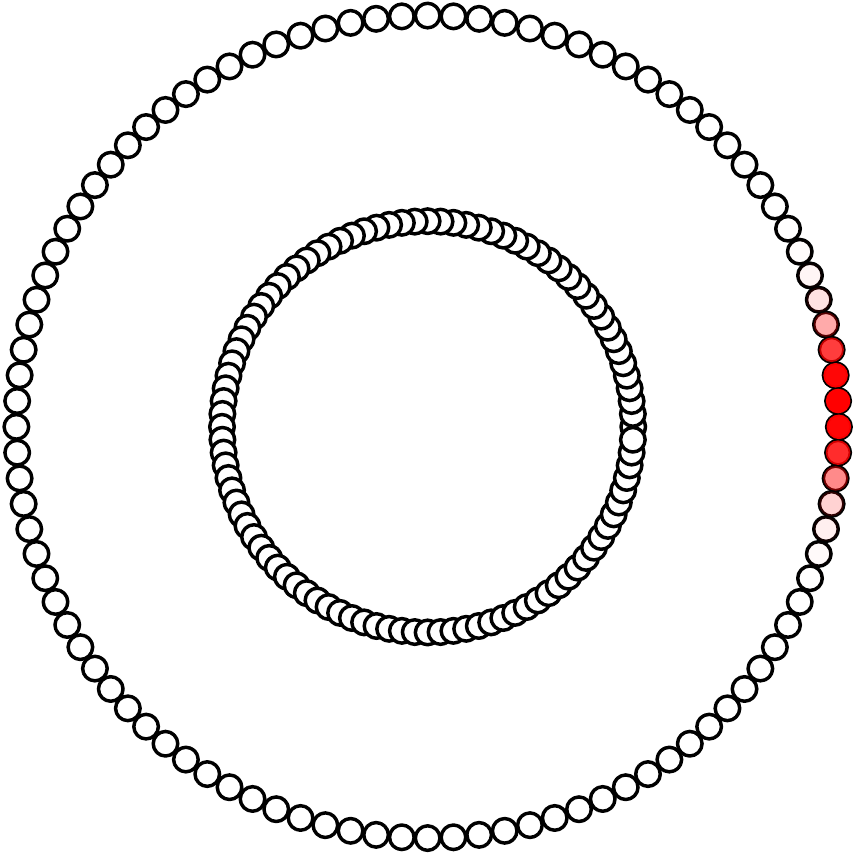}
            &
            \includegraphics[width=\linewidth]{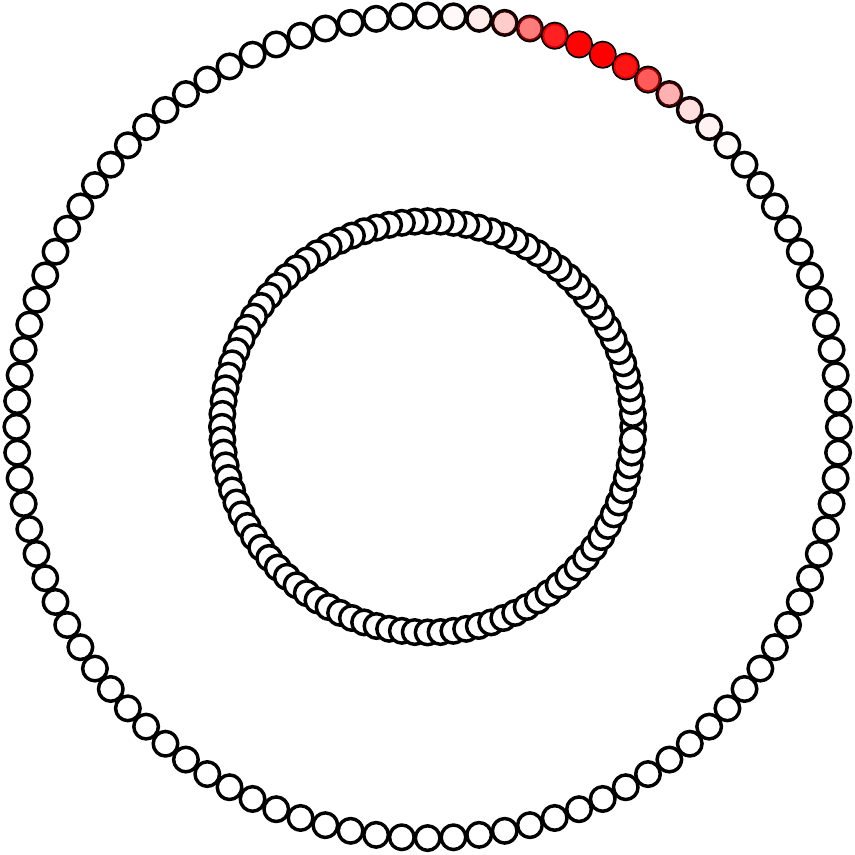}
            &
            \includegraphics[width=\linewidth]{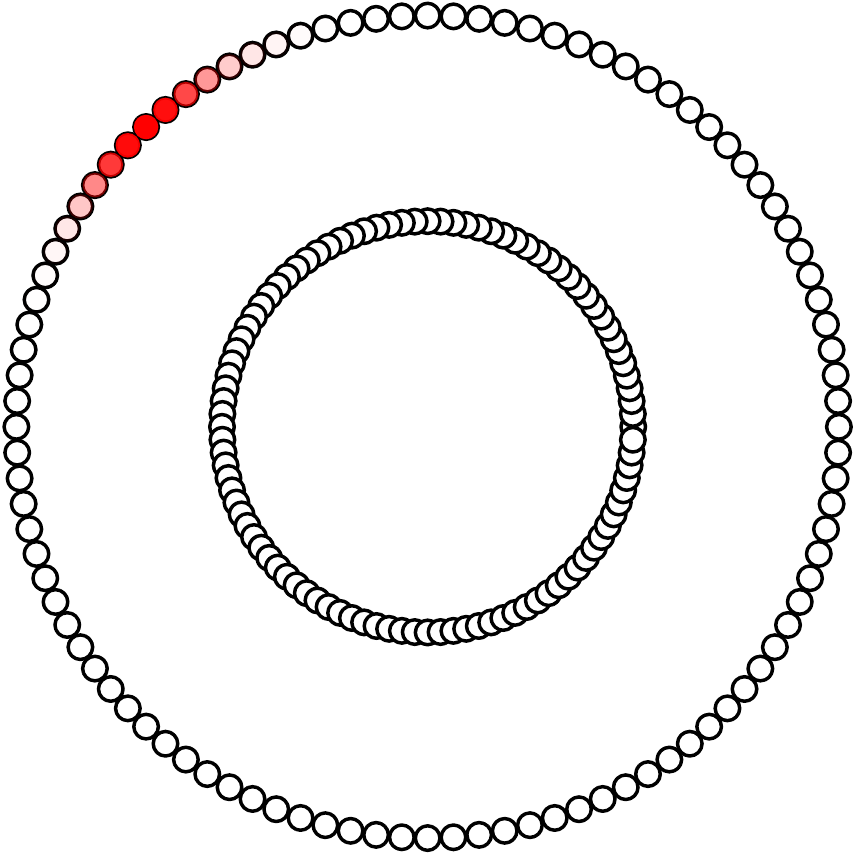}
            % &
            % \includegraphics[width=\linewidth]{offline_synthetic/circles_land40_unsupervised_rfa2_pc12}
            % &
            % \includegraphics[width=\linewidth]{offline_synthetic/circles_land40_unsupervised_rfa2_pc16}
            \\
            \\[-10pt]
            \includegraphics[width=\linewidth]{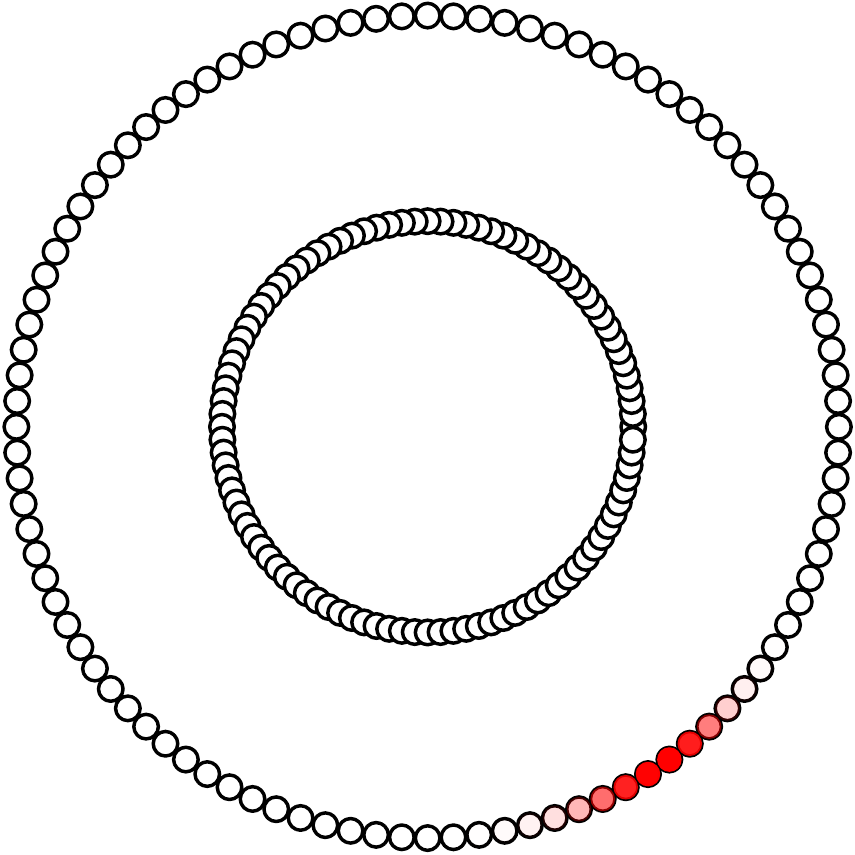}
            &
            \includegraphics[width=\linewidth]{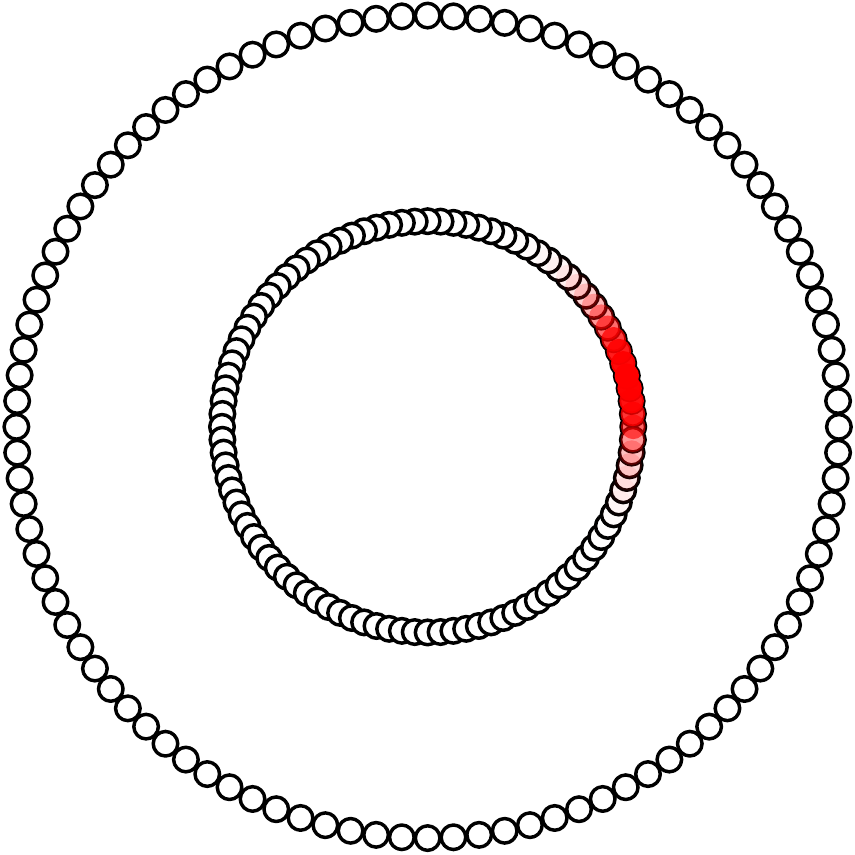}
            &
            \includegraphics[width=\linewidth]{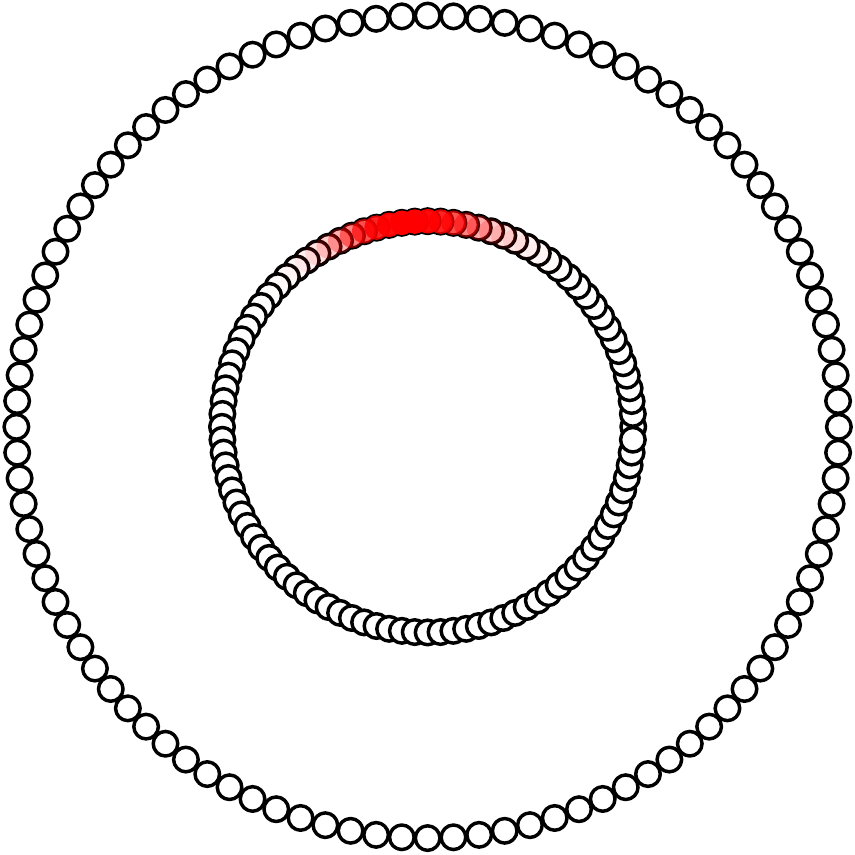}
            % &
            % \includegraphics[width=\linewidth]{offline_synthetic/circles_land40_unsupervised_rfa2_pc32}
            % &
            % \includegraphics[width=\linewidth]{offline_synthetic/circles_land40_unsupervised_rfa2_pc36}
            \\
        \end{tabu}
        \end{minipage}
        &
        &
        \includegraphics[width=\linewidth]{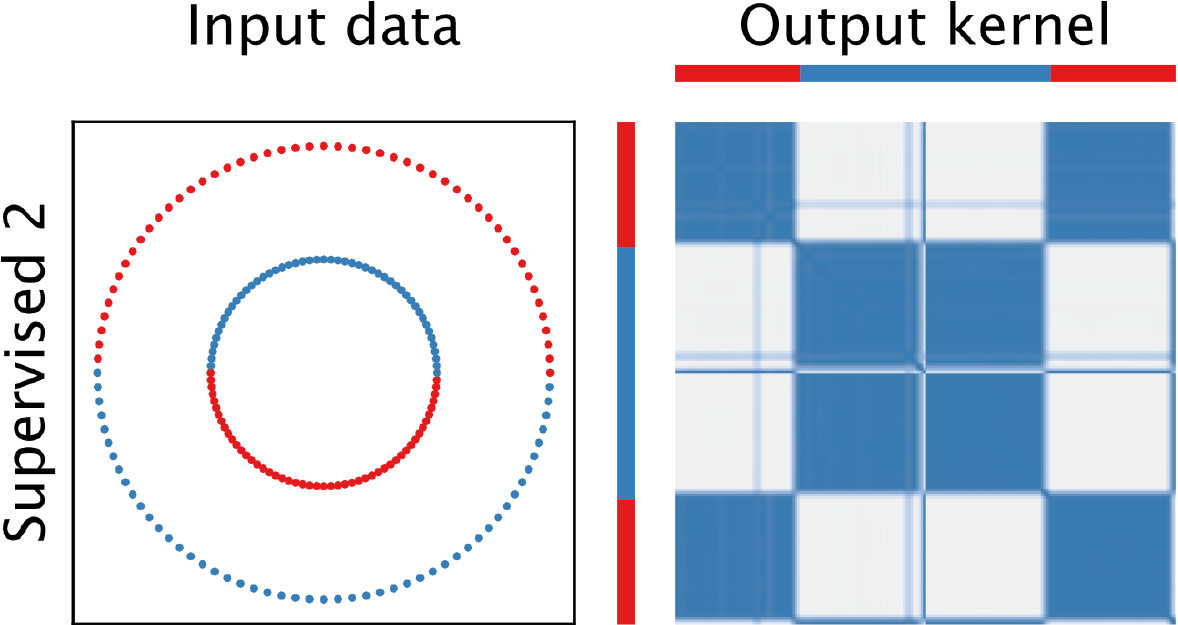}
        \\
    \end{tabu}
    
    \caption{
        (A) Neural Nystr\"om representations disentangle manifolds. See \cref{fig:synthetic_unsupervised} for a description of the individual plots. Better viewed with zoom.
        (B) Using these representations for supervised tasks, see \cref{sec:supervised}, we obtain good results under two sets of labels and with annotations on just $20\%$ of the available data.
        The output kernel is $\mat{H}^{(k)} \transpose{\mat{H}^{(k)}}$ with $\mat{H}^{(k)} \coloneqq [\vect{h}^{(k)}_{\vect{x}_1}, \dots, \vect{h}^{(k)}_{\vect{x}_n}] \in \Real^{r \times n}$, see \cref{fig:neustrom_architecture_extended}.
    }
    \label{fig:circles}
\end{figure}

We evaluate our supervised learning results in two ways. First, we compare the output kernel, $\mat{H}^{(k)} \transpose{\mat{H}^{(k)}}$, with the input conditional probability in \cref{eq:supervised_input_proba}, directly computing precision/recall gain (PRG) curves~\cite{flachPrecisionrecallgainCurvesPR2015} on these square matrices. Second, to compare predicted labels versus true labels, we take the matrix $\mat{H}^{(k)}$ and compute its nonnegative matrix factorization (NMF) with rank equal to the number of classes. Notice that, prior to this, the proposed method does not use the number of classes as (meta) parameter. We obtain our predicted labels by assigning each datapoint to its most dominant NMF component; we term this method NMF-HA.

In \cref{fig:supervised}, we show supervised learning results on two different datasets (output kernels and receptive fields in \cref{fig:digits_additional}, \cref{sec:additional_results}). For MNIST~\cite{lecunGradientbasedLearningApplied1998} (\cref{fig:supervised}B), we use $10^{4}$ randomly selected data points. For Digits\footnote{https://archive.ics.uci.edu/ml/datasets/Optical+Recognition+of+Handwritten+Digits} we use the subset provided in scikit-learn. To evaluate how much annotated data is needed during supervised learning, we use different amounts of training data to build the input probability in \cref{eq:supervised_input_proba}; we use the remainding annotated data for testing. In each case, we use 10 random splits to compute the mean and standard deviation of the PRG curves. Good results are observed, even when training with a small fraction of annotated data. In general, we observed that $\Theta(r)$ annotations are required for successful supervised learning.

We have already pointed out, in \cref{sec:supervised}, that our setting is different from traditional supervised classification: while we have access to abundant unlabeled data, labeled data is scarse. However, as a point of comparison, in \cref{fig:supervised} we show the results obtained with a kernel SVM classifier (using an RBF kernel). Whereas the SVM classifier takes advantage of knowing the number of classes, our method does not require such information. In \cref{fig:supervised}, we see that the proposed method outperforms the SVM classifier, when both are trained using 10\% of the data: all the SVM results lie under the Neural Nystr\"om PRG curve (better viewed with zoom).

\begin{figure}[p]
    \centering
    
    \begin{tabu} to .49\textwidth {@{\hspace{0pt}} X[5,c,m] @{\hspace{4pt}} X[80,c,m] @{\hspace{4pt}} X[70,c,m] @{\hspace{0pt}}}
        \vspace{-75pt}
        \textbf{A}
        &
        \includegraphics[width=\linewidth]{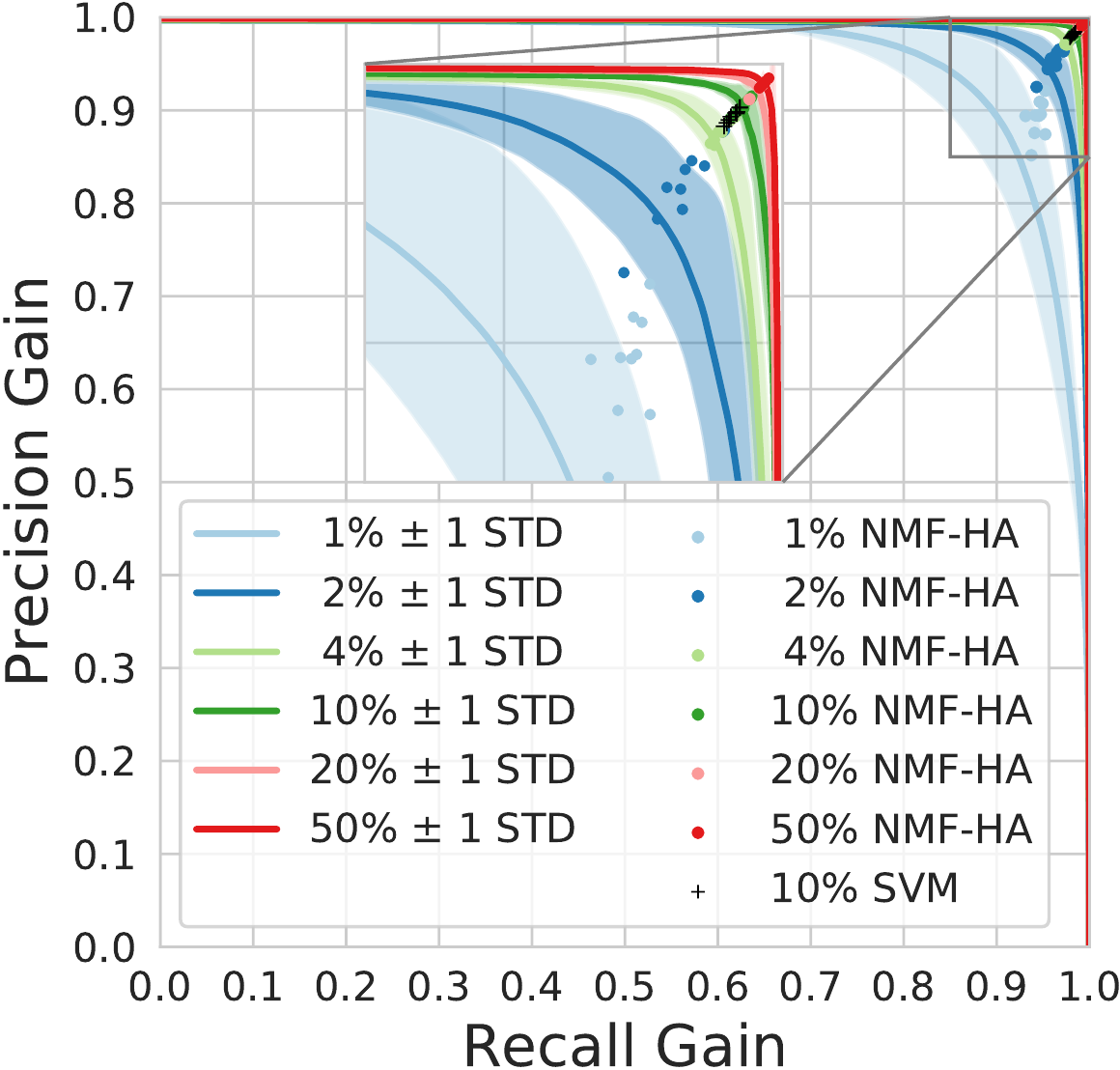}
        &
        \includegraphics[width=\linewidth]{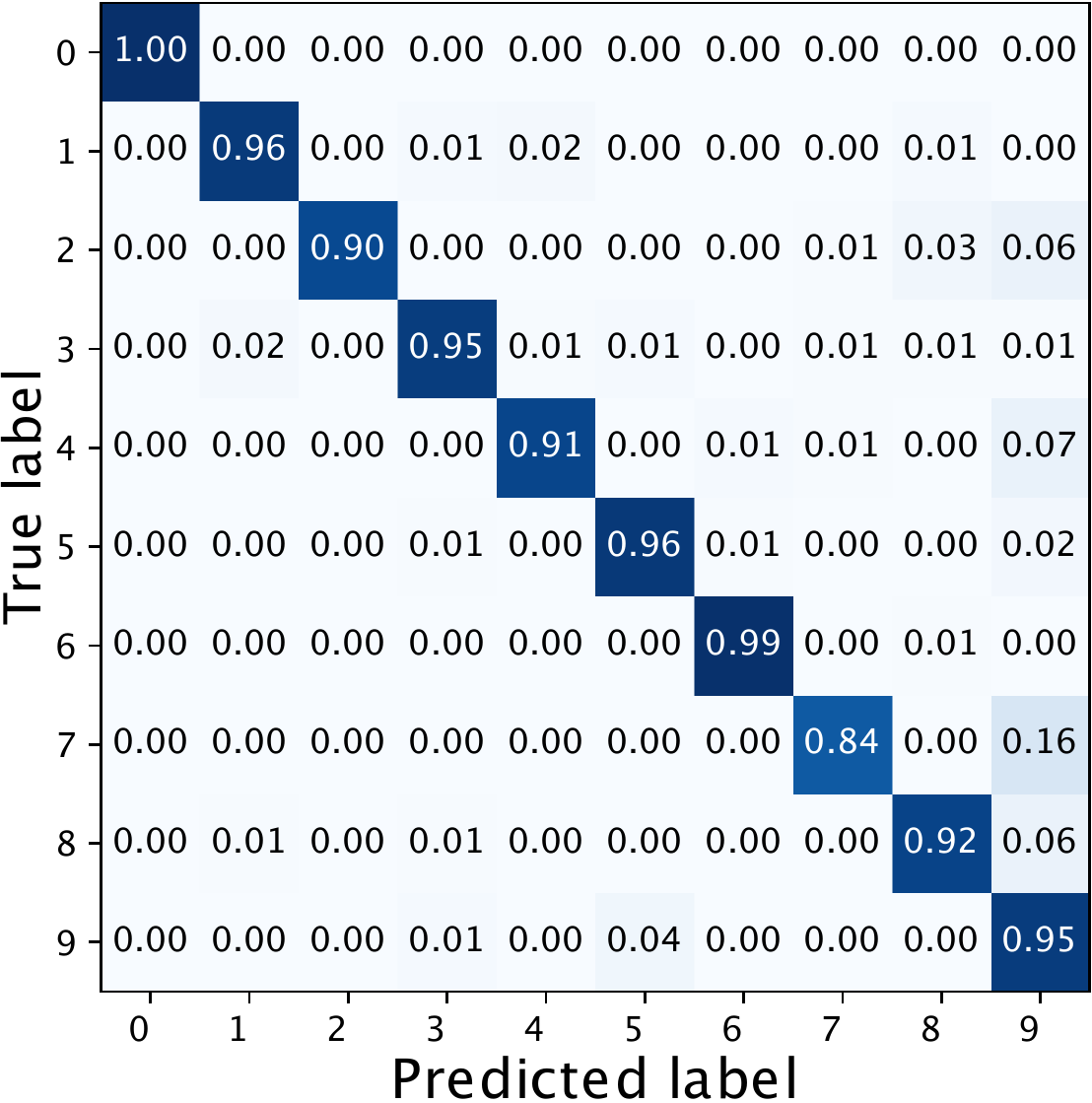}
        \\
    \end{tabu}%
    \hspace{6pt}%
    \begin{tabu} to .49\textwidth {@{\hspace{0pt}} X[5,c,m] @{\hspace{4pt}} X[80,c,m] @{\hspace{4pt}} X[70,c,m] @{\hspace{0pt}}}
        \vspace{-75pt}
        \textbf{B}
        &
        \includegraphics[width=\linewidth]{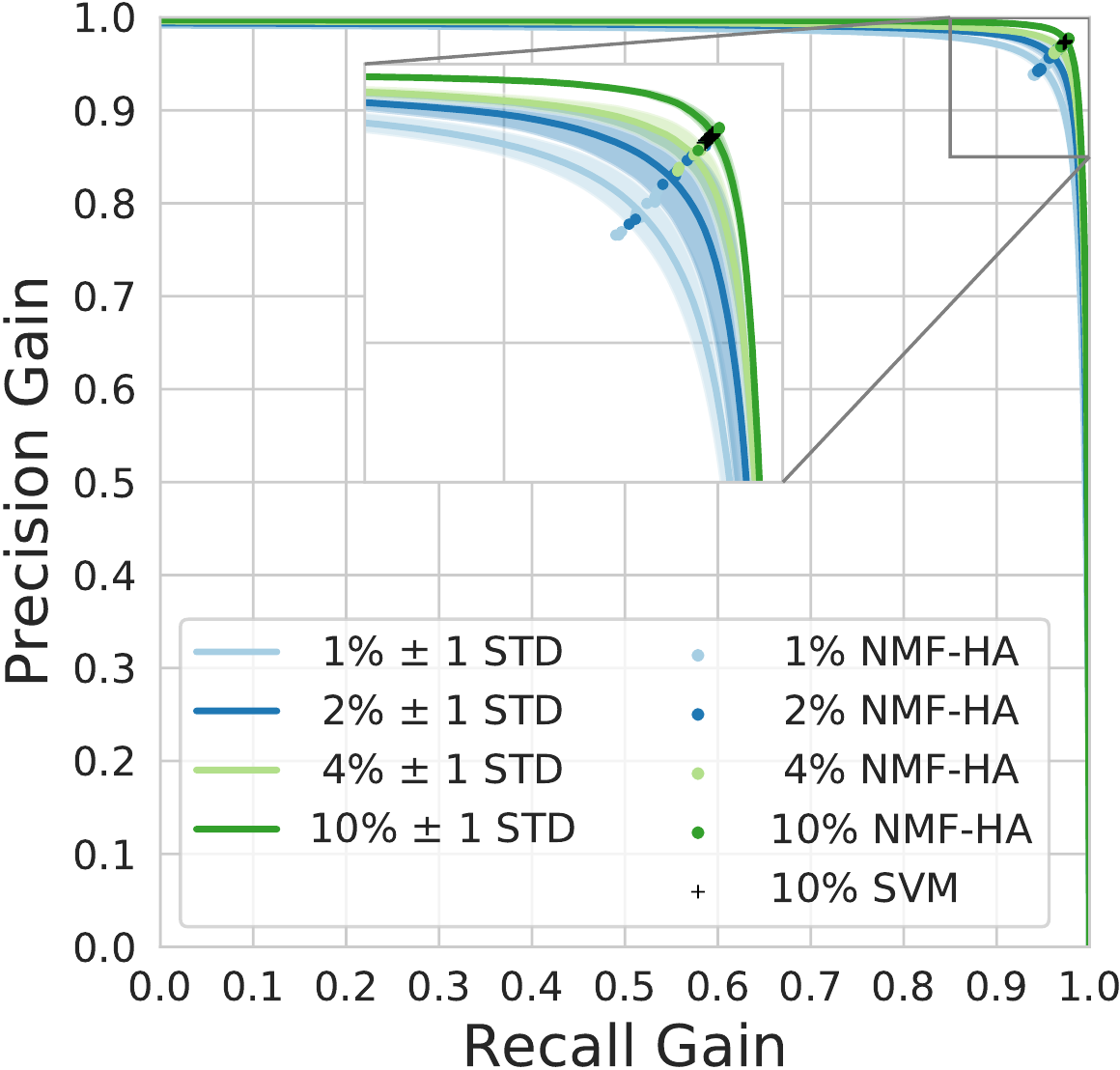}
        &
        \includegraphics[width=\linewidth]{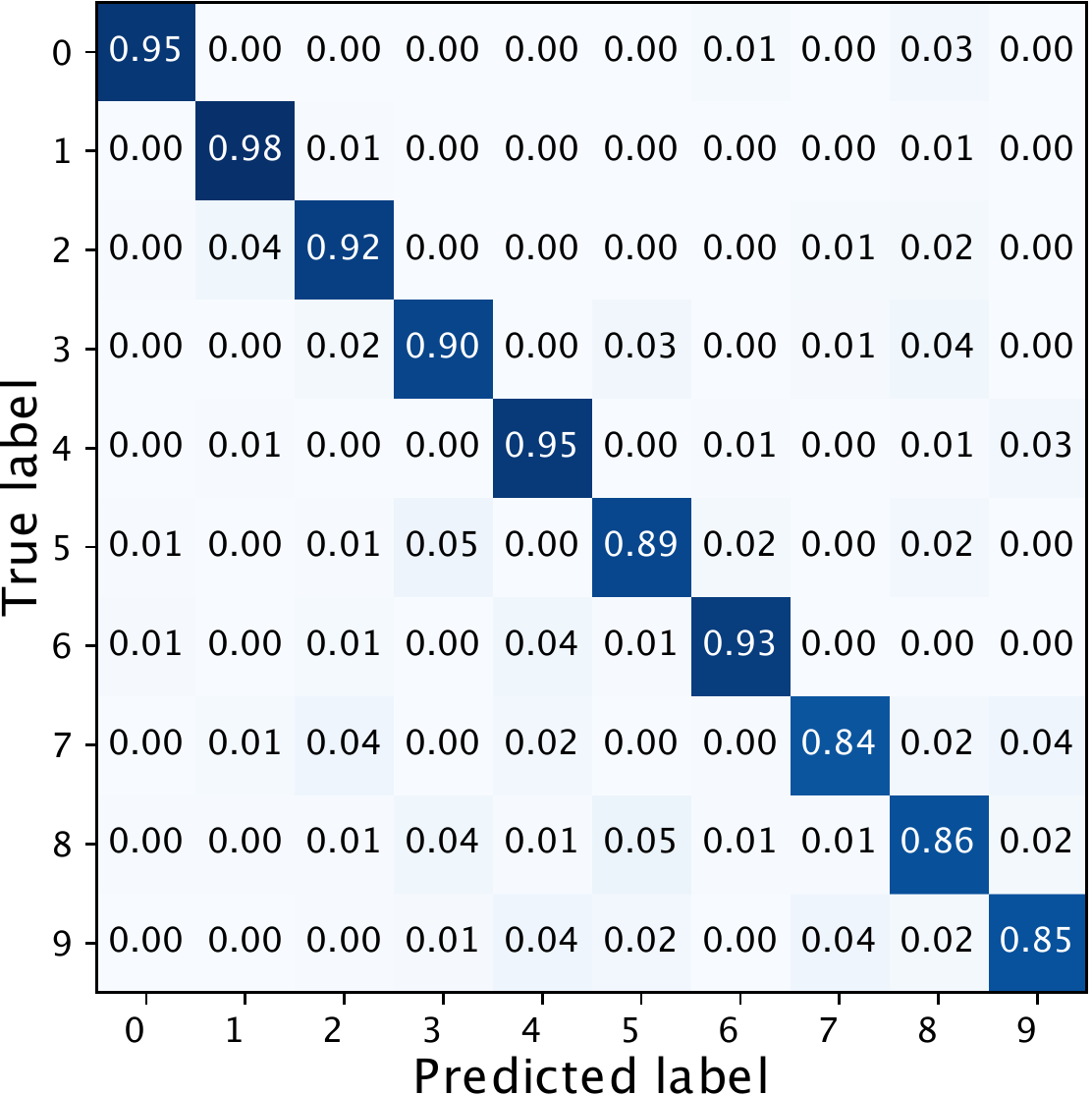}
        \\
    \end{tabu}
    
    \caption{Supervised classification with Neural Nystr\"om needs a limited amount of annotated data (indicated as a percentage) on the Digits (A) and MNIST (B) datasets.  See \cref{sec:results} for a description of how these results were obtained. For each percentage $\alpha$, we run 10 random trials and get 10 NMF-HA points.
    The confusion matrix corresponds to one of these trials. Better viewed with zoom.}
    \label{fig:supervised}
\end{figure}

\section{Conclusions}
\label{sec:conclusions}

We presented a biologically-inspired neural network, formally derived from the Nystr\"om method, to approximate conditional probabilities that arise as an agent explores a given space. The proposed network yields Neural Nystr\"om representations that soft-tile the input space with highly sparse receptive fields that are localized in the input space.
Our representations show that (1) place-cell-like neurons can encode information about conditional probabilities and (2) that the encoding of information about the current location (exhibiting localized and sparse receptive fields) is a necessary step to achieve this goal. Lastly, the proposed computational motif can be extended to handle supervised problems, creating class-specific place cells while exhibiting low sample complexity.

\textbf{Acknowledgements.} We would like to thank Mihai Capot\u{a}, Dmitri Chklovskii, Bryn Keller, Victor Minden, Anirvan Sengupta, Javier Turek, and Ted Willke for fruitful discussions and suggestions. We appreciate the availability of datasets provided by the UCI Machine Learning Repository~\cite{duaUCIMachineLearning2017}.

% \newpage

\bibliographystyle{abbrvnat}
\bibliography{neustrom}

\newpage

\appendix

% % \begin{center}
% % \begin{LARGE}
% %     Do place cells dream of conditional probabilities?\\
% %     Supplementary material
% % \end{LARGE}
% % \end{center}

% \setcounter{page}{1}

\section{Nystr\"om derivation}
\label{sec:nystrom}

This section closely follows Appendix A of~\cite{mairalEndtoendKernelLearning2016} and is included here for completeness.
Consider an input space $\set{X}$ and a positive semidefinite kernel function $K : \set{X} \times \set{X} \rightarrow \Real$. Let $\set{H}$ be an associated reproducing kernel Hilbert space and $\phi : \set{X} \rightarrow \set{H}$ be a feature map such that for any $\vect{x}_i, \vect{x}_j \in \set{X}$, $K(\vect{x}_i, \vect{x}_j) = \langle \phi(\vect{x}_i), \phi(\vect{x}_j) \rangle_{\set{H}}$.. Given a set of $T$ input points $\vect{x}_1, \dots, \vect{x}_n \in \set{X}$, we define the kernel matrix $\mat{K} \in \Real^{n \times n}$ by $\mat{K}_{ij} = K(\vect{x}_i, \vect{x}_j)$.

We also consider the kernelized data matrix $\mat{K} = \transpose{\mat{\Phi}} \mat{\Phi}$, where $\mat{\Phi} \in \Real^{d \times n}$ is the matrix containing $\phi(\vect{x}_1), \dots, \phi(\vect{x}_n)$ as columns (note that $d$ might be infinite).

We want to approximate
\begin{equation}
	K(\vect{x}_i, \vect{x}_j) = \langle \phi(\vect{x}_i), \phi(\vect{x}_j) \rangle_{\set{H}} \approx \langle \varphi(\vect{x}_i), \varphi(\vect{x}_j) \rangle
\end{equation}

Let $\vect{w}_1, \dots, \vect{w}_R \in \set{X}$ be a collection of $R$ landmark points that we will use to build our approximation.
We will approximate $\langle \phi(\vect{x}_i), \phi(\vect{x}_j) \rangle_{\set{H}}$ by the dot product of their orthogonal projections $\vect{p}_i$ and $\vect{p}_j$ on $\set{P} = \operatorname{Span}(\phi(\vect{w}_1), \dots , \phi(\vect{w}_R))$.
The orthogonal projection $f_t$ is defined as
\begin{equation}
	\vect{p}_i = \argmin_{\vect{p} \in \set{P}} \norm{\phi(\vect{x}_i) - \vect{p}}{\set{H}}^2 .
\end{equation}
Equivalently,
\begin{equation}
	\vect{p}_i = \sum_{r=1}^{R} \left( \alpha_i^{\star} \right)_r \, \phi(\vect{w}_r)
    \quad\text{with}\quad
    \vect{\alpha}_i^{\star} \in \argmin_{\vect{\alpha} \in \Real^R} \norm{\phi(\vect{x}_i) - \sum_{r=1}^{R} \vect{\alpha}_r \, \phi(\vect{w}_r)}{\set{H}}^2 .
    \label{eq:nystrom_projection_span_p}
\end{equation}

\begin{equation}
	\vect{\alpha}_i^{\star} \in \argmin_{\vect{\alpha} \in \Real^R} 1 - 2 \transpose{\vect{\alpha}} \mat{K}_{\mat{W}, \vect{x}_i}
    + \transpose{\vect{\alpha}} \mat{K}_{\mat{W}, \mat{W}} \vect{\alpha} ,
    \label{eq:nystrom_projection_span_ridge_regression}
\end{equation}
where we used $\langle \phi(\vect{x}), \phi(\vect{x}) \rangle_{\set{H}} = 1$ and
\begin{align}
	\mat{K}_{\mat{W}, \vect{x}} &= 
	\begin{bmatrix}
    	\langle \phi(\vect{w}_1) , \phi(\vect{x}) \rangle_{\set{H}} \\
        \vdots \\
        \langle \phi(\vect{w}_R) , \phi(\vect{x}) \rangle_{\set{H}}
    \end{bmatrix}    
    =
	\begin{bmatrix}
    	K(\vect{w}_1, \vect{x}) \\
        \vdots \\
        K(\vect{w}_R, \vect{x})
    \end{bmatrix} ,
    \\
    \mat{K}_{\mat{W}, \mat{W}} &= 
    \begin{bmatrix}
    	\langle \phi(\vect{w}_1) , \phi(\vect{w}_1) \rangle_{\set{H}} & \cdots & \langle \phi(\vect{w}_1) , \phi(\vect{w}_R) \rangle_{\set{H}} \\
        \vdots & \ddots & \vdots \\
        \langle \phi(\vect{w}_R) , \phi(\vect{w}_1) \rangle_{\set{H}} & \cdots & \langle \phi(\vect{w}_R) , \phi(\vect{w}_R) \rangle_{\set{H}}
    \end{bmatrix}
    =
	\begin{bmatrix}
    	K(\vect{w}_1, \vect{w}_1) & \cdots & K(\vect{w}_1, \vect{w}_R) \\
        \vdots & \ddots & \vdots \\
        K(\vect{w}_R, \vect{w}_1) & \cdots & K(\vect{w}_R, \vect{w}_R)
    \end{bmatrix} .
\end{align}
Assuming that $\mat{K}_{\mat{W}, \mat{W}}$ is invertible, the solution of \cref{eq:nystrom_projection_span_ridge_regression} is
$
	\vect{\alpha}_i^{\star} = \mat{K}_{\mat{W}, \mat{W}}^{-1} \mat{K}_{\mat{W}, \vect{x}_i}
$.
Writing $\vect{p}_i$ in vector form, i.e.,
$\vect{p}_i =  \left[ \phi(\vect{w}_1), \dots. , \phi(\vect{w}_R) \right] \vect{\alpha}_i^{\star}$, we finally get
% \begin{equation}
% 	\vect{p}_i =  \left[ \phi(\vect{w}_1), \dots. , \phi(\vect{w}_R) \right] \vect{\alpha}_i^{\star} .
% \end{equation}
\begin{subequations}
\begin{align}
	\langle \vect{p}_i, \vect{p}_j \rangle_{\set{H}}
	&=
    \langle \left[ \phi(\vect{w}_1), \dots. , \phi(\vect{w}_R) \right] \vect{\alpha}^{\star}_i, \left[ \phi(\vect{w}_1), \dots. , \phi(\vect{w}_R) \right] \vect{\alpha}^{\star}_j \rangle_{\set{H}}
    \\
    &=
    \transpose{\vect{\alpha}^{\star}_i} \mat{K}_{\mat{W}, \mat{W}} \vect{\alpha}^{\star}_j
    \\
    &=
    \transpose{\mat{K}_{\mat{W}, \vect{x}_i}} \mat{K}_{\mat{W}, \mat{W}}^{-1}  \mat{K}_{\mat{W}, \mat{W}} \mat{K}_{\mat{W}, \mat{W}}^{-1} \mat{K}_{\mat{W}, \vect{x}_j}
    \\
    &=
    \langle \varphi(\vect{x}_i), \varphi(\vect{x}_j) \rangle
    ,
\end{align}
\end{subequations}
where
\begin{equation}
	\varphi(\vect{x}) = \mat{K}_{\mat{W}, \mat{W}}^{-1/2} \mat{K}_{\mat{W}, \vect{x}}
	.
\end{equation}
When $\mat{K}_{\mat{W}, \mat{W}}$ is not invertible or simply badly conditioned, it is common to use
\begin{equation}
	\varphi(\vect{x}) = (\mat{K}_{\mat{W}, \mat{W}} + \varepsilon \mat{I})^{-1/2} \mat{K}_{\mat{W}, \vect{x}} .
\end{equation}

\begin{proposition}[\cite{zhangImprovedNystromLowrank2008}]
	Given a data set $\vect{x}_1, \dots, \vect{x}_n \in \set{X}$, and the landmark set $\vect{w}_1, \dots, \vect{w}_R \in \set{X}$, the Nystr\"om reconstruction of the kernel entry $K(\vect{x}_i,\vect{x}_j)$ will be exact if there exist two landmark points $\vect{w}_r, \vect{w}_s$ such that $\vect{x}_i = \vect{w}_r$ and $\vect{x}_j = \vect{w}_s$.
\end{proposition}

\begin{proof}
    \begin{subequations}
	\begin{align}
    	\langle \varphi(\vect{x}_i), \varphi(\vect{x}_j) \rangle
        &=
        \langle \mat{K}_{\mat{W}, \mat{W}}^{-1/2} \mat{K}_{\mat{W}, \vect{x}_i}, \mat{K}_{\mat{W}, \mat{W}}^{-1/2} \mat{K}_{\mat{W}, \vect{x}_j} \rangle \\
        &=
        \transpose{\mat{K}_{\mat{W}, \vect{x}_i}} \mat{K}_{\mat{W}, \mat{W}}^{-1} \mat{K}_{\mat{W}, \vect{x}_j} \\
        &=
        \transpose{\mat{K}_{\mat{W}, \vect{w}_r}} \mat{K}_{\mat{W}, \mat{W}}^{-1} \mat{K}_{\mat{W}, \vect{w}_l} \\
        &= \mat{S}_r \mat{K}_{\mat{W}, \mat{W}} \mat{K}_{\mat{W}, \mat{W}}^{-1} \mat{K}_{\mat{W}, \mat{W}} \mat{S}_l \\
        &= K(\vect{w}_r,\vect{w}_l) \\
        &= K(\vect{x}_i,\vect{x}_j) ,
    \end{align}
    \end{subequations}
    where $\mat{S}_k$ and $\mat{S}_i$ are sampling matrices that select the $k$-th row and $i$-th column, respectively.
\end{proof}

\section{The embedding}
\label{sec:embedding}

The embedding could be pre-trained separately using \cref{eq:KL_loss}. Depending on the input $p_{\text{in}}$ and the chosen $K$, this optimization problem has different names: SNE~\cite{hintonStochasticNeighborEmbedding2003}, t-SNE~\cite{vandermaatenVisualizingHighdimensionalData2008}, word2vec~\cite{mikolovDistributedRepresentationsWords2013}, and GLOVE~\cite{penningtonGloveGlobalVectors2014}, to name a few examples.

However, because Neural Nystr\"om is fully differentiable, we jointly learn the embedding and the Neural Nystr\"om parameters. Any neural network can be used to compute the embedding. As an option, we found that Random Fourier features~\cite{rahimiRandomFeaturesLarge2007} (RFF, see \cref{sec:rff} for a description) can be used as a layer to compute the embedding. In this work, the embedding takes the form
\begin{subequations}
\begin{align}
    \vect{v}_{\vect{x}}^{(0)}
    &=
    \begin{cases}
        \operatorname{RFF}(\vect{x}) & \text{if using RFF} \\
        \vect{x} & \text{otherwise} \\
    \end{cases}
     \\
    \vect{v}_{\vect{x}}^{(\ell+1)} &= \operatorname{PReLU} \left( \mat{A}^{(\ell)} \vect{v}_{\vect{x}}^{(\ell)}  + \vect{b}^{(\ell)} \right)
    \label{eq:embedding}
\end{align}
\end{subequations}
Then, the set of parameters to be learned is: for Neural Nystr\"om, $\mat{M}, \mat{W}$  and, for the embedding, $(\forall \ell)\, \mat{A}^{(\ell)}, \vect{b}^{(\ell)}$, PReLU slope, and optionally RFF variance (see \cref{sec:rff}).

Since the focus of this work is not on the embedding itself, we make no claims on the biological plausibility of an embedding of this shape, nor try to map it to brain structures. We just use a simple architecture to show that we can simultaneously learn the embedding and the hippocampal module.

\section{Random Fourier Features}
\label{sec:rff}

\begin{theorem}[Bochner~\cite{rudinFourierAnalysisGroups1990}]
    A continuous kernel $K(\vect{x}, \vect{y}) = K(\vect{x} - \vect{y})$ on $\Real^d$ is positive definite if and only if $K(\delta)$ is the Fourier transform of a non-negative measure.
\end{theorem}
Bochner's theorem guarantees that the kernel's Fourier
transform $p(\omega)$ is a proper probability distribution. We have
\begin{align}
    K(\vect{x} - \vect{y})
    &=
    \int_{\Real^d} p(\omega) e^{i \langle \omega,  \vect{x} - \vect{y} \rangle} \, d\omega
    =
    \mathrm{E}_{\omega} \left[ e^{i \langle \omega,  \vect{x} - \vect{y} \rangle} \right]
    \\
    &\approx
    \tfrac{1}{D}
    \sum_{l=1}^{D} e^{i \langle \omega_l,  \vect{x} - \vect{y} \rangle}
    =
    \tfrac{1}{D}
    \sum_{l=1}^{D} \cos \left( \langle \omega_l,  \vect{x} - \vect{y} \rangle \right) + i \sin \left( \langle \omega_l,  \vect{x} - \vect{y} \rangle \right)
    \\
    &\approx
    \sum_{l=1}^{D} \cos \left( \langle \omega_l,  \vect{x} \rangle \right) \cos \left( \langle \omega_l, \vect{y} \rangle \right) + \sin \left( \langle \omega_l,  \vect{x} \rangle \right) \sin \left( \langle \omega_l, \vect{y} \rangle \right)
    ,
\end{align}
where $\omega_l \sim p$ and the last approximation holds because both $p(\omega)$ and $K(\vect{x} - \vect{y})$ are real. We can then define the feature $\varphi( \vect{x}) \in \Real^{2D}$
% with individual components
% \begin{equation}
%     \varphi_l ( \vect{x}) =
%     \begin{cases}
%         \tfrac{1}{\sqrt{D}} \cos \left( \langle \omega_l,  \vect{x} \rangle \right)
%         & \text{$j$ is odd};
%         \\
%         \tfrac{1}{\sqrt{D}} \sin \left( \langle \omega_l,  \vect{x} \rangle \right)
%         & \text{$j$ is even},
%     \end{cases}
% \end{equation}
defined by
\begin{equation}
    \varphi_l ( \vect{x}) =
    \tfrac{1}{\sqrt{D}}
    \transpose{\left[
         \cos \left( \langle \omega_1,  \vect{x} \rangle \right) ,
         \dots ,
         \cos \left( \langle \omega_D,  \vect{x} \rangle \right) ,
         \sin \left( \langle \omega_1,  \vect{x} \rangle \right) ,
         \dots ,
         \sin \left( \langle \omega_D,  \vect{x} \rangle \right)
    \right]}
\end{equation}
and set
\begin{equation}
    K(\vect{x} - \vect{y})
    \approx
    \langle \varphi( \vect{x}), \varphi( \vect{y}) \rangle
    .
\end{equation}
We are only left with the specification of $p(\omega)$, which is the Fourier transform of the kernel $K$. For the RBF kernel with variance $\gamma^{-1}$, $p = N (0, \gamma)$.

We found that Random Fourier features~\cite{rahimiRandomFeaturesLarge2007} (RFF) can form useful embeddings. We could separately pre-train RFF using the loss function
\begin{equation}
   	- \iint 
   	p_{\text{in}} (\vect{y} | \vect{x})
   	\log
   	\frac{
   	    \langle \varphi( \vect{x}), \varphi( \vect{y}) \rangle
   	}{
   	    \int \langle \varphi( \vect{x}), \varphi( \vect{z}) \rangle
        \, d\vect{z}
   	}
   	\, d\vect{x} d\vect{y}
   	,
   	\label{eq:KL_loss_RFF}
\end{equation}
and the algorithm proposed in \cref{sec:optimization}. In this case, the only parameter is $\gamma$. Alternatively, we could learn an adaptive Fastfood transform~\cite{leFastfoodComputingKernel2013,yangDeepFriedConvnets2015}. Although initial tests were successful, we leave this exploration for future work. Of course, albeit their usefulness for machine learning, the neural inspiration is somewhat lost in \cref{eq:KL_loss_RFF}.

\section{Online computing with the successor representation: step by step derivation}
\label{sec:online_step_by_step}

Plugging \cref{eq:successor_representation_input_proba} into \cref{eq:neustrom_loss}, we get
\begin{subequations}
\begin{align}
	\mathcal{L}
	&= 
	- \iint p_{\text{in}} (\vect{y} | \vect{x})
	\log \frac{
	    \transpose{\vect{g}_\vect{x}} \vect{g}_{\vect{y}}
	}{
	    \int \transpose{\vect{g}_\vect{x}} \vect{g}_{\vect{z}}
        \, d\vect{v}_{\vect{z}}
	}
	\, d\vect{x} d\vect{y} .
	\\
	&\propto 
	- \iint \mathbb{E}_{\pi, p} \left[ \sum_{t=0}^{\infty} \gamma^t \indicator{\vect{x}_t = \vect{y}} \,\bigg|\, \vect{x}_0 = \vect{x} \right]
	\left(
	\log \frac{
	    \transpose{\vect{g}_\vect{x}} \vect{g}_{\vect{y}}
	}{
	    \int \transpose{\vect{g}_\vect{x}} \vect{g}_{\vect{z}}
        \, d\vect{v}_{\vect{z}}
	}
	\right) 
	\, d\vect{x} d\vect{y}
	\\
	&= 
	- \iint
	\mathbb{E}_{\pi, p}
	\left[
    	\sum_{t=0}^{\infty} \gamma^t \indicator{\vect{x}_t = \vect{y}} 
		\log \frac{
    	    \transpose{\vect{g}_\vect{x}} \vect{g}_{\vect{y}}
    	}{
    	    \int \transpose{\vect{g}_\vect{x}} \vect{g}_{\vect{z}}
            \, d\vect{v}_{\vect{z}}
    	}
    	\,\bigg|\, \vect{x}_0 = \vect{x}
	\right]
	\, d\vect{x} d\vect{y} .
	\\
	&=
	- \int
	\mathbb{E}_{\pi, p}
	\left[
	    \int
    	\sum_{t=0}^{\infty} \gamma^t \indicator{\vect{x}_t = \vect{y}}
		\log \frac{
    	    \transpose{\vect{g}_\vect{x}} \vect{g}_{\vect{y}}
    	}{
    	    \int \transpose{\vect{g}_\vect{x}} \vect{g}_{\vect{z}}
            \, d\vect{v}_{\vect{z}}
    	}
    	d\vect{y}
    	\,\bigg|\, \vect{x}_0 = \vect{x}
	\right]
	\, d\vect{x} .
	\\
	&= 
	- \int
	\mathbb{E}_{\pi, p}
	\left[
    	\sum_{t=0}^{\infty} \gamma^t
    	\int
    	\indicator{\vect{x}_t = \vect{y}}
		\log \frac{
    	    \transpose{\vect{g}_\vect{x}} \vect{g}_{\vect{y}}
    	}{
    	    \int \transpose{\vect{g}_\vect{x}} \vect{g}_{\vect{z}}
            \, d\vect{v}_{\vect{z}}
    	}
    	d\vect{y}
    	\,\bigg|\, \vect{x}_0 = \vect{x}
	\right]
	\, d\vect{x}
	\\
	&=
	- \int
	\mathbb{E}_{\pi, p}
	\left[
    	\sum_{t=0}^{\infty} \gamma^t
		\log \frac{
    	    \transpose{\vect{g}_\vect{x}} \vect{g}_{\vect{x}_{t}}
    	}{
    	    \int \transpose{\vect{g}_\vect{x}} \vect{g}_{\vect{z}}
            \, d\vect{v}_{\vect{z}}
    	}
    	\,\bigg|\, \vect{x}_0 = \vect{x}
	\right]
	\, d\vect{x}
	\\
	&\propto 
	- \int
	\sum_{\substack{
	    \tau_{\pi, p} = [ \vect{x}_{0}, \vect{x}_{1}, \dots, \vect{x}_{T} ]
	    \\
	    \vect{x}_{0} = \vect{x}
	}}
	\sum_{t=0}^{T} \gamma^t
	\log \frac{
	    \transpose{\vect{g}_\vect{x}} \vect{g}_{\vect{x}_{t}}
	}{
	    \int \transpose{\vect{g}_\vect{x}} \vect{g}_{\vect{z}}
        \, d\vect{v}_{\vect{z}}
	}
	\, d\vect{x}
	.
\end{align}
\label{eq:SR_neustrom_online_step_by_step}%
\end{subequations}
In the last equality we replaced the expectation by the empirical expectation computed from samples.

\section{Implementation specification and details}
\label{sec:implementation}

Unless specified, for simplicity, the input conditional probabilities are computed as a row-normalized kernel, see \cref{eq:kernel_input_proba}.
Specifically, we use an RBF kernel. When required to ensure manifold disentangling (i.e., avoiding links that do not follow the manifold geometry), we only compute the kernel values over a small set of nearest neighbors for each point~\cite{hintonStochasticNeighborEmbedding2003}.

At initialization, we set $\mat{M} = \mat{I}$ and the landmark matrix $\mat{W}$ to the $k$-means centroids of the set $\{ \vect{v}_{\vect{x}_i}\}_{i=1}^{n}$ (the ouput of the still untrained embedding network). When RFF are used, we initialize $\gamma=1$.

We use the PyTorch library (version 1.0.1) for our implementation. We optimize using the AMSGrad method~\cite{reddiConvergenceADAM2018}.
We observed that once the method is close to convergence (plateauing KL-divergence), doing a single additional round of $k$-means re-initialization helps further decrease the loss and improves the results. Stochastic optimization continues after this step.
During training, we gradually reduce the learning rate: we use the ReduceLROnPlateau scheduler (\url{https://pytorch.org/docs/stable/optim.html#torch.optim.lr_scheduler.ReduceLROnPlateau}) with parameters patience=10, cooldown=10, and threshold=$10^{-5}$. Once the learning rate is $10^{-8}$ or lower, we stop the optimization process.

\Cref{tab:architecture_specs,tab:training_specs} respectively provide the architecture and training specifications used in every example of this work.

\begin{table}[t]
    \caption{Architecture specifications. NN stands for nearest neighbors. In all cases the MLP for the embedding has two layers, with the first layer's dimensions specified in the table and the second layer $\mat{A}^{(2)} \in \Real^{d_1 \times 100}$, where $d_1$ is the output dimension of $\mat{A}^{(1)}$. The number of landmarks $r$ was chosen to roughly provide low KL divergence while keeping $r$ as small as possible. We intend to do systematic studies of $r$'s effect in future work.}
    \label{tab:architecture_specs}

    \centering
    \begin{small}
    \begin{tabu} to \textwidth {l *{6}{c}}
        \toprule
        & data dim. & Input prob. & \# NN & \# RFF & $\mat{A}^{(1)}$ & $r$ \\
    
        \midrule
    
        One circle (\cref{fig:synthetic_unsupervised}A) & 2 & RBF kernel ($\gamma=30$) & - & 100 & $100 \times 100$ & 40 \\
        Square Grid (\cref{fig:synthetic_unsupervised}B) & 2 & RBF kernel ($\gamma=30$) & - & 100 & $100 \times 100$ & 25 \\
        Two circles (\cref{fig:circles}B) & 2 & RBF kernel ($\gamma=30$) & - & 100 & $100 \times 100$ & 40 \\
        Digits (\cref{fig:supervised}A) & 64 & RBF kernel ($\gamma=30$) & 5 & 100 & $1000 \times 100$ & 100 \\
        MNIST (\cref{fig:supervised}B) & 784 & RBF kernel ($\gamma=30$) & 9 & - & $784 \times 300$ & 300 \\
        Teapots (\cref{fig:teapot}) & 23028 & RBF kernel ($\gamma=0.005$) & 3 & 1000 & $1000 \times 100$ & 80 \\
        Teapots (\cref{fig:teapot_nomad}) & 23028 & \cref{eq:sdp-kmeans} ($k=20$) & - & 1000 & $1000 \times 100$ & 80 \\
        
        \bottomrule
    \end{tabu}
    \end{small}
\end{table}

\begin{table}[t]
    \caption{Training specifications. For supervised learning, we change the batch size for different data sampling regimes to approximately ensure that a similar number of batches per epoch is processed in every case.}
    \label{tab:training_specs}

    \centering
    \begin{small}
    \begin{tabu} to \textwidth {l *{5}{c}}
        \toprule
        & learning mode & \# epochs & learning rate & batch size \\
    
        \midrule
    
        One circle (\cref{fig:synthetic_unsupervised}A) & unsupervised & 50 & $10^{-5}$ & 10 \\
        Square Grid (\cref{fig:synthetic_unsupervised}B) & unsupervised & 50 & $10^{-5}$ & 1000 \\
        Two circles (\cref{fig:circles}A) & unsupervised & 200 & $10^{-4}$ & 100 \\
        Two circles (\cref{fig:circles}B) & unsupervised & 100 & $10^{-3}$ & 20 \\
        Digits (\cref{fig:supervised}A) & unsupervised  & 500 & $10^{-4}$ & 1024 \\
        Digits (\cref{fig:supervised}A) & supervised (with 1\% of the data) & 1000 & $10^{-4}$ & 1 \\
        Digits (\cref{fig:supervised}A) & supervised (with 2\% of the data) & 1000 & $10^{-4}$ & 1 \\
        Digits (\cref{fig:supervised}A) & supervised (with 4\% of the data) & 1000 & $10^{-4}$ & 4 \\
        Digits (\cref{fig:supervised}A) & supervised (with 10\% of the data) & 1000 & $10^{-4}$ & 16 \\
        Digits (\cref{fig:supervised}A) & supervised (with 20\% of the data) & 1000 & $10^{-4}$ & 256 \\
        Digits (\cref{fig:supervised}A) & supervised (with 50\% of the data) & 1000 & $10^{-4}$ & 1024 \\
        MNIST (\cref{fig:supervised}B) & unsupervised & 500 & $10^{-4}$ & 1024 \\
        MNIST (\cref{fig:supervised}B) & supervised (with 1\% of the data) & 500 & $10^{-4}$ & 16 \\
        MNIST (\cref{fig:supervised}B) & supervised (with 2\% of the data) & 500 & $10^{-4}$ & 64 \\
        MNIST (\cref{fig:supervised}B) & supervised (with 4\% of the data) & 500 & $10^{-4}$ & 256 \\
        MNIST (\cref{fig:supervised}B) & supervised (with 10\% of the data) & 500 & $10^{-4}$ & 1024 \\
        Teapots (\cref{fig:teapot}) & unsupervised & 100 & $10^{-4}$ & 32 \\
        Teapots (\cref{fig:teapot_nomad}) & unsupervised & 500 & $10^{-4}$ & 256 \\
        
        \bottomrule
    \end{tabu}
    \end{small}
\end{table}

\section{Additional experimental results}
\label{sec:additional_results}

Additional results and plots are provided in \cref{fig:teapot,fig:teapot_nomad,fig:circles_additional,fig:digits_additional,fig:mnist_additional}.

\begin{figure}[t]
    \centering
    
    \begin{tabu} to \textwidth {X[45,c,m] X[16,c,m]}
        \includegraphics[width=\linewidth]{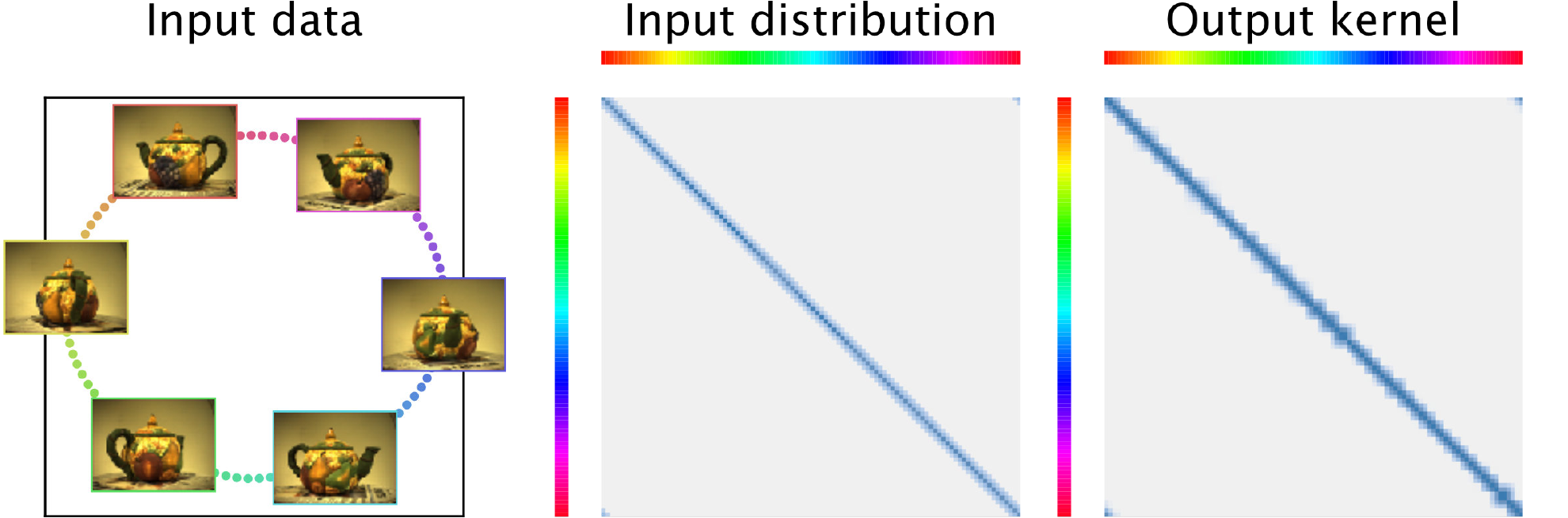}
        &
        \includegraphics[width=\linewidth]{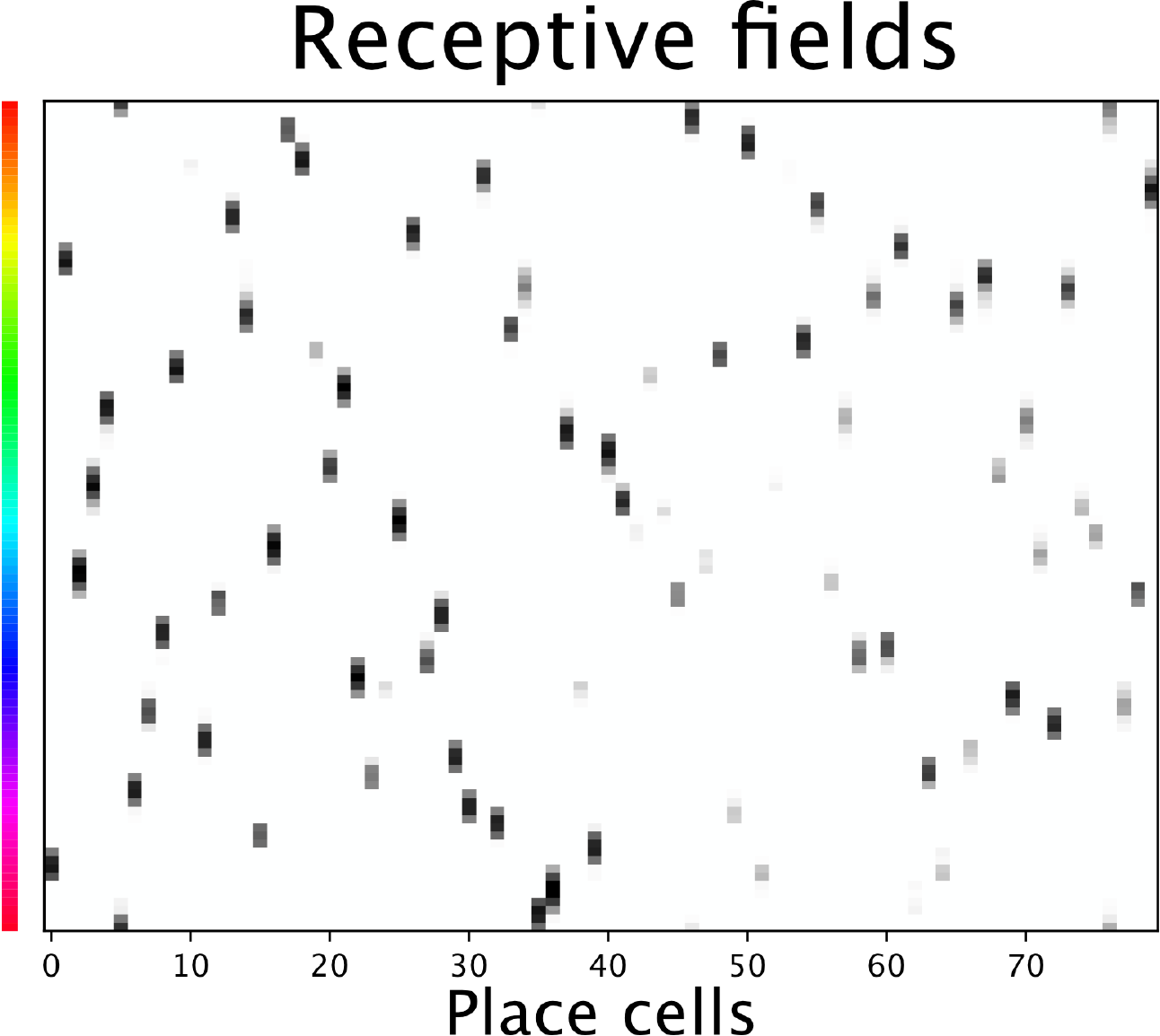}
        \\
    \end{tabu}
    \caption{
        Neural Nystr\"om can correctly represent a circular manifold in a high-dimensional space (23028 dimensions, i.e., $76 \times 101$ pixels, 3 color channels)~\cite{weinbergerIntroductionNonlinearDimensionality2006}. 
        The input conditional probabilities are computed using a row-normalized RBF kernel.
        On the left, we show the first two principal components of the input data.
        Let $\mat{G} = [\vect{g}_{\vect{x}_1}, \cdots, \vect{g}_{\vect{x}_n}] \in \Real^{r \times n}$, see \cref{eq:neustrom_net} and \cref{fig:neustrom_arch_hip_module}.
        The output kernel matrix is $\mat{G} \transpose{\mat{G}}$ and for the receptive fields we plot $\transpose{\mat{G}}$.
    }
    \label{fig:teapot}
\end{figure}

\begin{figure}[t]
    \centering
    
    \begin{tabu} to \textwidth {X[45,c,m] X[16,c,m]}
        \includegraphics[width=\linewidth]{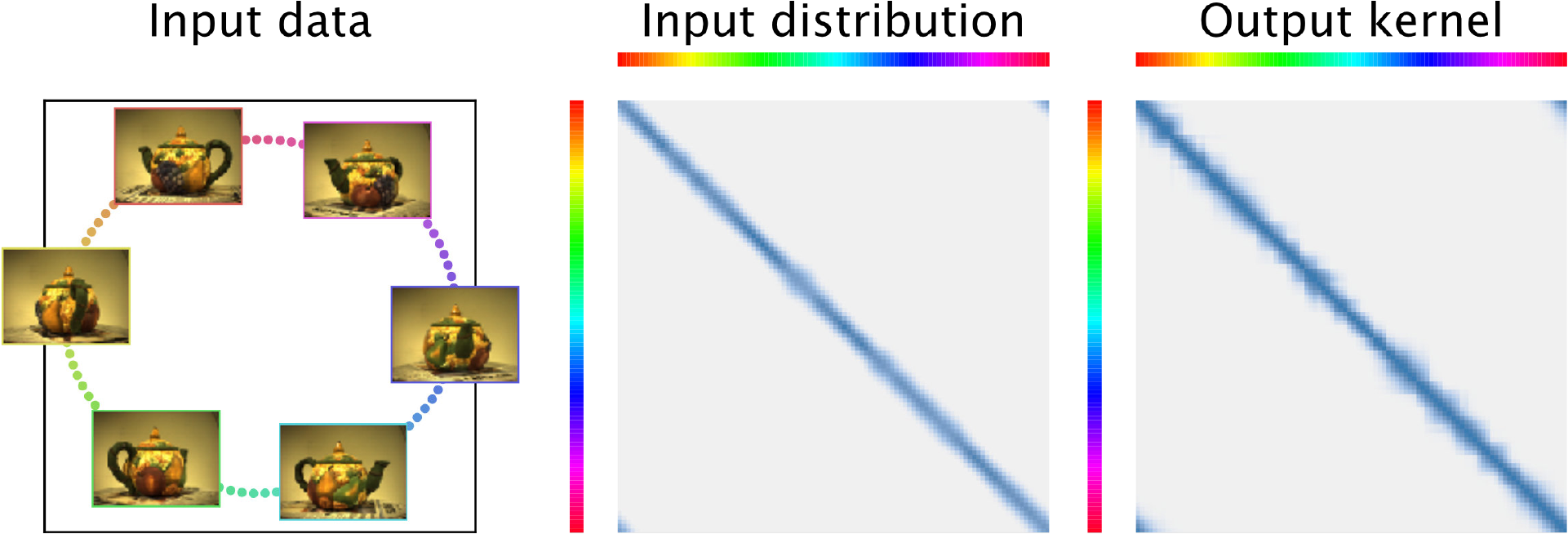}
        &
        \includegraphics[width=\linewidth]{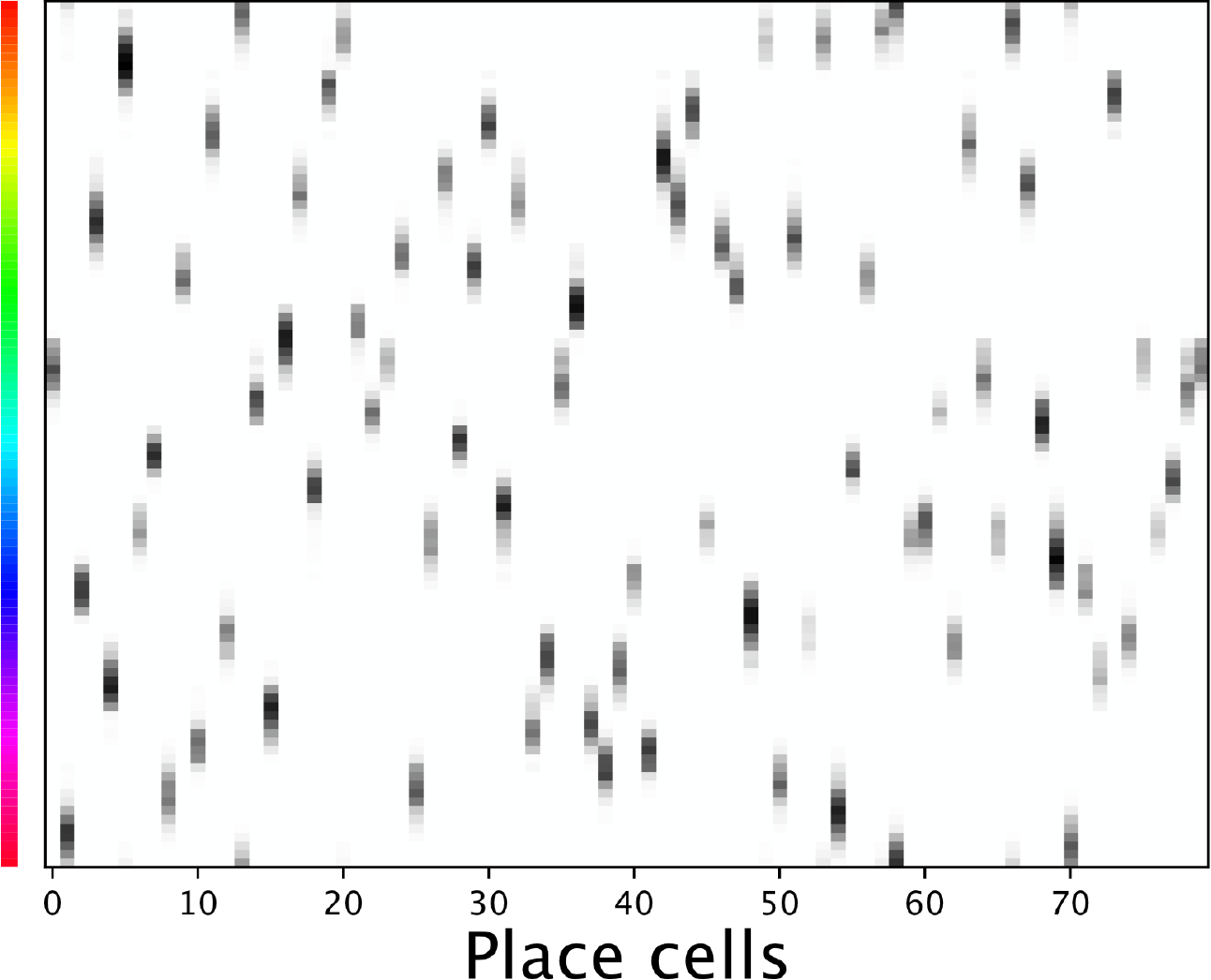}
        \\
    \end{tabu}
    \caption{
        Neural Nystr\"om can correctly represent a circular manifold in a high-dimensional space (23028 dimensions, i.e., $76 \times 101$ pixels, 3 color channels)~\cite{weinbergerIntroductionNonlinearDimensionality2006}.
        The input conditional probabilities are computed using \cref{eq:sdp-kmeans}.
        On the left, we show the first two principal components of the input data.
        Let $\mat{G} = [\vect{g}_{\vect{x}_1}, \cdots, \vect{g}_{\vect{x}_n}] \in \Real^{r \times n}$, see \cref{eq:neustrom_net} and \cref{fig:neustrom_arch_hip_module}.
        The output kernel matrix is $\mat{G} \transpose{\mat{G}}$ and for the receptive fields we plot $\transpose{\mat{G}}$.
    }
    \label{fig:teapot_nomad}
\end{figure}

\begin{figure}[t]
    \centering
    
    \begin{tabu} to .7\textwidth {X[20,c,m] X[12,c,m]}
        \includegraphics[width=\linewidth]{offline_synthetic/circles_land40_class1.pdf}
        &
        \includegraphics[width=\linewidth]{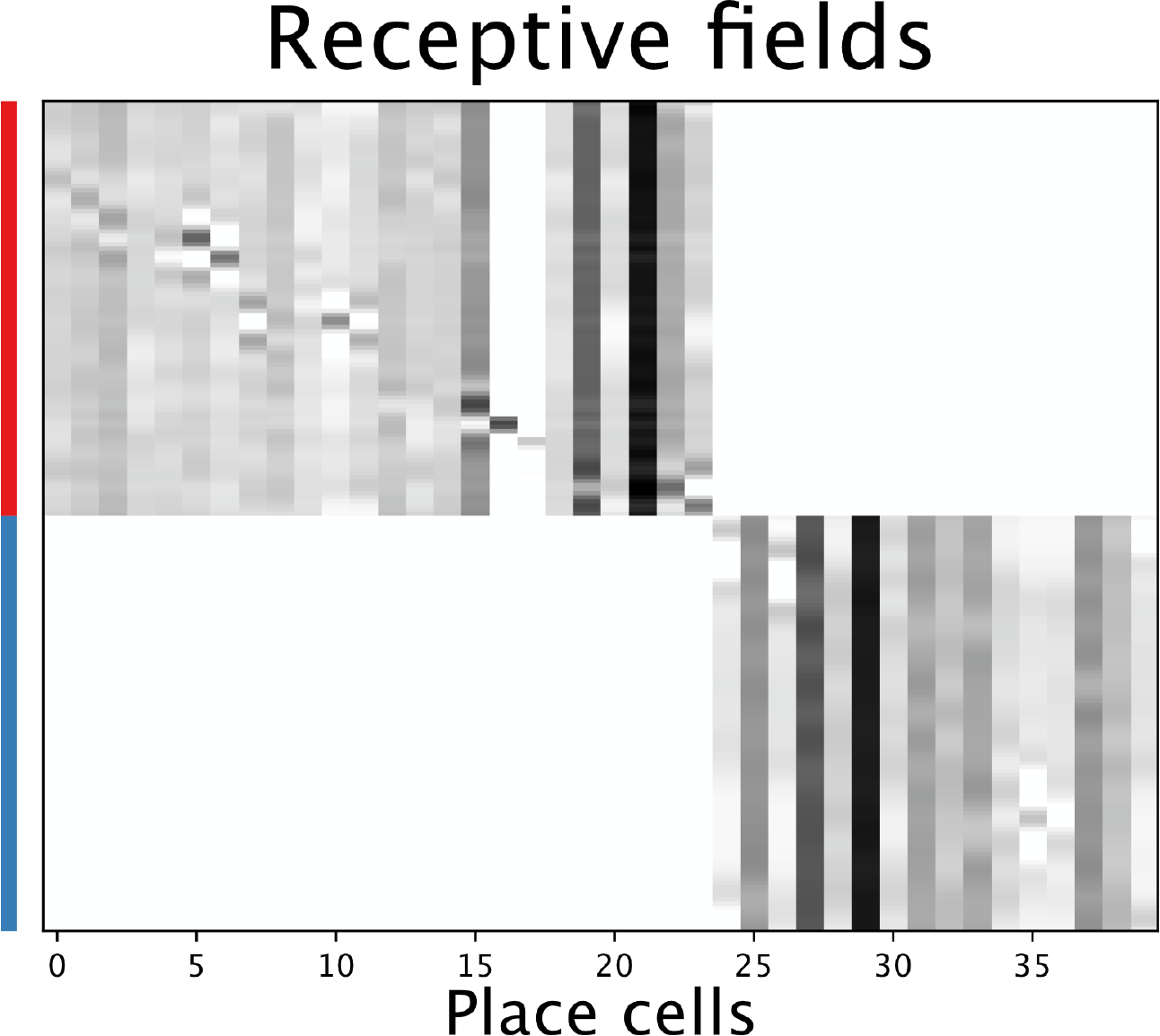}
        \\
        \\[-4pt]
        
        \includegraphics[width=\linewidth]{offline_synthetic/circles_land40_class2.pdf}
        &
        \includegraphics[width=\linewidth]{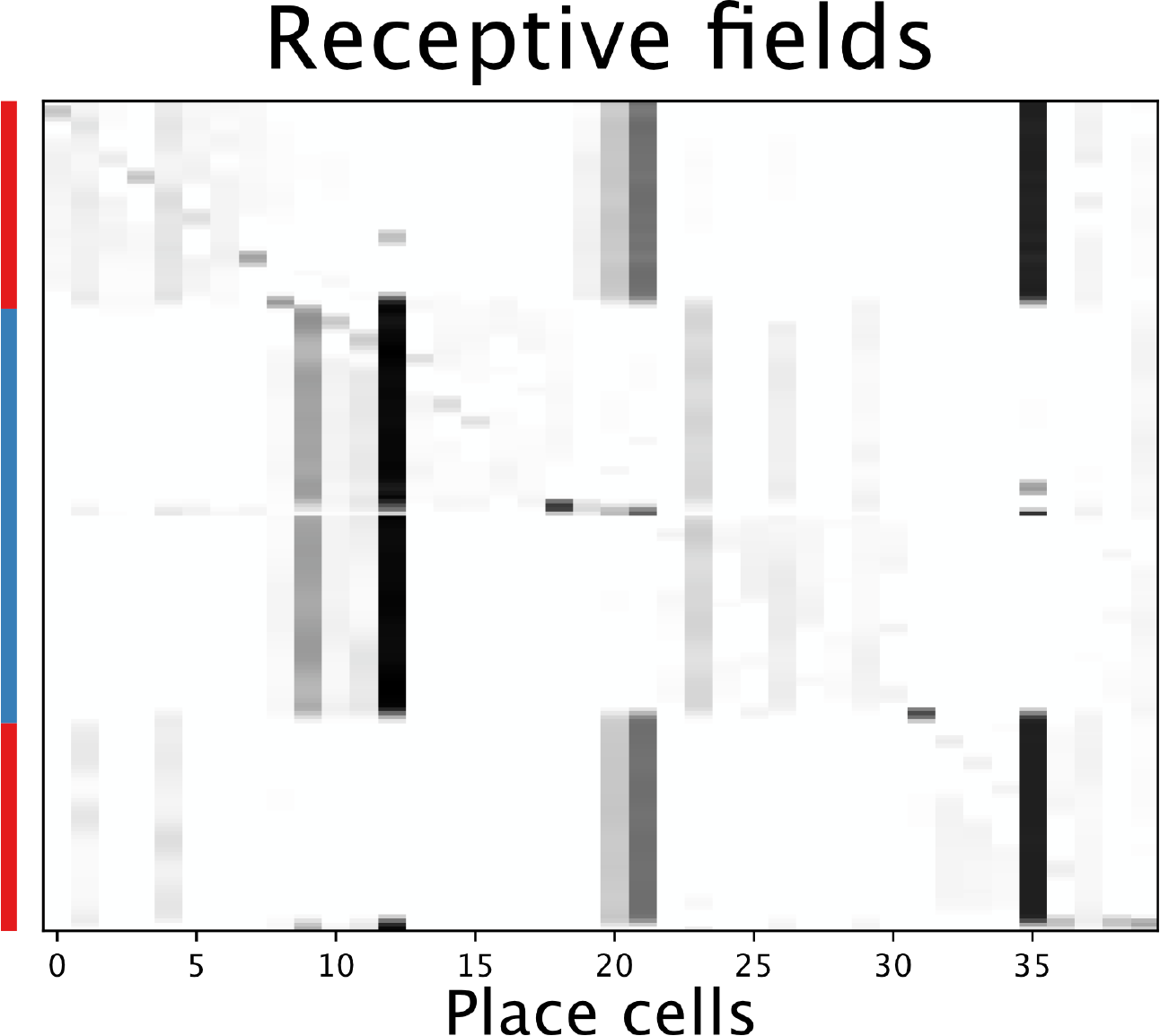}
        \\
    \end{tabu}
        
    \caption{Input data, output kernel, and receptive fields matrix for the experiment in \cref{fig:circles}B.}
    \label{fig:circles_additional}
\end{figure}

\begin{figure}[t]
    \centering
    
    \begin{tabu} to .7\textwidth {X[1,c,m] @{\hspace{4pt}} X[7,c,m] X[9,c,m]}
        &
        \hspace{12pt}Output kernel
        &
        Receptive fields
        \\

        \begin{sideways}
            Unsupervised
        \end{sideways}
        &
        \includegraphics[width=\linewidth]{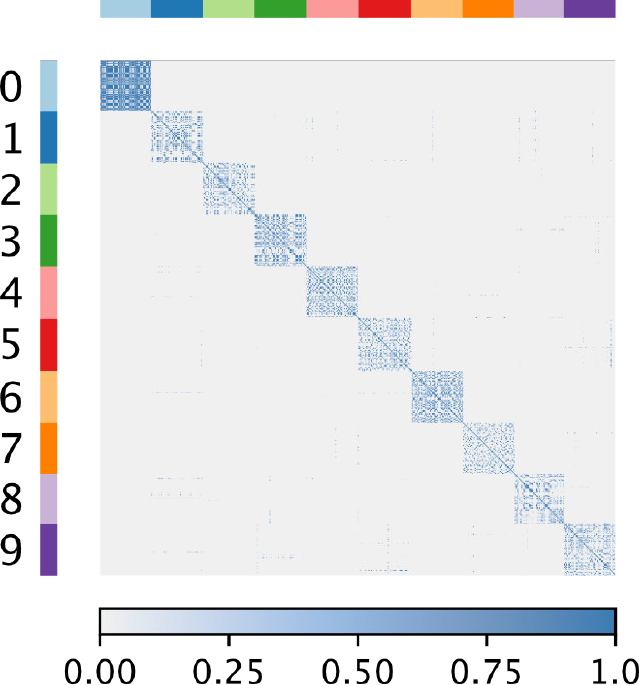}
        &
        \includegraphics[width=\linewidth]{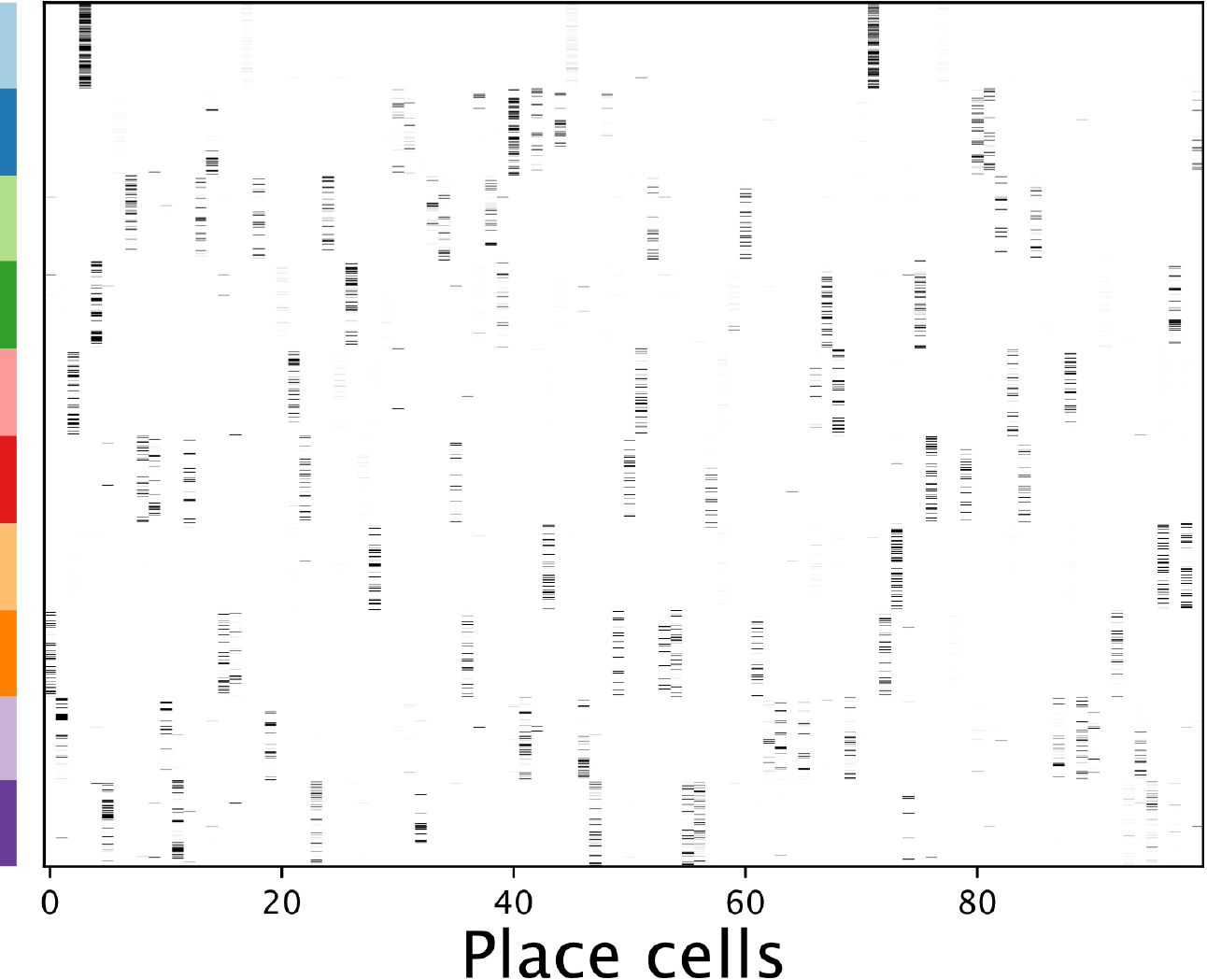}
        \\
        \\[-6pt]
        
        &
        \hspace{12pt}Output kernel
        &
        Receptive fields
        \\
        
        \begin{sideways}
            Supervised
        \end{sideways}
        &        
        \includegraphics[width=\linewidth]{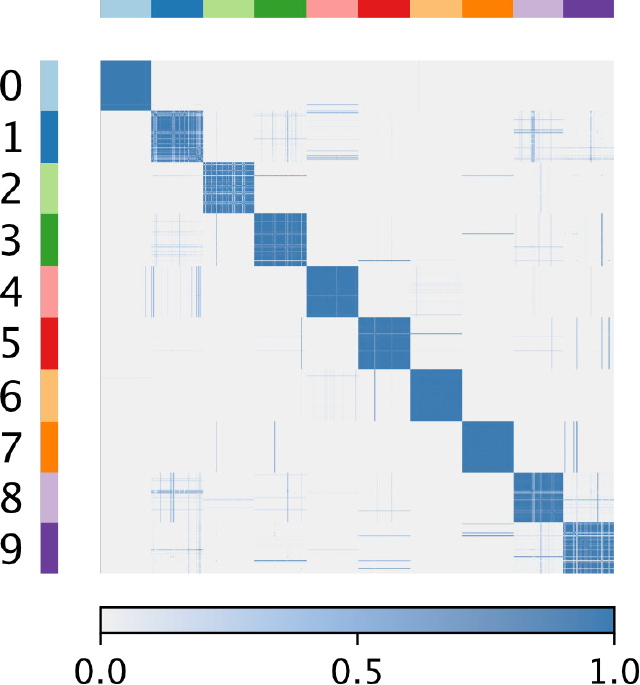}
        &
        \includegraphics[width=\linewidth]{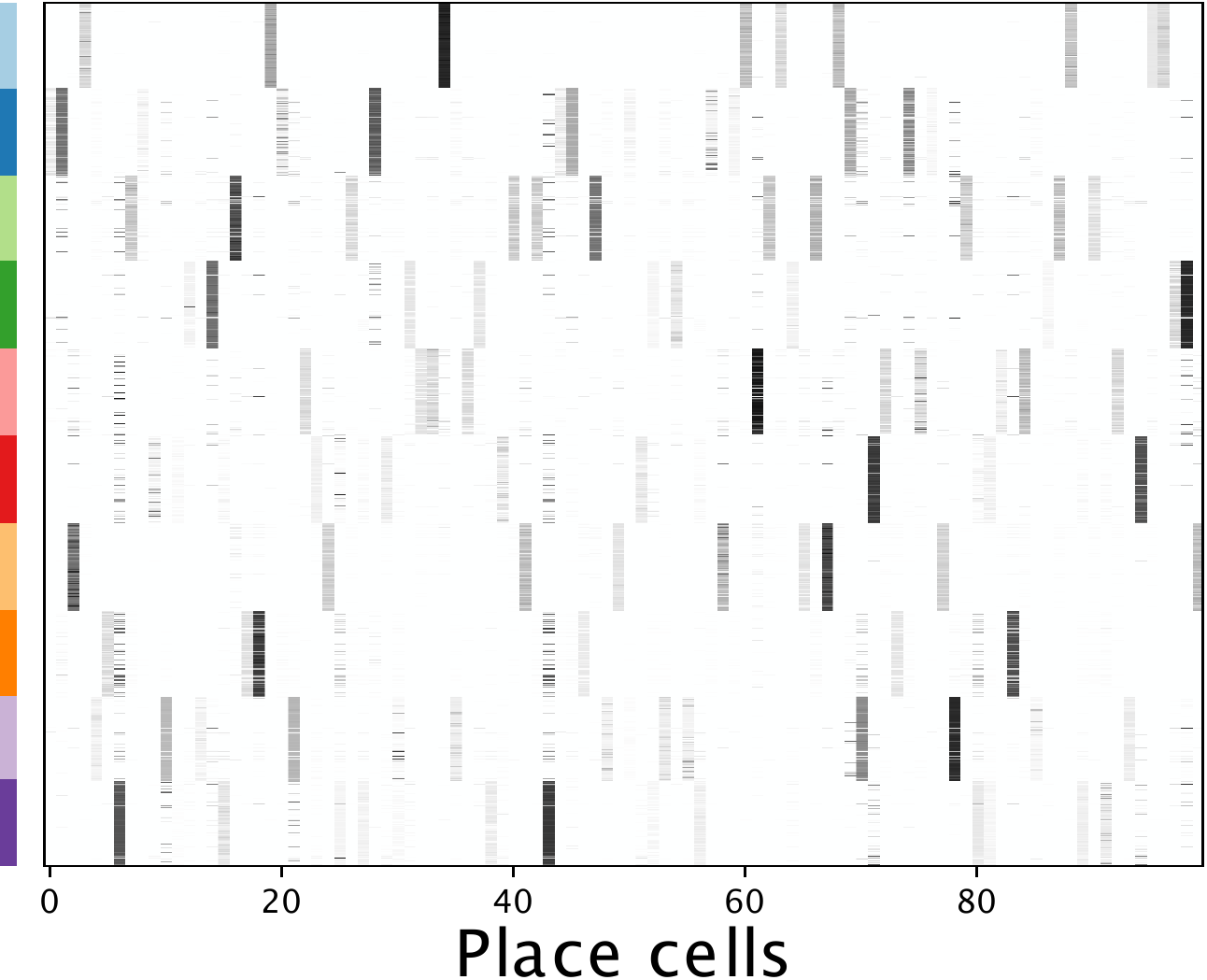}
        \\
    \end{tabu}
        
    \caption{Output kernel and receptive fields matrix for the experiment in \cref{fig:supervised}A, training the supervised layer with annotations on 10\% of the data.}
    \label{fig:digits_additional}
\end{figure}

\begin{figure}[t]
    \centering
    
    \begin{tabu} to .7\textwidth {X[1,c,m] @{\hspace{4pt}} X[7,c,m] X[9,c,m]}
        &
        \hspace{12pt}Output kernel
        &
        Receptive fields
        \\

        \begin{sideways}
            Unsupervised
        \end{sideways}
        &
        \includegraphics[width=\linewidth]{offline_images/digits_nc10_land100_unsupervised_kernel}
        &
        \includegraphics[width=\linewidth]{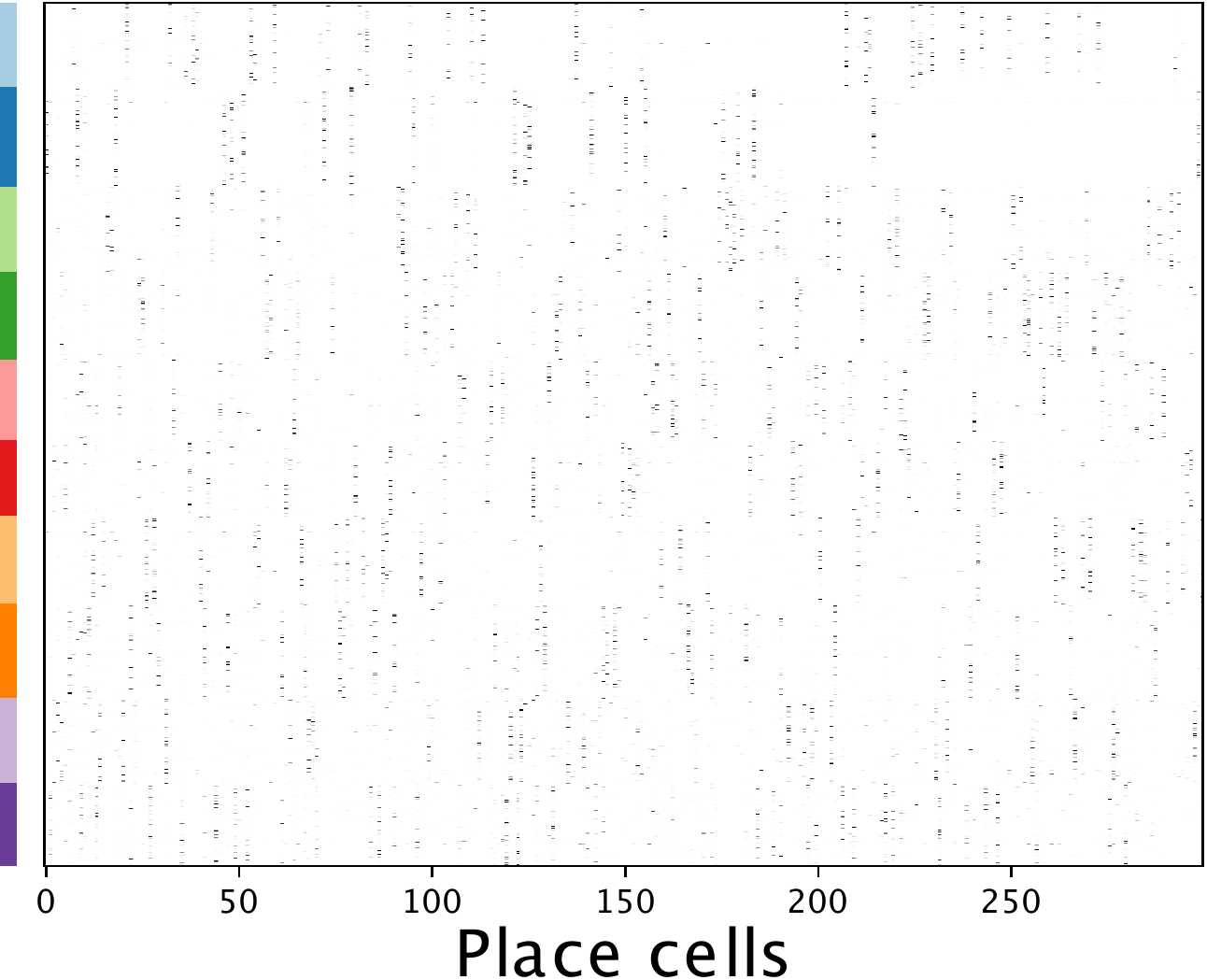}
        \\
        \\[-6pt]
        
        &
        \hspace{12pt}Output kernel
        &
        Receptive fields
        \\
        
        \begin{sideways}
            Supervised
        \end{sideways}
        &        
        \includegraphics[width=\linewidth]{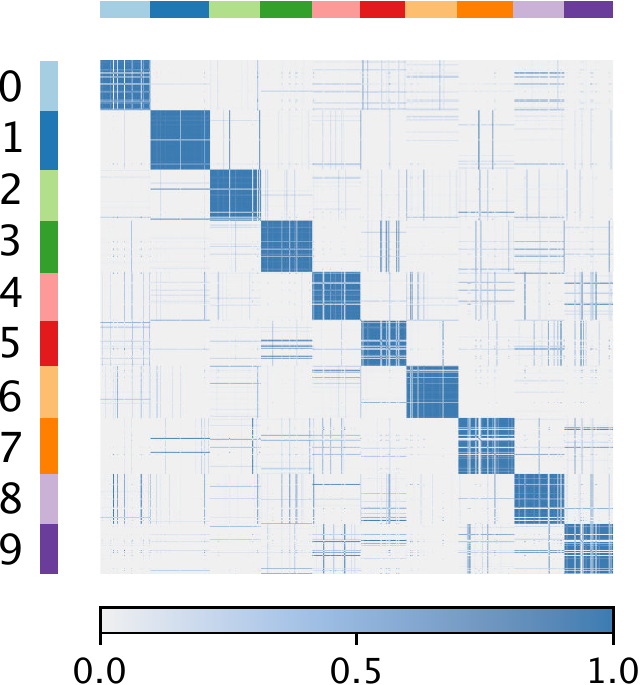}
        &
        \includegraphics[width=\linewidth]{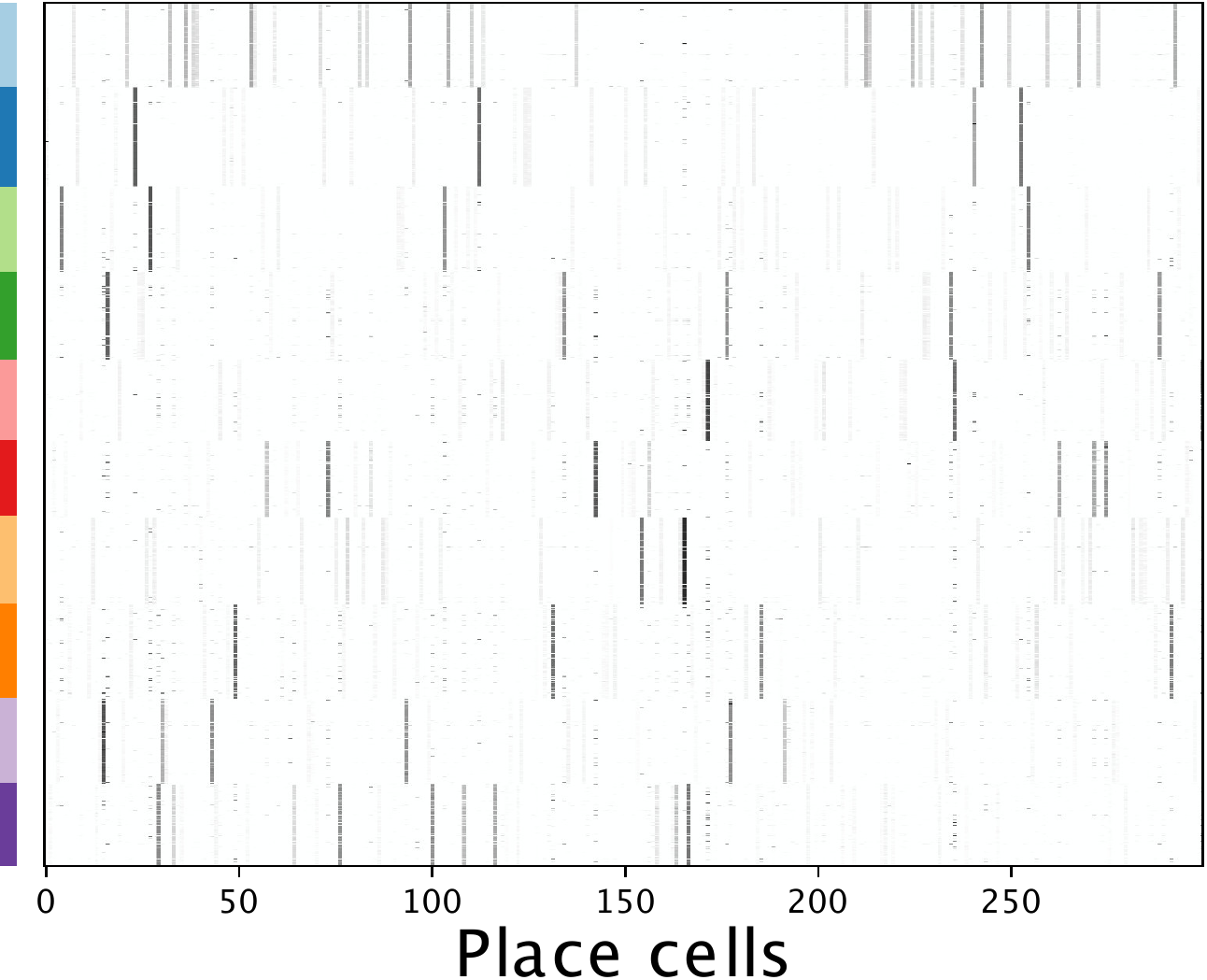}
        \\
    \end{tabu}
        
    \caption{Output kernel and receptive fields matrix for the experiment in \cref{fig:supervised}B, training the supervised layer with annotations on 10\% of the data.}
    \label{fig:mnist_additional}
\end{figure}

\end{document}